\documentclass[10pt,journal,compsoc]{IEEEtran}

\ifCLASSOPTIONcompsoc
  \usepackage[nocompress]{cite}
\else
  \usepackage{cite}
\fi

\usepackage{bm}
\usepackage{graphicx}  %Required
 \usepackage{mathrsfs}

\usepackage{cite}  %Required

\usepackage{amssymb}
\usepackage{algorithm}
\usepackage{amsmath}
\usepackage{amsthm}
\usepackage{tabularx,multirow,booktabs}
\usepackage{balance}
\usepackage[noend]{algpseudocode}
\usepackage{epsfig,subfigure}
\usepackage{color}

\allowdisplaybreaks

\newcommand{\Tr}{^{\rm T}}

\renewcommand{\vec}{{\rm vec}}

\renewcommand{\a}{{\bf a}}
\renewcommand{\b}{{\bf b}}

\newcommand{\g}{{\bf g}}

\renewcommand{\u}{{\bf u}}
\renewcommand{\v}{{\bf v}}
\newcommand{\w}{{\bf w}}
\newcommand{\x}{{\bf x}}
\newcommand{\y}{{\bf y}}

\newcommand{\A}{{\bf A}}

\newcommand{\Bcal}{\mathcal{B}}

\newcommand{\Ccal}{\mathcal{C}}

\newcommand{\Ecal}{\mathcal{E}}

\newcommand{\Gcal}{{\mathcal{G}}}
\renewcommand{\H}{{\bf H}}
\newcommand{\I}{{\bf I}}

\newcommand{\Q}{{\bf Q}}

\newcommand{\Vcal}{\mathcal{V}}

\newcommand{\X}{{\bf X}}

\newcommand{\Z}{{\bf Z}}

\newcommand{\bPhi}{\boldsymbol{\Phi}}

\newcommand{\bDelta}{\boldsymbol{\Delta}}
\newcommand{\blambda}{\boldsymbol{\lambda}}
\newcommand{\bLambda}{\mathbf{\Lambda}}

\newcommand{\bOmega}{\boldsymbol{\Omega}}
\newcommand{\bomega}{\boldsymbol{\omega}}

\newcommand{\bmu}{\boldsymbol{\mu}}
\newcommand{\1}{{\bf 1}}

\newcommand{\argmin}{\operatornamewithlimits{argmin}}

\newcommand{\lrincir}[1]{\left( #1 \right)}
\newcommand{\abs}[1]{\lvert#1\rvert}

\newcommand{\lrnorm}[1]{\left\lVert#1\right\rVert}

\newcommand{\RR}{\mathbb{R}}

\newcommand{\refabove}[2]{\displaystyle_{#1}^{(#2)}}

 \newtheorem{theorem}{\bf{Theorem}}
 \newtheorem{lemma}{\bf{Lemma}}

 \newtheorem{remark}{\bf{Remark}}

\begin{document}

\title{Simultaneous Clustering and Optimization for Evolving Datasets}

\author{Yawei Zhao,
        En Zhu$^\ast$,
        Xinwang Liu$^\ast$,
        Chang Tang,
        Deke Guo,
        Jianping Yin$^\ast$
        \thanks{${\ast}$ Corresponding authors.}
\IEEEcompsocitemizethanks{
\IEEEcompsocthanksitem Yawei Zhao, En Zhu and Xinwang Liu are with the College of Computer, National University of Defense Technology, Changsha, Hunan, 410073, China. E-mail: zhaoyawei@nudt.edu.cn; enzhu@nudt.edu.cn; xinwangliu@nudt.edu.cn.
\IEEEcompsocthanksitem Chang Tang is with the School of Computer Science, China University of Geosciences, Wuhan, 430074, China. E-mail: tangchang@cug.edu.cn.
\IEEEcompsocthanksitem Deke Guo is with the Science, Technology and Information Systems Engineering Laboratory, National University of Defense Technology, Changsha, Hunan, 410073, China. E-mail: guodeke@gmail.com.

\IEEEcompsocthanksitem  Jianping Yin is with the Dongguan University of Technology, Dongguan, Guangdong, 523000, China. E-mail: jpyin@dgut.edu.cn.
} 
}


\IEEEcompsoctitleabstractindextext{%

\begin{abstract}

Simultaneous clustering and optimization (SCO) has recently drawn much attention due to its wide range of practical applications. Many methods have been previously proposed to solve this problem and obtain the optimal model. However, when a dataset evolves over time, those existing methods have to update the model frequently to guarantee accuracy; such updating is computationally infeasible. In this paper, we propose a new formulation of SCO to handle evolving datasets.  Specifically, we propose a new variant of the alternating direction method of multipliers (ADMM) to solve this problem efficiently. The guarantee of model accuracy is analyzed theoretically for two specific tasks: ridge regression and convex clustering. Extensive empirical studies confirm the effectiveness of our method. 

\end{abstract}

}
\maketitle

\IEEEpeerreviewmaketitle

\section{Introduction}
\label{introduction}

Simultaneous clustering and optimization (SCO) has recently drawn much attention in the machine learning and data mining community \cite{Hallac:2015fy,Zhao:2018tk}. Let us consider an example to explain this. Suppose that we want to predict the price of houses in New York City. The prices of houses located in the same region should be predicted by using similar prediction models. The prices of houses located in different regions should be predicted by using different prediction models. Traditional methods usually involve two separate steps. Such methods first learn prediction models for every house and then use clustering methods such as k-means clustering \cite{Lloyd1982Least,bottou1994convergence,8587193,8519323} to determine the similarity among the obtained prediction models. However, the purpose of the SCO task is to perform prediction and identify the similarity among prediction models simultaneously; this approach usually outperforms traditional solutions. 

Previous methods such as network lasso \cite{Hallac:2015fy,Ghosh2016An,Zhao:2018tk} formulate the SCO problem as that of a convex objective function with a sum-of-norms regularizer, which usually leads to a high computational cost for to following reasons:
\begin{itemize}
\item The number of optimization variables increases linearly with the number of instances and features.
\item The optimization variables are highly nonseparable due to the sum-of-norms regularization.
\item The objective function is extremely nonsmooth. 
\end{itemize} Due to those challenges, many optimization methods such as the alternating minimization algorithm (AMA) \cite{Chi:2013ey} and the alternating direction method of multipliers (ADMM) \cite{Hallac:2015fy,Zhao:2018tk}  have been developed to reduce the computational cost. Although these existing methods 
obtain great efficiency improvement for a static dataset, they are unable to handle an evolving dataset directly due to exceptionally high computational complexity. The reason is that when a dataset evolves over time, the optimal model of SCO has to be frequently updated over time.  Otherwise, the model accuracy cannot be guaranteed. However, it is impractical to update the optimal model frequently due to the high computational cost. Therefore, finding an effective method of performing SCO on an evolving dataset with a guarantee of model accuracy is a meaningful problem.  

Let us consider an example to explain our motivation. Determining a cluster path \cite{Chi:2013ey,Tan:2015vr,Radchenko:2017gg} is one of SCO tasks. This is done by performing convex clustering over multiple rounds.   As illustrated in Figure \ref{figure_illustrative_motivation}, the cluster paths for images $1$ and $2$ are very similar due to few changes of pixels. However, the cluster path of image $3$, namely, Figure \ref{figure_plane_clusterpath_05}, is significantly different from them. It is impractical to update the optimal cluster path for every image timely because it takes at least $18$ seconds to obtain a cluster path. Additionally, compared with Figure \ref{figure_plane_clusterpath_00} and Figure \ref{figure_plane_clusterpath_01}, we observe that a cluster path for image $1$, i.e., Figure \ref{figure_plane_clusterpath_00}, can be used for a similar image $2$ with some slight loss of accuracy. Compared with Figure \ref{figure_plane_clusterpath_00} and Figure \ref{figure_plane_clusterpath_05}, we observe that a cluster path should be updated if the image is changed significantly. In a general scenario, we are thus motivated by the following two nontrivial and challenging problems:
\begin{itemize}
\item How to obtain a model with a guaranteed accuracy that can be used for a dataset undergoing slight changes, and
\item When to update the model if the dataset has evolved to become sufficiently different.
\end{itemize}

\begin{figure*}[!t]
\setlength{\abovecaptionskip}{0pt}
\setlength{\belowcaptionskip}{0pt}
\centering 
\subfigure[gray image $1$]{\includegraphics[width=0.64\columnwidth]{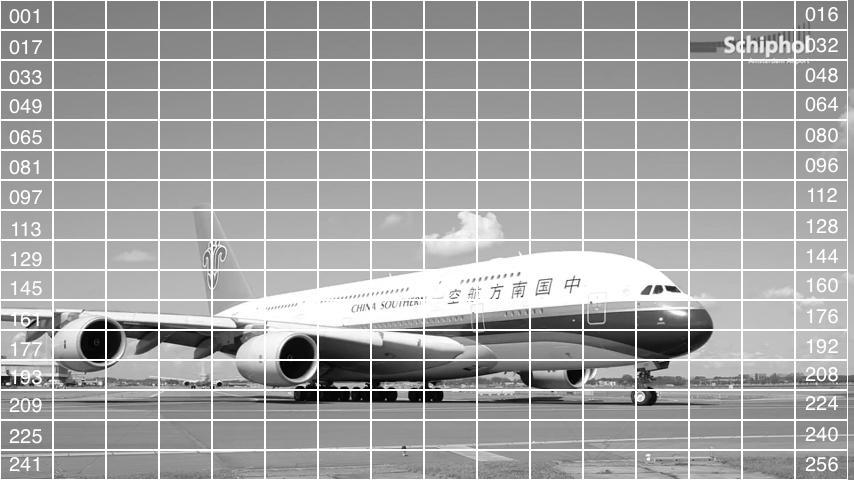}\label{gray_plane_cross00_lines}}
\subfigure[gray image $2$]{\includegraphics[width=0.64\columnwidth]{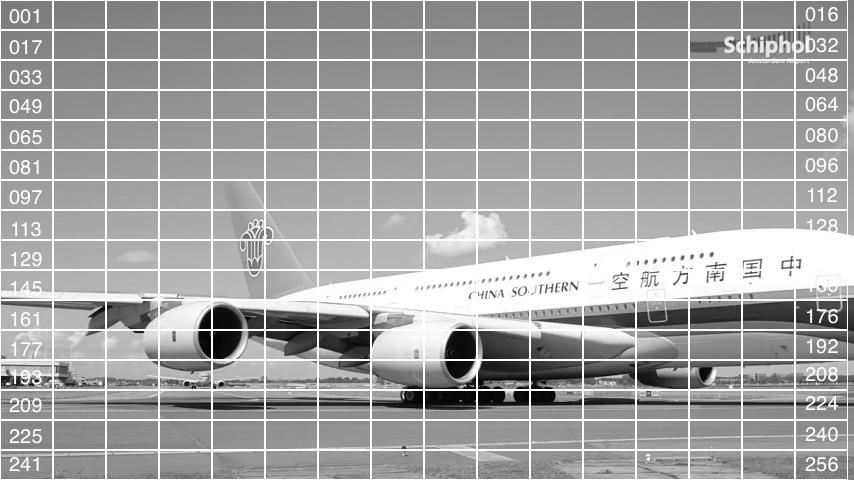}\label{gray_plane_cross01_lines}}
\subfigure[gray image $3$]{\includegraphics[width=0.64\columnwidth]{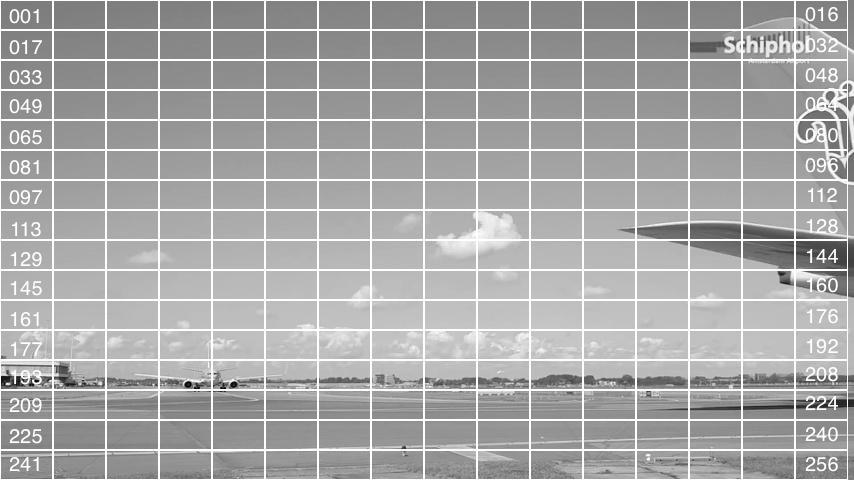}\label{gray_plane_cross05_lines}}
\subfigure[cluster path $1$, $19.7$ seconds]{\includegraphics[width=0.64\columnwidth]{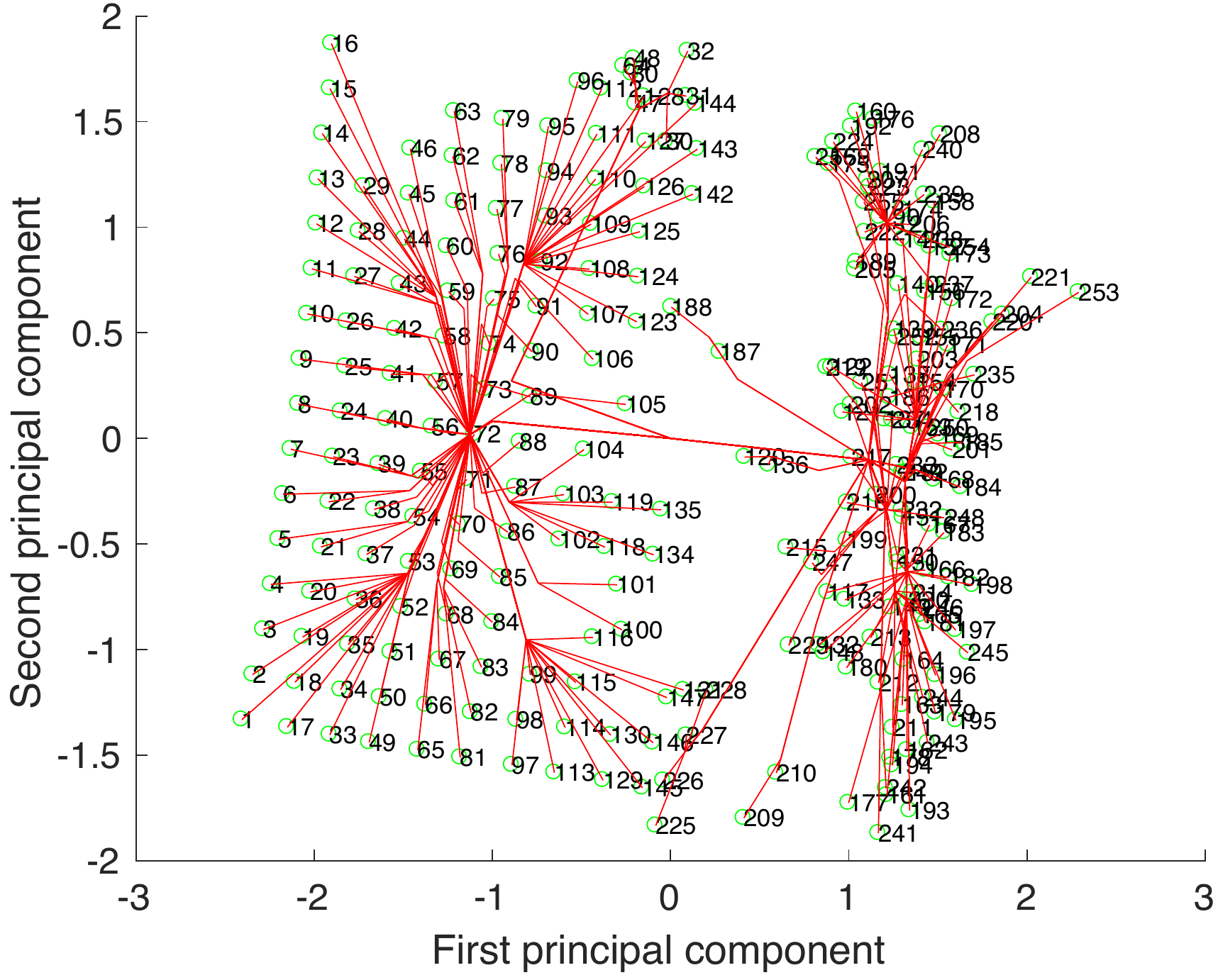}\label{figure_plane_clusterpath_00}}
\subfigure[cluster path $2$, $19.0$ seconds]{\includegraphics[width=0.64\columnwidth]{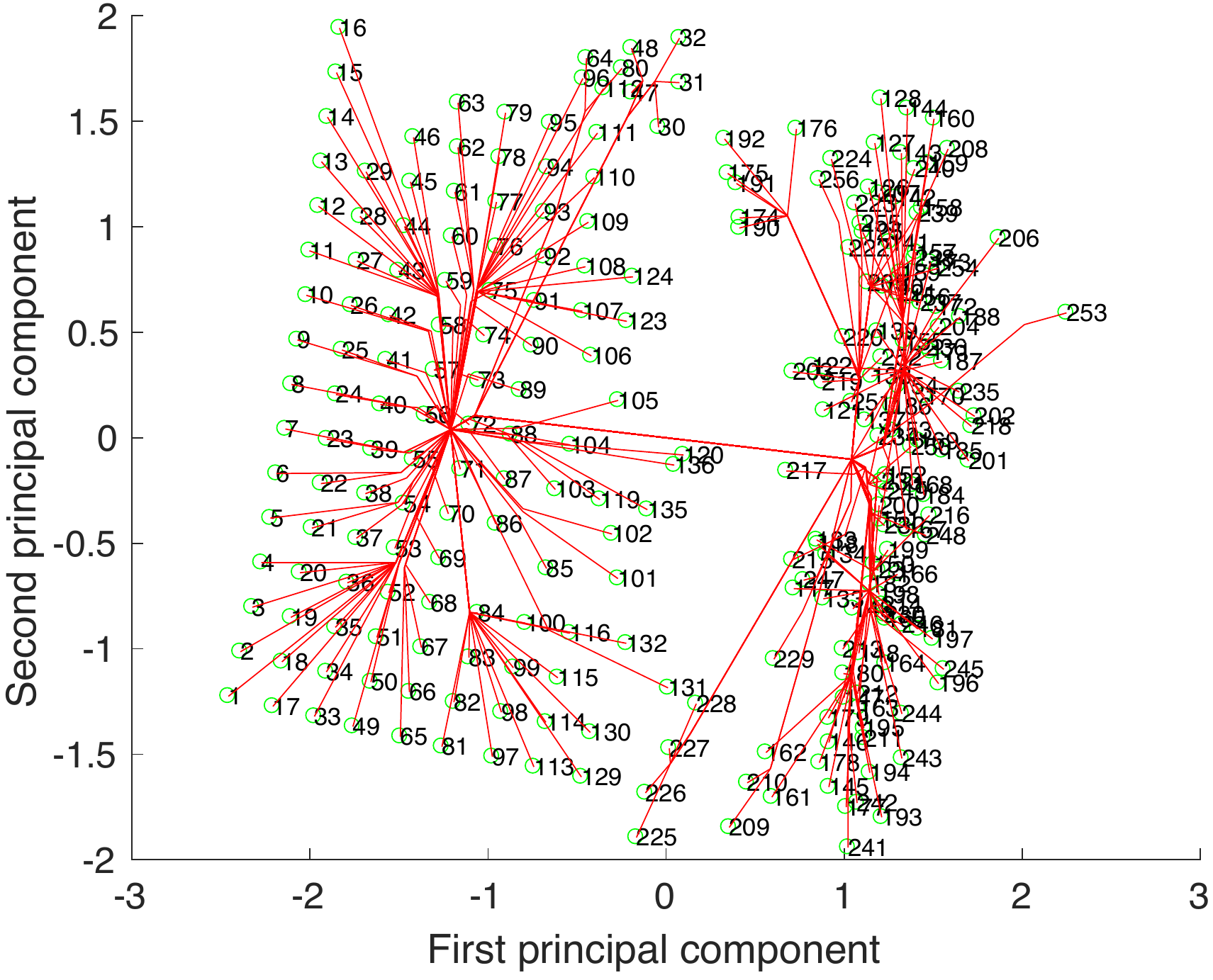}\label{figure_plane_clusterpath_01}}
\subfigure[cluster path $3$, $18.8$ seconds]{\includegraphics[width=0.64\columnwidth]{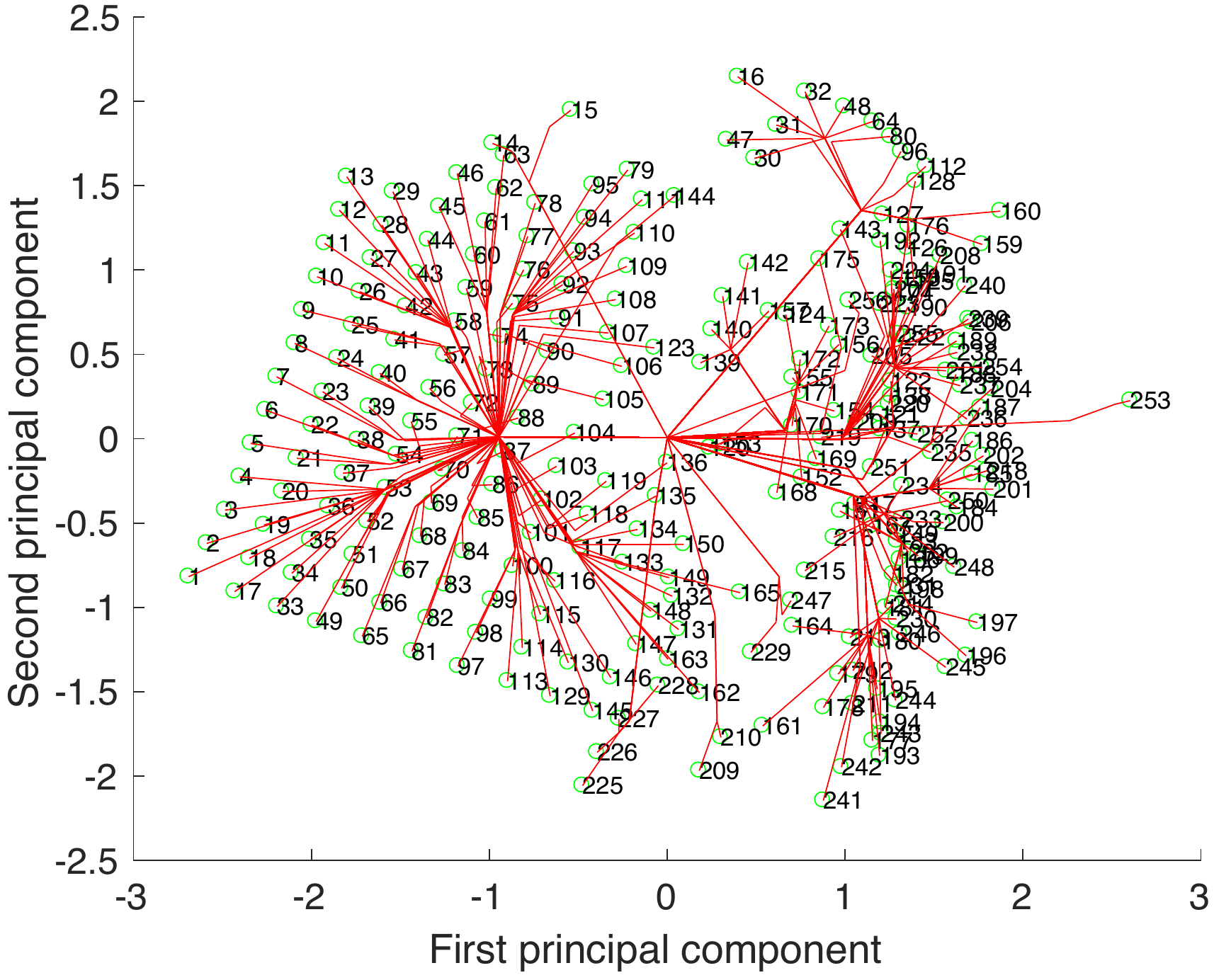}\label{figure_plane_clusterpath_05}}
\caption{There is a plane at an airport, and it is crossing the field of view. Figure \ref{gray_plane_cross01_lines} shows the next frame in the video after the frame shown in Figure \ref{gray_plane_cross00_lines}. Figure \ref{gray_plane_cross05_lines} shows that the plane is nearly outside the field of view after $2$ seconds.  We partition every image into $256$ blocks and identify each block with a unique number. The maximum gray value corresponding to pixels in a block is used to represent the block. The cluster paths of Figures \ref{gray_plane_cross00_lines}, \ref{gray_plane_cross01_lines}, and \ref{gray_plane_cross05_lines} are illustrated in Figures \ref{figure_plane_clusterpath_00}, \ref{figure_plane_clusterpath_01}, and \ref{figure_plane_clusterpath_05}, respectively. Note that the cluster path of an image is obtained by running the previous method for more than $18$ seconds, which is unacceptably long in video analysis.  This illustrative example shows that it is vitally important to find an effective method for obtaining the cluster path with a guaranteed accuracy for a video dataset.}
\label{figure_illustrative_motivation}
\end{figure*}

In this paper, we aim to answer the above two challenging questions. We reformulate the problem of SCO as a convex problem with cone constraints in the dual space. Compared with the formulation in the previous studies, the new formulation does not contain the sum-of-norms regularizer and is thus much easier to solve. Then, a new regularizer is proposed to allow for the optimal model to be insensitive to slightly evolving data. A new metric is proposed for deciding when to update the model for a significant change of a dataset. Additionally, we propose a new variant of the alternating direction method of multipliers (ADMM) to solve the proposed problem efficiently. Furthermore, the guarantee of the model accuracy is analyzed theoretically for convex clustering and ridge regression tasks. Extensive empirical studies show the advantages of the proposed method. In brief, our contributions are summarized as follows:
\begin{itemize}
\item A new formulation of the SCO problem is proposed to handle an evolving dataset. 
\item An efficient variant of ADMM is proposed to solve the proposed problem efficiently.
\item The accuracy of the model is analyzed theoretically for convex clustering and ridge regression tasks.
\item Extensive empirical studies show the superiority of the proposed method.
\end{itemize}

The paper is organized as follows. Section \ref{sect_related_work} reviews the related studies. Section \ref{sect_preliminaries} introduces the preliminaries. Section \ref{sect_scoed} describes a new formulation of SCO for an evolving dataset. Section \ref{sect_admm} discusses a novel ADMM method for solving the problem. Section \ref{sect_theoretical_analysis} presents a theoretical analysis. Section \ref{sect_experiment} details the experiments. Section \ref{sect_conclusion} concludes the paper.

\section{Related studies}
\label{sect_related_work}
In this section, we briefly review the related literature.
\subsection{Network lasso}
Network lasso was the first method proposed to solve the SCO problem \cite{Hallac:2015fy}. It was subsequently implemented and published as a general optimization tool for graph analysis in \cite{Leskovec2016SNAP}. Recently, network lasso has drawn much attention in various application scenarios \cite{Ghosh2016An,Jung:2017ug,Zhao:2018tk}. It was used in \cite{Ghosh2016An} to predict the location of a shared bike. The study \cite{Jung:2017ug} examined a sufficient condition for the network topology to yield an accurate solution. Network lasso was extended in \cite{Zhao:2018tk} to handle noisy and missing data. The cited existing studies extended network lasso for static datasets; however, such approaches are unsuitable for handling evolving datasets due to a high computational cost.

\subsection{Convex clustering}
The special case of SCO, i.e., convex clustering \cite{Lindsten2011Just}, has been extensively studied for several years. Compared to the traditional clustering methods, it has a convex formulation, leading to a robust clustering result. For example, the clustering result of k-means is sensitive to the seeds, and picking good seeds is challenging \cite{Lloyd1982Least,ZHAO2018184}. However, due to a convex objective function, the result of convex clustering can be determined. It is not impacted by any heuristic rules used in traditional clustering methods. A formulation of convex clustering was proposed in \cite{Lindsten2011Just} by relaxing the formulation of k-means clustering. Subsequently, \cite{Zhu:2014wo} and \cite{icml2017} provided several sufficient conditions for recovering the clustering membership theoretically.  Other studies, e.g., \cite{Chi:2013ey,Yuan2018An}, focus on improving the efficiency of convex clustering. Although those previous studies attained great improvement of convex clustering for static datasets, they are unsuitable for handling evolving datasets due to a high computational cost. The method proposed in the paper reduces such computational cost and makes a good tradeoff between efficiency and accuracy.

\section{Preliminaries}
\label{sect_preliminaries}
In this section, several important notations are introduced. Then, the problem of simultaneous clustering and optimization is presented.

\begin{figure}[!t]
\setlength{\abovecaptionskip}{0pt}
\setlength{\belowcaptionskip}{0pt}
\centering 
\subfigure[data graph]{\includegraphics[width=0.45\columnwidth]{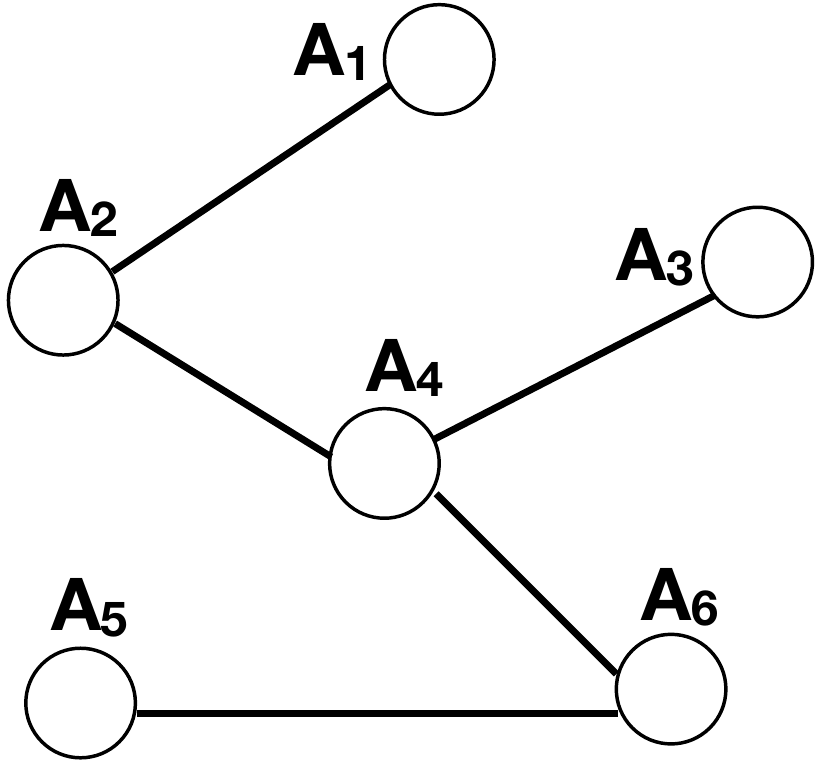}\label{figure_basic_sco_data_graph}}
\hspace{5pt}
\subfigure[variable graph $\Gcal$]{\includegraphics[width=0.45\columnwidth]{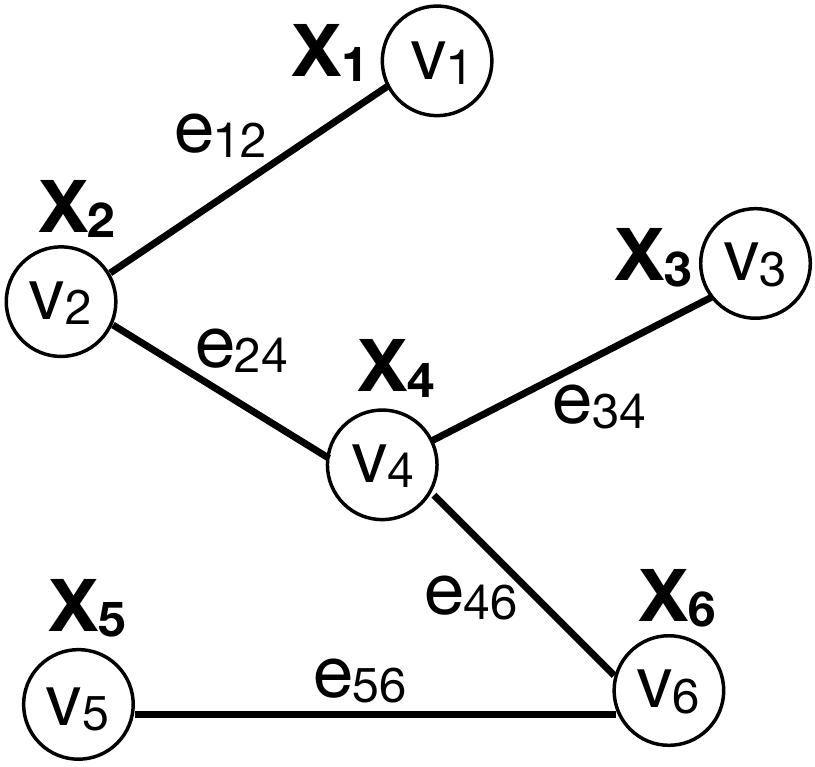}\label{figure_basic_sco_variable_graph}}
\caption{Illustrative example of the graph abstraction. For any $i$ such that $1\le i\le 6$, $\A_i$ represents an instance of the dataset, $\X_i$ represents the corresponding optimization variable, $v_i$ represents a vertex of graph $\Gcal$, and $e_{ij}$ represents the edge connecting $v_i$ and $v_j$.}
\label{figure_illustrative_graph_abstraction_sco}
\end{figure}

\subsection{SCO: simultaneous clustering and optimization}
SCO is an optimization framework with a formulation that is usually presented as a convex objective function with a sum-of-norms regularizer. Before discussing its formulation, we introduce the data model -- a graph abstraction of a dataset.

The basic data model consists of two graphs -- a data graph $\Gcal^{\text{data}}$ and a variable graph $\Gcal^{\text{variable}}$ -- that are generated as follows.
\begin{itemize}
\item \textbf{Data graph $\Gcal^{\text{data}}$.} Every instance in a dataset is represented by a vertex. Generally, the $K$-nearest neighbors ($K$-NN) method is performed on the dataset. If an instance is one of the $K$ nearest neighbors of another instance, then an edge is generated to connect them. Thus, we obtain a graph $\Gcal^{\text{data}}$ that measures the similarity among instances.  
\item \textbf{Variable graph $\Gcal^{\text{variable}}$. } Variable graph $\Gcal^{\text{variable}}$ is generated based on graph $\Gcal^{\text{data}}$. An optimization variable, e.g., $\X_i$, is represented by a vertex $v_i$. If two instances, e.g., $\A_i$ and $\A_j$ in graph $\Gcal^{\text{data}}$, are connected by an edge, then vertices $v_i$ and $v_j$ are connected by edge $e_{ij}$.
\end{itemize}

\begin{figure}[!t]
\setlength{\abovecaptionskip}{0pt}
\setlength{\belowcaptionskip}{0pt}
\centering 
\subfigure[$\alpha = 6$]{\includegraphics[width=0.45\columnwidth]{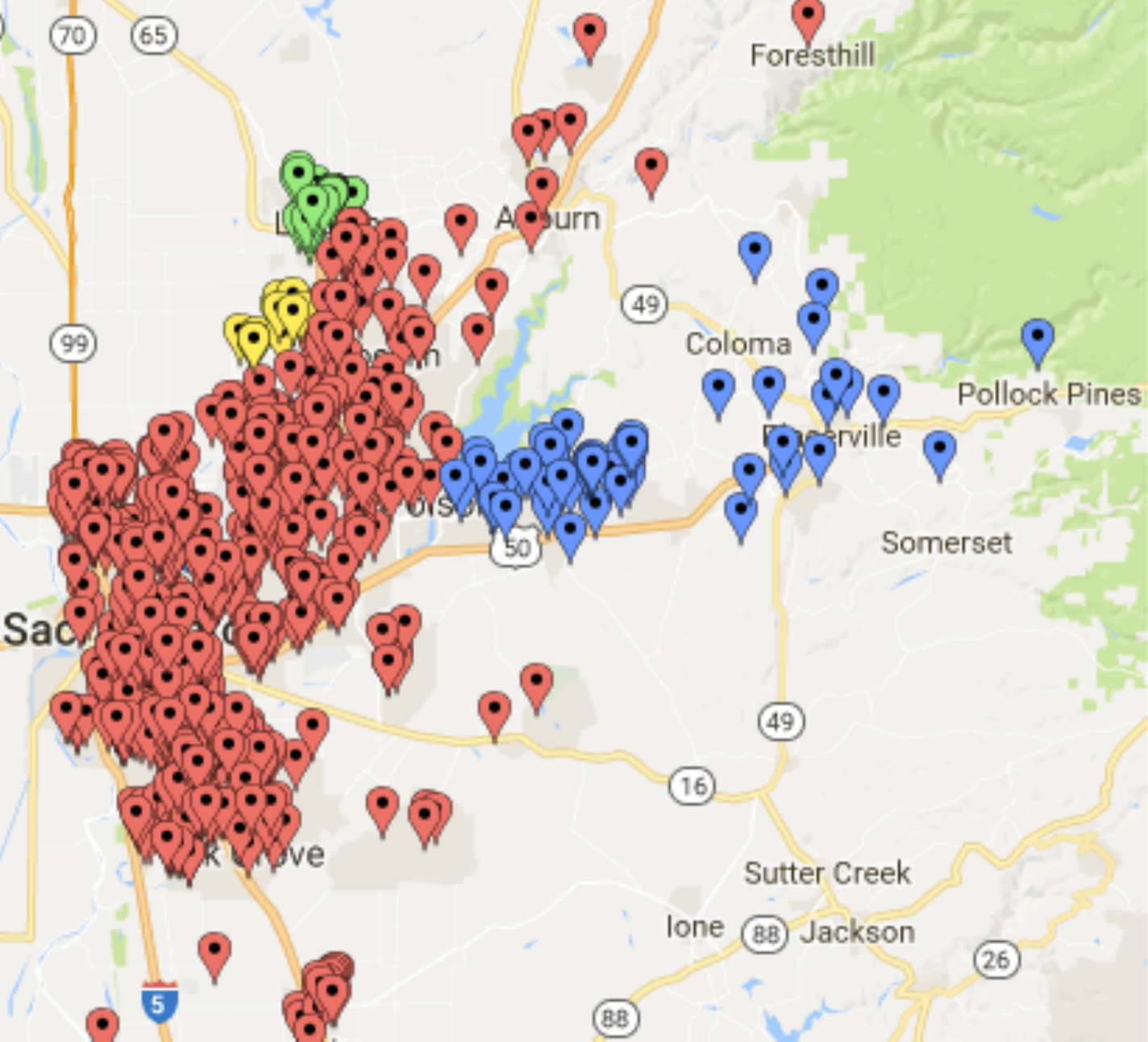}\label{figure_example_triangle_4clusters}}
\hspace{5pt}
\subfigure[$\alpha = 4$]{\includegraphics[width=0.45\columnwidth]{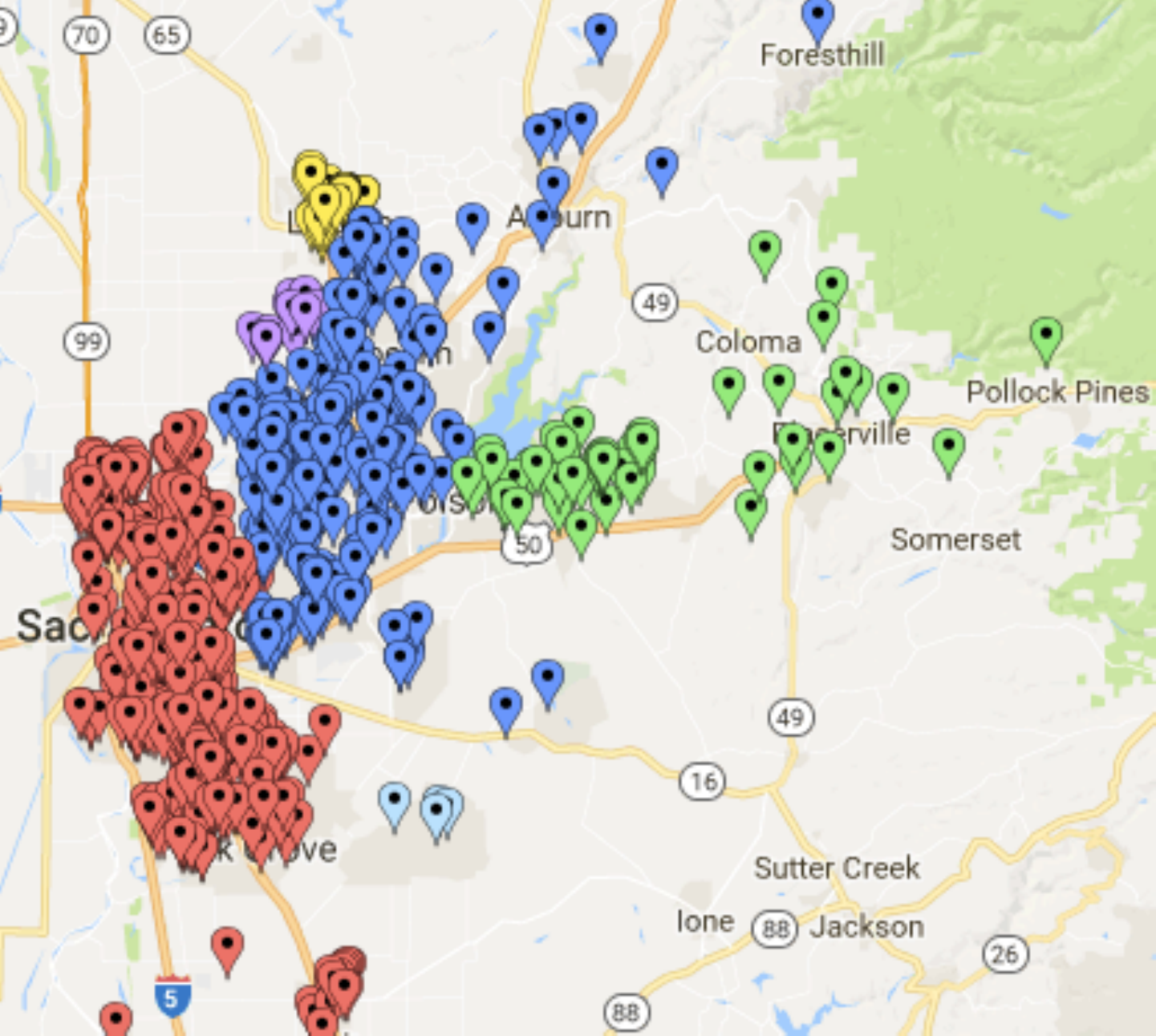}\label{figure_example_triangle_6clusters}}
\caption{Example illustrating SCO if $f_i(\X_i)$ is instantiated to be the empirical loss of ridge regression \cite{Zhao:2018tk}. The task needs to learn a prediction model for every house in the Greater Sacramento area. A marker represents a house. If markers have the same color, the corresponding houses have identical prediction models. As we observe, with the increase in $\alpha$, the same model is used for more houses. The experimental details are described in \cite{Zhao:2018tk}. }
\label{figure_illustrative_sco_house_price_prediction}
\end{figure} 

Let us consider an illustrative example as a further explanation. Figure \ref{figure_basic_sco_data_graph} shows a data graph for a dataset consisting of $6$ instances. Figure \ref{figure_basic_sco_variable_graph} shows the variable graph for the data graph $\Gcal^{\text{data}}$ in the left panel. Since SCO is formulated based on the variable graph $\Gcal^{\text{variable}}$, we denote the variable graph $\Gcal^{\text{variable}}$ by $\Gcal$ by default in the following content. The vertex set of $\Gcal$ is denoted by $\Vcal$, and the edge set of $\Gcal$ is denoted by $\Ecal$. Thus, SCO \cite{Hallac:2015fy} is formulated as
\begin{align}
\label{equa_formulation_basic_sco}
\min_{\X\in\RR^{n\times d}} \underbrace{\sum_{v_i\in\Vcal} f_i(\X_i)}_{\text{the empirical loss}} + \alpha \underbrace{\sum_{e_{ij}\in\Ecal} w_{ij}\lrnorm{\X_i - \X_j}_{p}}_{\text{the sum-of-norms regularizer}}.
\end{align} Here, $p$ is picked from $1, 2, \infty$. Parameter $\alpha$ is a hyperparameter that should be specified. Parameter $w_{ij}$ is the weight of the edge $e_{ij}$. When the variable graph $\Gcal$ is obtained, $w_{ij}$ is usually set to be inversely proportional to the distance between $\A_i$ and $\A_j$ \cite{Chi:2013ey}.    The optimal value of $\X_i$, denoted by $\X_{i\ast}$, represents the prediction model for the $i$-th instance, i.e., $\A_i$.

The complete objective function of \eqref{equa_formulation_basic_sco} consists of two parts explained as follows.
\begin{itemize}
\item \textbf{Empirical loss.} SCO can be used on various data analysis tasks by instantiating the empirical loss of $f_i(\X_i)$. Function $f_i(\cdot)$  usually represents the empirical loss due to the $i$-th instance, which is determined by a specific machine learning task, such as regression, classification, and clustering. For example, if  
\begin{align}
\nonumber
\sum_{v_i\in\Vcal} f_i(\X_i) = \sum_{v_i\in\Vcal} \lrnorm{\X_i-\A_i}_2^2 =  \lrnorm{\X-\A}_F^2
\end{align} holds, it represents the empirical loss of convex clustering.   If the number of vertices in $\Vcal$ is $n$, and 
\begin{align}
\nonumber
\sum_{v_i\in\Vcal} f_i(\X_i) = \sum_{i=1}^n \lrincir{\lrnorm{\A_i\X\Tr_i - \y_i}^2_2 + \gamma \lrnorm{\X_i}_2^2}
\end{align} holds, it represents the empirical loss of ridge regression. 
\item \textbf{Sum-of-norms regularizer.} The sum-of-norms regularizer $\lrnorm{\X_i - \X_j}_{p}$ is essential for performing simultaneous clustering and optimization.  As illustrated in \eqref{equa_formulation_basic_sco}, SCO learns a model, e.g., $\X_i$, for every instance, e.g., $\A_i$. The most noticeable difference between SCO and the traditional machine learning task is the \textit{sum-of-norms regularizer}. 
The regularizer controls the similarity among those models. If $\alpha = 0$, every instance has a different model from those of others. With the increase in $\alpha$, some instances tend to have similar or even identical models. If $\alpha$ is sufficiently large, all instances may have the same model.  We refer to an illustrative example \cite{Zhao:2018tk} in Figure \ref{figure_illustrative_sco_house_price_prediction} to provide further explanations about the regularizer.
\end{itemize}

\subsection{Notation}
Suppose that graph $\Gcal$ consists of $n$ vertices and $m$ edges.    The $i$-th vertex is represented by $v_i$, and its corresponding optimization variable is $\X_i\in\RR^{1\times d}$. The edge connecting $v_i$ and $v_j$ is represented by $e_{ij}$. The weight of $e_{ij}$ is denoted by $w_{ij}$. $\A\in\RR^{n\times d}$ represents the data matrix consisting of $n$ instances, and every instance is characterized by $d$ features. $\X\in\RR^{n\times d}$ represents the matrix of optimization variables.  The other important notations are shown as follows.
\begin{itemize}
\item Ordinary lowercase letters, e.g., $\alpha$ and $\beta$, represent constant scalars. Bold lowercase letters, e.g., $\y$, represent vectors.  Bold capital letters, e.g., $\A\in \RR^{n\times d}$ and $\Q\in\RR^{m\times n}$, represent matrices.
\item A bold capital letter with a subscript, e.g., $\A_{i}\in\RR^{1\times d}$, represents the $i$-th row of a matrix. A bold capital letter with two subscripts, e.g., $\Q_{ij}$, represents the element located at the $i$-th row and the $j$-th column.
\item $\vec(\cdot)$ represents the column stacking vectorization of a matrix.
\item $\text{diag}(\v)$ represents the diagonal matrix consisting of the elements of vector $\v$.
\item $\otimes$ represents the Kronecker product. $\odot$ represents the Hadamard product.
\item $\lrnorm{\cdot}$ represents a norm of a vector, and $\lrnorm{\cdot}_{\ast}$ represents its dual norm.
\item $\I_d$ represents the identity matrix. $\mathbf{0}$ and $\mathbf{1}$ represent constant matrices with elements of $0$ and $1$, respectively.
\end{itemize}

\section{Simultaneous clustering and optimization for evolving datasets}
\label{sect_scoed}
In this section, we propose a new formulation for performing SCO. Then, several examples are provided to present more details about the formulation.

\subsection{Formulation}
Since the dataset is evolving, the change in the data matrix $\A$ will impact the solution. Without a loss of generality, we reformulate \eqref{equa_formulation_basic_sco} as follows:
\begin{align}
\label{equa_basic_simultaneous_clustering_optimization_primal}
\min_{\X\in\RR^{n\times d}} f(\X; \A) + \alpha \sum_{e_{ij}\in\Ecal} w_{ij}\lrnorm{\X_i - \X_j}_{p}.
\end{align} Suppose that the number of vertices in $\Vcal$ is n, i.e., $n=|\Vcal|$, and the number of edges in $\Ecal$ is $m$, i.e., $m=|\Ecal|$. We introduce an auxiliary matrix $\Q\in\RR^{m\times n}$ to help reformulate the optimization's objective function. Any row of $\Q$ consists of one positive element, one negative element and the remaining zero elements. Any row of $\Q$ corresponds to an edge of graph $\Gcal$.  If the $k$-th row of $\Q$, i.e., $\Q_k$, has a positive value at the $i$-th element and a negative value at the $j$-th element, it corresponds to edge $e_{ij}$. In this case, the positive element of $\Q_k$ is $\Q_{ki}:=\alpha w_{ij}$, and the negative element of $\Q_k$ is $\Q_{kj} := - \alpha w_{ij}$. In other words,
\begin{align}
\Q_k = (\mathbf{0}, \underbrace{\Q_{ki}}_{\Q_{ki}:=\alpha w_{ij}}, \mathbf{0}, \underbrace{\Q_{kj}}_{\Q_{kj}:=-\alpha w_{ij}}, \mathbf{0}).
\end{align}

Additionally, the sum-of-norms regularization is usually denoted by the $l_{1,p}$ norm, i.e., $\lrnorm{\cdot}_{1,p}$ \cite{Hallac:2015fy,Chi:2013ey,icml2017,Yuan2018An}. It is defined by $\lrnorm{\cdot}_{1,p} := \sum_{e_{ij}\in\Ecal} \lrnorm{\cdot}_{p}$. Using $\Q$ to reformulate the optimization objective function of \eqref{equa_basic_simultaneous_clustering_optimization_primal}, we thus obtain
\begin{align}
\nonumber
\min_{\X\in\RR^{n\times d}} f(\X; \A) + \lrnorm{\Q\X}_{1,p}.
\end{align} 

\begin{theorem}
\label{theorem_formulation_dual}
The dual formulation of \eqref{equa_basic_simultaneous_clustering_optimization_primal} is equivalent to 
\begin{align}
\label{equa_final_dual_objective}
\min_{\blambda\in\RR^{m\times d}} & f^{\ast}(-\vec\Tr(\blambda)(\I_d\otimes\Q); \A)
\end{align} subject to 
\begin{align}
\nonumber
\lrnorm{\blambda_i}_q \le 1, 1\le i \le m.
\end{align} Once the quantity $\blambda_\ast$ that minimizes \eqref{equa_final_dual_objective} has been obtained, the quantity $\X_\ast$ that minimizes the primal problem \eqref{equa_basic_simultaneous_clustering_optimization_primal} is obtained by 
\begin{align}
\label{equa_minimizer_primal}
\nabla f(\X_\ast;\A) + (\I_d \otimes \Q)\Tr \vec(\blambda_\ast) = 0.
\end{align} Here, $p\in\{1,2,\infty\}$, and $\frac{1}{p}+\frac{1}{q} = 1$. Table \ref{table_constraint_norms_types} shows the specific values of $p$ and $q$.
\end{theorem}

\begin{algorithm}[t]
    \caption{SCO for evolving datasets}
    \label{algo_evolution_detected_regularization}
    \begin{algorithmic}[1]
    \Require Data matrix $\A$, a positive $\beta$ and a threshold $c$ to control the accuracy of the solution for evolving datasets.
        \While {The change in $\A$, i.e., $\A+\bDelta$, is detected.}
            \State Obtain $\Delta_{f^{\ast}}$ according to \eqref{equa_difinition_bDelta}.
            \If {$\Delta_{f^{\ast}} \ge c$}
                \State $\A\leftarrow \A+\bDelta$.
                \State Solve the optimization problem \eqref{equa_unware_regularized_dual_objective}, and obtain the minimizer $\blambda_\ast$.
                \State Obtain $\X_\ast$ according to \eqref{equa_minimizer_primal}.
            \EndIf
        \EndWhile 
    \end{algorithmic}
\end{algorithm}

\begin{table}
\centering
\caption{$l_p$ norm and its dual norm $l_q$.}
\begin{tabular}{c|c|c|c}
\hline 
$\lrnorm{\cdot}_p$ & $\lrnorm{\cdot}_1$ & $\lrnorm{\cdot}_2$ & $\lrnorm{\cdot}_{\infty}$\tabularnewline
\hline 
\hline 
$\lrnorm{\cdot}_q$ & $\lrnorm{\cdot}_{\infty}$ & $\lrnorm{\cdot}_2$ & $\lrnorm{\cdot}_1$\tabularnewline
\hline 
constraints for $\lrnorm{\cdot}_q$ & box & ball & simplex \tabularnewline
\hline  
\end{tabular}
\label{table_constraint_norms_types}
\end{table}

Therefore, our new formulation of the dual problem is
\begin{align}
\label{equa_unware_regularized_dual_objective}
\min_{\blambda\in\RR^{m\times d}} & f^{\ast}(-\vec\Tr(\blambda)(\I_d\otimes\Q); \A) + \beta \lrnorm{(\I_d\otimes \Q)\Tr \vec(\blambda)}_s
\end{align} subject to
\begin{align}
\nonumber
\lrnorm{\blambda_i}_q \le 1, 1\le i \le m.
\end{align} Here, $s$ can be $1$, $2$ or $\infty$.  Parameter $\beta$ controls the accuracy of the solution for the evolving dataset. A small $\beta$ means that the solution of \eqref{equa_unware_regularized_dual_objective} is sensitive to a perturbation of the data matrix $\A$.  A large $\beta$ causes the solution of \eqref{equa_unware_regularized_dual_objective} to be robust to a perturbation of the data matrix.

\textbf{Intuitive idea.} When $\A$ evolves to be $\A+\bDelta$, the minimizer $\blambda_\ast$ is not the true minimizer $\tilde{\blambda}_\ast$ of \eqref{equa_final_dual_objective} based on the evolving dataset $\A+\bDelta$. However, if the difference between $\A$ and $\A+\bDelta$ is not significant, it is reasonable to use $\blambda_\ast$ as an approximation of the true minimizer of \eqref{equa_final_dual_objective}. Note that
\begin{align}
\nonumber
\lrnorm{\nabla f(\X_\ast;\A)} = & \lrnorm{ (\I_d \otimes \Q)\Tr \vec(\blambda_\ast)}.
\end{align} Consider a fixed $\X_\ast$; then, $\lrnorm{\nabla f(\X_\ast;\A)}$ changes when $\A$ evolves into $\A+\bDelta$. The smaller the change is between $\lrnorm{\nabla f(\X_\ast;\A)}$ and $\lrnorm{\nabla f(\X_\ast;\A+\bDelta)}$, the closer $\X_\ast$ and the true minimizer $\tilde{\X}_\ast$ based on $\A+\bDelta$. Making $\lrnorm{\nabla f(\X_\ast;\A)}$ insensitive to  $\bDelta$ is equivalent to making $\lrnorm{ (\I_d \otimes \Q)\Tr \vec(\blambda_\ast)}$ insensitive to $\bDelta$. The reason is that $\blambda_\ast$ determines $\X_\ast$ according to \eqref{equa_minimizer_primal}.  We add a regularizer $\lrnorm{(\I_d\otimes \Q)\Tr \vec(\blambda)}_s$  into the objective function of \eqref{equa_final_dual_objective}, which penalizes the objective function for a large $\lrnorm{(\I_d\otimes \Q)\Tr \vec(\blambda)}_s$. When $\A$ evolves, the regularizer guarantees that the change in $\lrnorm{(\I_d\otimes \Q)\Tr \vec(\blambda)}_s$ is bounded within the user's control.

We define a new metric to decide when to update the model for the evolving dataset.
\begin{align}
\label{equa_difinition_bDelta}
\Delta_{f^{\ast}} :=& \lvert f^{\ast}(-\vec\Tr(\blambda_\ast)(\I_d\otimes\Q); \A+\bDelta) \\ \nonumber  
&- f^{\ast}(-\vec\Tr(\blambda_\ast)(\I_d\otimes\Q); \A)\rvert.
\end{align} If this metric exceeds a threshold, then the model needs to be updated.  Finally, we propose Algorithm \ref{algo_evolution_detected_regularization} to perform SCO for evolving datasets.

\subsection{Examples}
\label{subsect_formulation_example}
We provide two examples -- convex clustering \cite{Lindsten2011Just} and ridge regression \cite{Shalev2014Understanding} -- to present additional explanations of the proposed formulation. 

\textbf{Convex clustering.} The optimization objective function of convex clustering is
\begin{align}
\nonumber
\min_{\X\in\RR^{n\times d}} \lrnorm{\X - \A}^2_F + \lrnorm{\Q\X}_{1,p}.
\end{align} In this case, $f(\X;\A) = \lrnorm{\X - \A}^2_F$, and $f^{\ast}(-\vec\Tr(\blambda; \A)(\I_d\otimes \Q)\Tr)= -\vec\Tr(\A)(\I_d\otimes \Q)\Tr \vec(\blambda) + \frac{1}{4} \vec\Tr(\blambda)(\I_d\otimes \Q)(\I_d\otimes \Q)\Tr\vec(\blambda)$. We thus obtain the formulation of the dual problem of convex clustering as
\begin{align}
\nonumber
\min_{\blambda\in\RR^{m\times d}} & -\vec\Tr(\A)(\I_d\otimes \Q)\Tr \vec(\blambda) \\ \nonumber 
& + \frac{1}{4} \vec\Tr(\blambda)(\I_d\otimes \Q)(\I_d\otimes \Q)\Tr\vec(\blambda)
\end{align} subject to 
\begin{align}
\nonumber
&\lrnorm{\blambda_i}_q \le 1, 1\le i \le m.
\end{align} The dual problem of convex clustering is
\begin{align}
\nonumber
\min_{\blambda\in\RR^{m\times d}} & -\vec\Tr(\A)(\I_d\otimes \Q)\Tr \vec(\blambda) \\ \nonumber 
& + \frac{1}{4} \vec\Tr(\blambda)(\I_d\otimes \Q)(\I_d\otimes \Q)\Tr\vec(\blambda) \\ \nonumber
& + \beta \lrnorm{(\I_d\otimes \Q)\Tr \vec(\blambda)}_s
\end{align} subject to 
\begin{align}
\nonumber
&\lrnorm{\blambda_i}_q \le 1, 1\le i \le m.
\end{align} Once the minimizer $\blambda_\ast$ of \eqref{equa_unware_regularized_dual_objective} has been obtained, the minimizer of the primal problem \eqref{equa_basic_simultaneous_clustering_optimization_primal} is obtained according to \eqref{equa_minimizer_primal}, i.e., $2\lrincir{\vec(\X) - \vec(\A)} + (\I_d\otimes \Q)\Tr\vec(\blambda_\ast) = 0$ in this case.

\textbf{Ridge regression.} The objective function of ridge regression with the sum-of-norms regularizer is
\begin{align}
\nonumber
\min_{\X\in\RR^{n\times d}} & \sum_{i=1}^n \lrincir{\lrnorm{\A_i\X\Tr_i - \y_i}^2_2 + \gamma \lrnorm{\X_i}_2^2} + \lrnorm{\Q\X}_{1,p}.
\end{align} After the constant item is discarded, the above becomes equivalent to
\begin{align}
\nonumber
\min_{\X\in\RR^{n\times {d}}} & \vec\Tr(\X) \bOmega \vec(\X) - 2\y\Tr\bLambda\vec(\X)  + \lrnorm{\Q\X}_{1,p}
\end{align} where $\bOmega\in\RR^{(nd)\times (nd)}$ is defined by
\begin{align}
\nonumber
\bOmega := \text{diag}(\vec(\A))\text{diag}(\vec(\A))+\gamma \I_{nd},
\end{align} and $\bLambda \in\RR^{n \times (nd)}$ is defined by
\begin{align}
\nonumber
\bLambda := (\1_{1\times d}\otimes \I_n)\text{diag}(\vec(\A)).
\end{align} 

%Thus,
%\begin{align}
%\nonumber
%f(\X;\A) = \vec\Tr(\X)\bOmega \vec(\X) - 2\y\Tr\bLambda\vec(\X),
%\end{align} and 
%\begin{align}
%\nonumber
%& f^{\ast}(-\vec\Tr(\blambda)(\I_{d}\otimes \Q); \A) \\ \nonumber 
%=& \frac{1}{4} \vec\Tr(\blambda)(\I_{d}\otimes \Q) \bOmega^{-1}(\I_{d}\otimes \Q)\Tr \vec(\blambda) \\ \nonumber 
%&+\frac{1}{2}\vec\Tr(\blambda)(\I_{d}\otimes \Q)\bOmega^{-1}\bLambda\y.
%\end{align}  

Therefore, the final dual problem of ridge regression is formulated as
\begin{align}
\nonumber
\min_{\blambda\in\RR^{m\times d}} & \frac{1}{4} \vec\Tr(\blambda)(\I_{d}\otimes \Q) \bOmega^{-1}(\I_{d}\otimes \Q)\Tr \vec(\blambda) \\ \nonumber 
& + \frac{1}{2}\vec\Tr(\blambda)(\I_{d}\otimes \Q)\bOmega^{-1}\bLambda\y + \beta \lrnorm{(\I_{d}\otimes \Q)\Tr \vec(\blambda)}_s
\end{align} subject to
\begin{align}
\nonumber
\lrnorm{\blambda_i}_q \le 1, 1\le i \le m.
\end{align} Once the minimizer $\blambda_\ast$ of \eqref{equa_unware_regularized_dual_objective} has been obtained, the minimizer of the primal problem \eqref{equa_basic_simultaneous_clustering_optimization_primal} is obtained according to \eqref{equa_minimizer_primal}, i.e., $2\bOmega\vec(\X_\ast) -2\bLambda\Tr\y + (\I_d\otimes \Q)\Tr\vec(\blambda_\ast) = 0$ in this case.

\section{New variant of ADMM for solving \eqref{equa_unware_regularized_dual_objective}}
\label{sect_admm}
In this section, we describe a novel ADMM method for efficiently solving the proposed formulation, i.e., \eqref{equa_unware_regularized_dual_objective}.

\begin{algorithm}[t]
    \caption{New variant of ADMM for solving  \eqref{equa_unware_regularized_dual_objective}}
    \label{algo_admm}
    \begin{algorithmic}[1]
        \Require Data matrix $\A\in\mathbb{R}^{n\times d}$ and a positive $\beta$.
        \State Initialize $\blambda^{(0)}$ and $\bmu^{(0)}$, and set $t=0$.
        \For {Stopping criterion is not satisfied}
            \State Update $\blambda^{(t+1)}$ by solving the optimization problem \eqref{equa_admm_update_lambda}.
            \State Update $\u^{(t+1)}$ according to \eqref{equa_update_u}.
            \State Update $\bmu^{(t+1)}$ according to \eqref{equa_admm_mu_update}.
            \State $t=t+1$;
        \EndFor
        \State\Return The final value of $\blambda$.
    \end{algorithmic}
\end{algorithm}

\subsection{Algorithmic framework}

The constraint set $\Ccal$ of \eqref{equa_unware_regularized_dual_objective} is defined by 
\begin{align}
\nonumber
\Ccal = \{\blambda | \lrnorm{\blambda_i}_q \le 1, 1\le i\le m\}. 
\end{align} Similarly, a new notation $\u\in\RR^{md\times 1}$ is defined by 
\begin{align}
\nonumber
\u := (\I_d\otimes \Q)\Tr \vec(\blambda).
\end{align} Thus, \eqref{equa_unware_regularized_dual_objective} is reformulated as 
\begin{align}
\nonumber
\min_{\blambda\in\Ccal} & f^{\ast}(-(\I_d\otimes \Q)\Tr \vec(\blambda);\A) + \beta \lrnorm{\u}_s
\end{align} subject to
\begin{align}
\nonumber
(\I_d\otimes \Q)\Tr \vec(\blambda)-\u = 0.
\end{align} Denote $h(\blambda) := f^{\ast}(-(\I_d\otimes \Q)\Tr \vec(\blambda);\A)$. Recall that $g(\u) = \lrnorm{\u}_{1,p}$. Its augmented Lagrangian multiplier is 
\begin{align}
\nonumber
L(\blambda, \u, \bmu) = & h(\blambda) + g(\u) + \bmu\Tr \lrincir{(\I_d\otimes \Q)\Tr \vec(\blambda)-\u} \\ \nonumber 
&+ \frac{\rho}{2}\lrnorm{(\I_d\otimes \Q)\Tr \vec(\blambda)-\u}_2^2.
\end{align}

\textbf{Update $\blambda$.} The $(t+1)$-th update of $\blambda$ is 
\begin{align}
\nonumber
\blambda^{(t+1)} = & \argmin_{\blambda\in\Ccal} L(\blambda, \u^{(t)}, \bmu^{(t)}) \\ \nonumber
= & \argmin_{\blambda\in\Ccal} h(\blambda) +  \lrincir{\vec\Tr(\blambda) (\I_d\otimes \Q) }\bmu^{(t)} \\ \label{equa_admm_update_lambda}
& + \frac{\rho}{2}\lrnorm{(\I_d\otimes \Q)\Tr \vec(\blambda)-\u^{(t)}}_2^2.
\end{align}

It is worth noting that most of the previous studies \cite{Chi:2013ey,Hallac:2015fy,Lindsten2011Just,Yuan2018An,icml2017} investigate the case of  $p=q=2$. However, we observe that the case of $p=1$ and $q=\infty$ is more efficient than that of $p=q=2$. If $p=1$ and $q=\infty$, the update of $\blambda$ can be naturally parallelized. We present the details in the appendix. 
\begin{remark}
The parallelization speed-up of the update of $\blambda$ is up to $O(d)$ if $p=1$ and $q=\infty$.  
\end{remark} 

\textbf{Update $\u$.} The $(t+1)$-th update of $\u$ is formulated as follows:
\begin{theorem}
\label{theorem_admm_update_u}
\begin{align}
\label{equa_update_u}
\u^{(t+1)} = \bomega^+ \odot \bomega,
\end{align} where 
\begin{align}
\nonumber
\bomega :=  \frac{1}{\rho}\bmu^{(t)} + (\I_d\otimes \Q)\Tr \vec(\blambda^{(t+1)}).
\end{align}
\end{theorem}
According to Theorem \ref{theorem_admm_update_u}, the update of $\u$ has a closed form that can be computed very efficiently.

\textbf{Update $\bmu$.} The $(t+1)$-th update of $\bmu$ is simply 
\begin{align}
\label{equa_admm_mu_update}
\bmu^{(t+1)} = \bmu\Tr + \rho \lrincir{(\I_d\otimes\Q)\vec(\blambda^{(t+1)}) - \u^{(t+1)}}.
\end{align}

The new variant of ADMM is presented in Algorithm \ref{algo_admm}.

\subsection{Complexity analysis} If $f(\cdot)$ is convex, ADMM has been proven to be convergent \cite{Boyd:2011}. A new norm $\lrnorm{\w}_\H$ is defined by
\begin{align}
\nonumber
\lrnorm{\w}_\H := \w\Tr \H \w,
\end{align} where $\w\in\RR^{(nd+2md)}$ is defined by 
\begin{align}
\nonumber
\w := \lrincir{\vec\Tr(\blambda);\u;\bmu},
\end{align} and $\H$ is defined by
\begin{align}
\nonumber
\H := \begin{pmatrix}
 \mathbf{0}_{nd\times nd}&  & \\ 
 & \rho \I_{md} & \\ 
 &  & \frac{1}{\rho} \I_{md}.
\end{pmatrix}
\end{align} Therefore, we obtain the following theorem \cite{He:2015gq}:
\begin{theorem}[Theorem $5.1$ in \cite{He:2015gq}]
\nonumber
Assume that $f(\cdot)$ is convex. If the ADMM is run for $T$ iterations, we have 
\begin{align}
\nonumber
\lrnorm{\w^T - \w^{T+1}}_\H^2 \le \frac{1}{T+1} \lrnorm{\w^0 - \w_\ast}_\H^2
\end{align} where $\w_\ast=\lrincir{\vec\Tr(\blambda_\ast);\u_\ast;\bmu_\ast}$.  
\end{theorem} It means that the convergence rate of the ADMM method is $O(\frac{1}{T})$ \cite{He:2015gq}. Formally, the quality of the solution obtained by the ADMM method is presented as follows. 
\begin{remark}
\nonumber
If we want to obtain an $\epsilon$-accurate ($\epsilon>0$) solution, i.e., $\lrnorm{\w^T - \w^{T+1}}_\H^2 \le \epsilon$, the ADMM method needs to perform $O(\frac{1}{\epsilon})$ iterations.
\end{remark}

\section{Theoretical analysis}
\label{sect_theoretical_analysis}
In this section, we analyze the accuracy of the approximate solution obtained by Algorithm \ref{algo_evolution_detected_regularization} for two specific tasks: convex clustering and ridge regression.

\subsection{Meta framework for analysis}
We define the constraints of \eqref{equa_unware_regularized_dual_objective}, i.e., $\Ccal$ as  
\begin{align}
\nonumber
\Ccal = \{\blambda\in\RR^{m\times d} | \lrnorm{\blambda_i}_q \le 1, {~} 1\le i \le m \}.
\end{align}

Define the minimizer of \eqref{equa_unware_regularized_dual_objective} based on the evolving dataset $\A+\bDelta$ as
\begin{align}
\nonumber
\tilde{\blambda}_\ast :=& \argmin_{\blambda\in\Ccal} f^{\ast}(-\vec\Tr(\blambda)(\I_d\otimes \Q); \A+\bDelta) \\ \nonumber 
&+ \beta \lrnorm{(\I_d\otimes\Q)\vec(\blambda)}_s.
\end{align}

Recall that $\blambda_\ast$ is defined by
\begin{align}
\nonumber
\blambda_\ast =& \argmin_{\blambda\in\Ccal} f^{\ast}(-\vec\Tr(\blambda)(\I_d\otimes \Q); \A) \\ \nonumber 
&+ \beta \lrnorm{(\I_d\otimes\Q)\vec(\blambda)}_s.
\end{align}

According to Algorithm \ref{algo_evolution_detected_regularization}, we have
\begin{align}
\nonumber
&f^{\ast}(-\vec\Tr(\tilde{\blambda}_\ast)(\I_d\otimes \Q); \A+\bDelta) \\ \nonumber 
\le &f^{\ast}(-\vec\Tr(\blambda_\ast)(\I_d\otimes \Q); \A+\bDelta) \\ \label{equa_difference_lambda_hat_control_c}
\le & f^{\ast}(-\vec\Tr(\blambda_\ast)(\I_d\otimes \Q); \A) + c.
\end{align} This inequality provides a method for bounding the difference between $\blambda_\ast$ and $\tilde{\blambda}_\ast$. Additionally, according to \eqref{equa_minimizer_primal}, we have
\begin{align}
\nonumber
&\nabla f(\tilde{\X}_\ast; \A+\bDelta) - \nabla f(\X_\ast; \A) \\ \label{equa_lambda_hat_difference} 
=& (\I_d\otimes \Q)\Tr\lrincir{\tilde{\blambda}_\ast - \blambda_\ast}.
\end{align} This equality transforms the bound on the difference between $\tilde{\X}_\ast$ and $\X_\ast$ into the bound on the difference between $\blambda_\ast$ and $\tilde{\blambda}_\ast$. Therefore, by combining \eqref{equa_difference_lambda_hat_control_c} and \eqref{equa_lambda_hat_difference}, we are able to bound the difference between $\tilde{\X}_\ast$ and $\X_\ast$. 
\begin{remark}
The difference between $\tilde{\X}_\ast$ and $\X_\ast$ may be bounded by \eqref{equa_difference_lambda_hat_control_c} and \eqref{equa_lambda_hat_difference}.
\end{remark}

\subsection{Guarantee of model accuracy for convex clustering} 

In convex clustering, $f(\X;\A) = \lrnorm{\X-\A}_F^2$, $\nabla f(\X;\A) = 2(\vec(\X) - \vec(\A))$, and $f^{\ast}(-\vec\Tr(\blambda)(\I_d\otimes \Q)\Tr;\A)= -\vec\Tr(\A)(\I_d\otimes \Q)\Tr \vec(\blambda) + \frac{1}{4} \vec\Tr(\blambda)(\I_d\otimes \Q)(\I_d\otimes \Q)\Tr\vec(\blambda)$. The difference between $\tilde{\X}_\ast$ and $\X_\ast$ is bounded according to the following theorem:
\begin{theorem}
\label{theorem_x_bound_convex_clustering}
When the data matrix $\A$ evolves to be $\A+\bDelta$, and Algorithm \ref{algo_evolution_detected_regularization} is run to perform convex clustering, the difference between $\X_\ast$ obtained by Algorithm \ref{algo_evolution_detected_regularization} based on $\A$ and the true solution $\tilde{\X}_\ast$ based on $\A+\bDelta$ is bounded as
\begin{align}
\nonumber
& \vec\Tr(\A)\lrincir{\vec(\tilde{\X}_\ast) - \vec(\X_\ast)} \\ \nonumber 
\le & \vec\Tr(\A)\vec(\bDelta) + \frac{c}{2} \\ \nonumber
& + \frac{1}{2\beta}\lrnorm{\vec(\bDelta)}\vec\Tr(\A+\bDelta)\vec(\A+\bDelta).
\end{align}
\end{theorem}

\subsection{Guarantee of model accuracy for ridge regression} 

In Section \ref{subsect_formulation_example}, we showed that $f(\X;\A) = \vec\Tr(\X)\bOmega \vec(\X)-2\y\Tr\bLambda \X$ and $f^{\ast}(-\vec\Tr(\blambda)(\I_{d}\otimes \Q); \A) = \frac{1}{4} \vec\Tr(\blambda)(\I_{d}\otimes \Q) \bOmega^{-1}(\I_{d}\otimes \Q)\Tr \vec(\blambda)+\frac{1}{2} \vec\Tr(\blambda)(\I_d\otimes \Q)\bOmega^{-1}\bLambda\y$ for the ridge regression task. Define two new notations, $\tilde{\bOmega} \in \RR^{(nd) \times (nd)}$ and $\tilde{\bLambda} \in \RR^{n \times (nd)}$, as 
\begin{align}
\nonumber
\tilde{\bOmega} := \text{diag}(\vec(\A+\bDelta))\text{diag}(\vec(\A+\bDelta)) + \gamma \I_{n d },
\end{align} and  
\begin{align}
\nonumber
\tilde{\bLambda} := (\mathbf{1}_{1\times d}\otimes \I_n)\text{diag}(\vec(\A+\bLambda)).
\end{align}

According to \eqref{equa_difference_lambda_hat_control_c}, we have
\begin{align}
\nonumber
c \ge & \frac{1}{4}\vec\Tr(\tilde{\blambda}_\ast)(\I_{d}\otimes \Q)\tilde{\bOmega}^{-1}(\I_{d}\otimes \Q)\Tr \vec(\tilde{\blambda}_\ast) \\ \label{equa_x_bound_ridge_regression}
& - \frac{1}{4}\vec\Tr(\blambda)(\I_{d}\otimes \Q)\bOmega^{-1}(\I_{d}\otimes \Q)\Tr \vec(\blambda).
\end{align} 

Define a new notation $\bPhi\in\RR^{(nd)\times (nd)}$ as 
\begin{align}
\nonumber
\bPhi:= 2\text{diag}(\vec(\bDelta))\lrincir{\mathbf{1}_{ d \times d }\otimes\I_n}\text{diag}(\vec(\A)).
\end{align}  

\begin{theorem}
\label{theorem_x_bound_ridge_regression}
When the data matrix $\A$ evolves to be $\A+\bDelta$, and Algorithm \ref{algo_evolution_detected_regularization} is run to perform ridge regression, the difference between $\X_\ast$ obtained by Algorithm \ref{algo_evolution_detected_regularization} based on $\A$ and the true solution $\tilde{\X}_\ast$ based on $\A+\bDelta$ is bounded as
\begin{align}
\nonumber
&\vec\Tr(\tilde{\X}_\ast)\bOmega \vec(\tilde{\X}_\ast) - \vec\Tr(\X_\ast)\bOmega\vec(\X_\ast) \\ \nonumber
\le &  \frac{1}{16\beta^2}\lrnorm{\y\Tr\tilde{\bOmega}^{-1}\tilde{\bLambda}\y}^2 \lrnorm{\tilde{\bOmega}^{-1}\bPhi \tilde{\bOmega}^{-1}} \\ \nonumber 
& + \frac{1}{16\beta^2}\lrnorm{\y\Tr\bOmega^{-1}\bLambda\y}^2\lrnorm{\bOmega^{-1} \bPhi\bOmega^{-1}}+4c.
\end{align}
\end{theorem}

%\frac{1}{4\beta}\y\Tr\tilde{\bOmega}^{-1}\bLambda\y
\begin{remark}
Given $\A$ and $\bDelta$, it is possible to obtain a relatively accurate solution for convex clustering and ridge regression by setting a small $c$ and a large $\beta$ in Algorithm  \ref{algo_evolution_detected_regularization}.
\end{remark}

\begin{figure*}[!t]
\setlength{\abovecaptionskip}{0pt}
\setlength{\belowcaptionskip}{0pt}
\centering 
\subfigure[\textit{iris}]{\includegraphics[width=0.49\columnwidth]{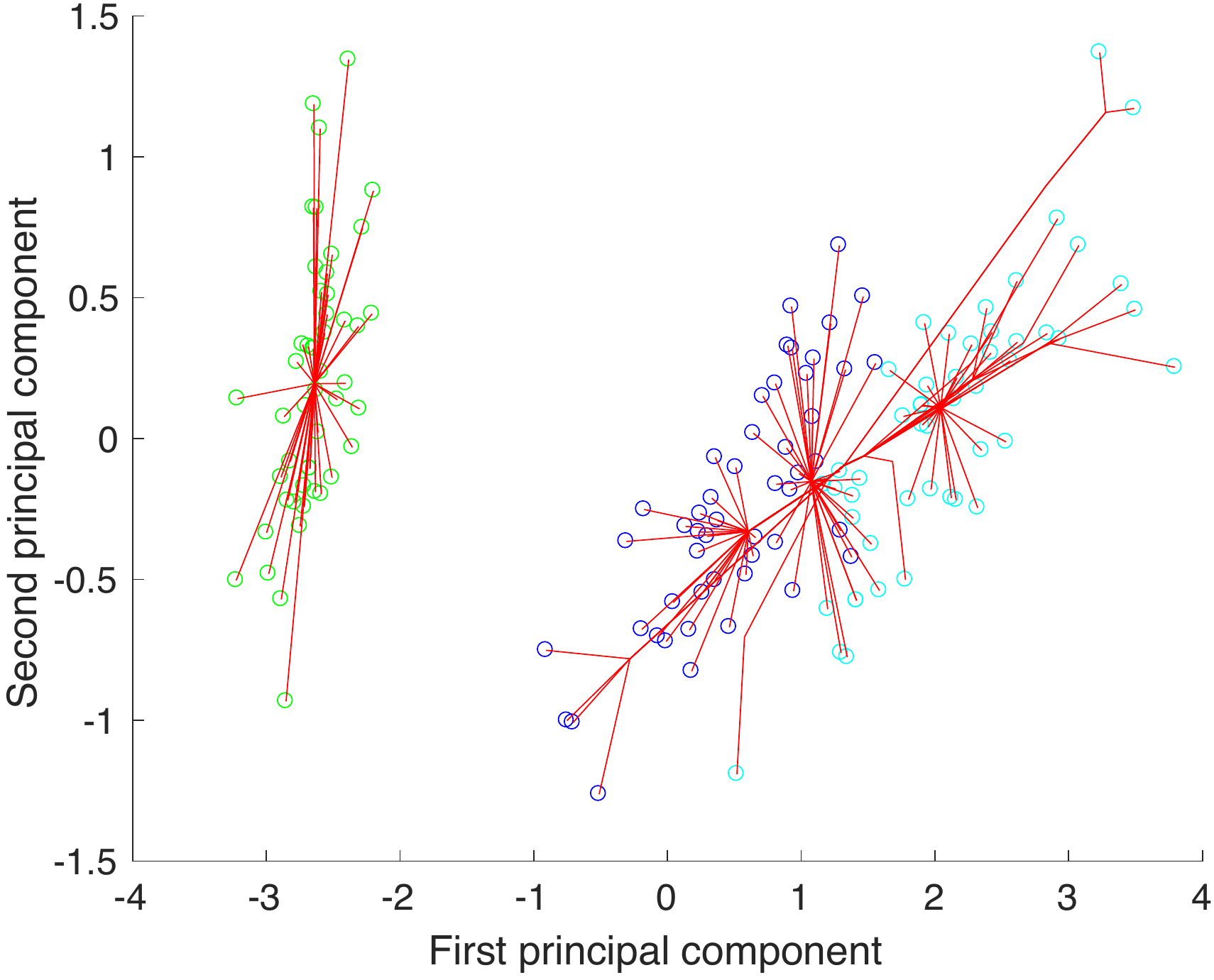}\label{figure_iris_clusterpath}}
\subfigure[\textit{moon}]{\includegraphics[width=0.49\columnwidth]{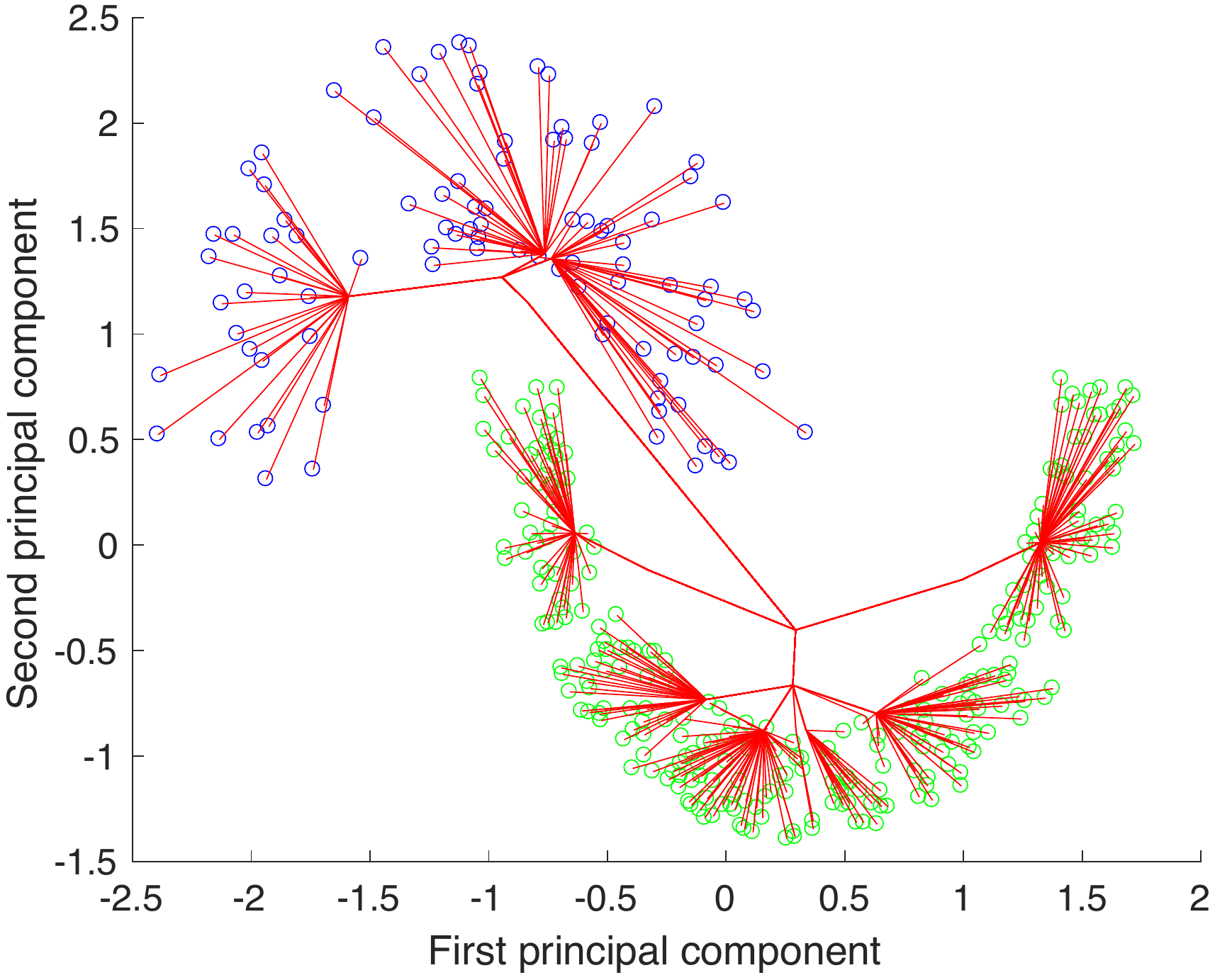}\label{figure_moons_clusterpath}}
\subfigure[\textit{segment}]{\includegraphics[width=0.48\columnwidth]{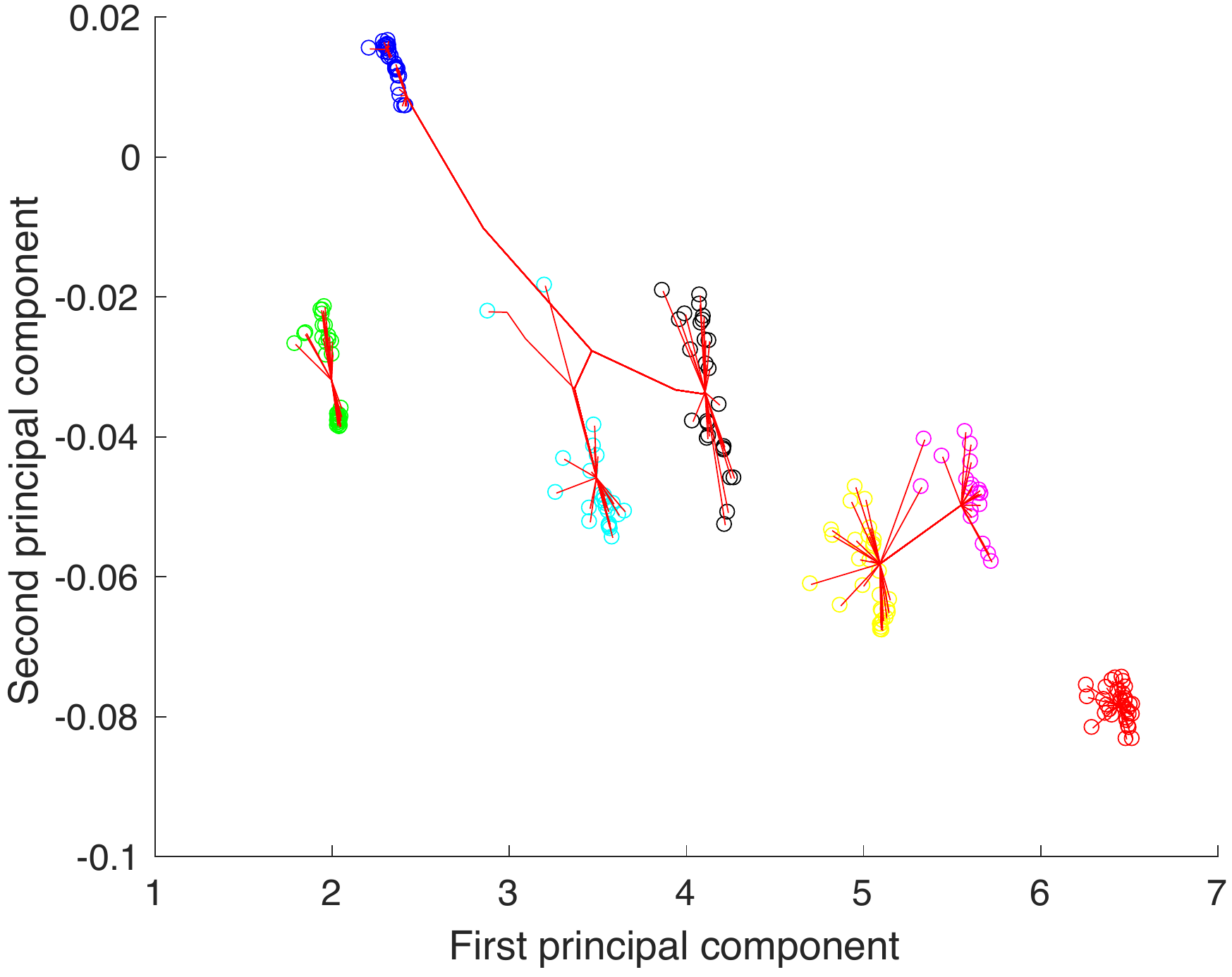}\label{figure_segment_clusterpath}}
\subfigure[\textit{svm-guide}]{\includegraphics[width=0.49\columnwidth]{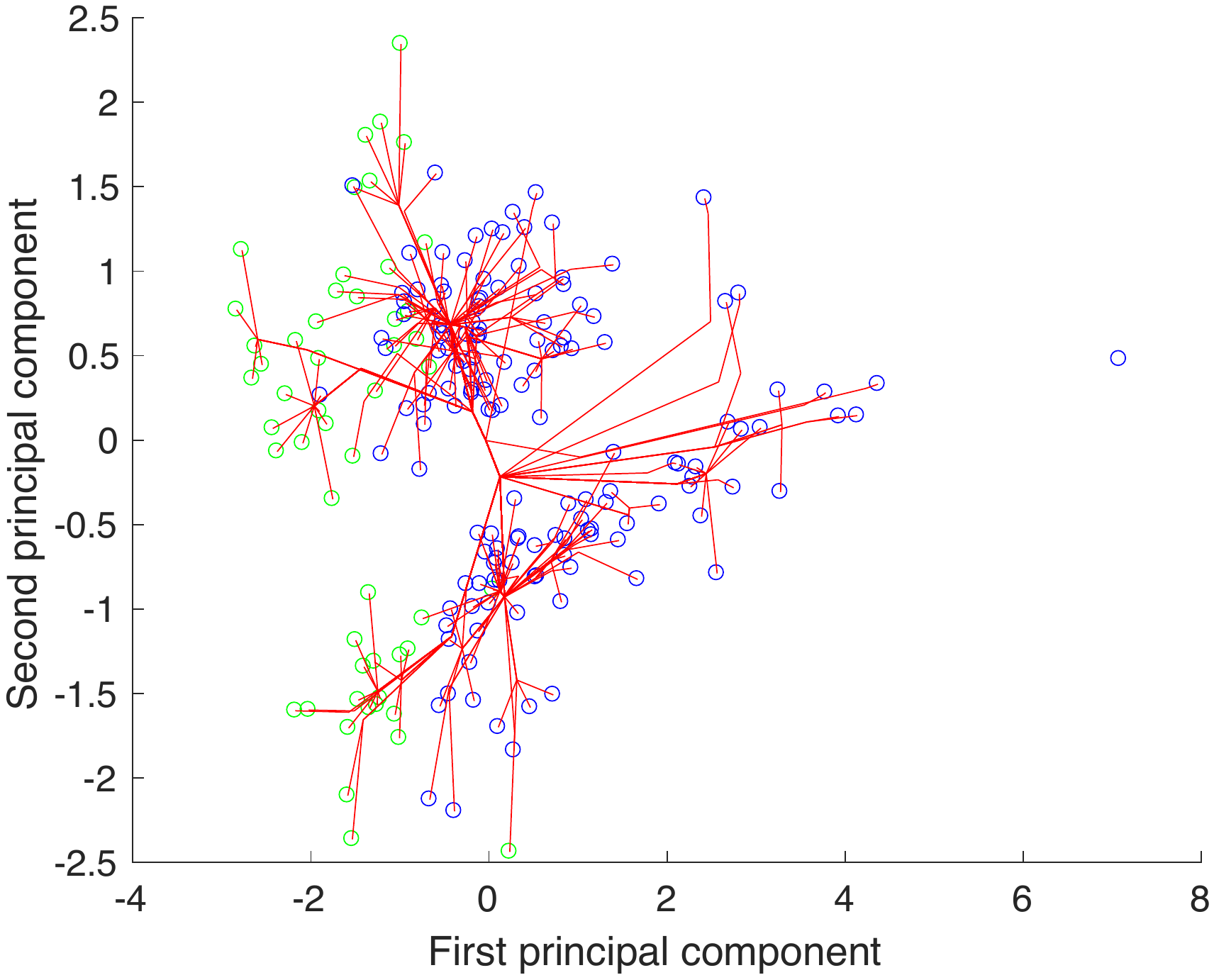}\label{figure_svmguide_clusterpath}}
\caption{Cluster paths obtained for the datasets \textit{iris}, \textit{moon}, \textit{segment} and \textit{svm-guide}. }
\label{figure_cluster_path}
\end{figure*}

\begin{figure*}[!t]
\setlength{\abovecaptionskip}{0pt}
\setlength{\belowcaptionskip}{0pt}
\centering 
\subfigure[\textit{iris}]{\includegraphics[width=0.49\columnwidth]{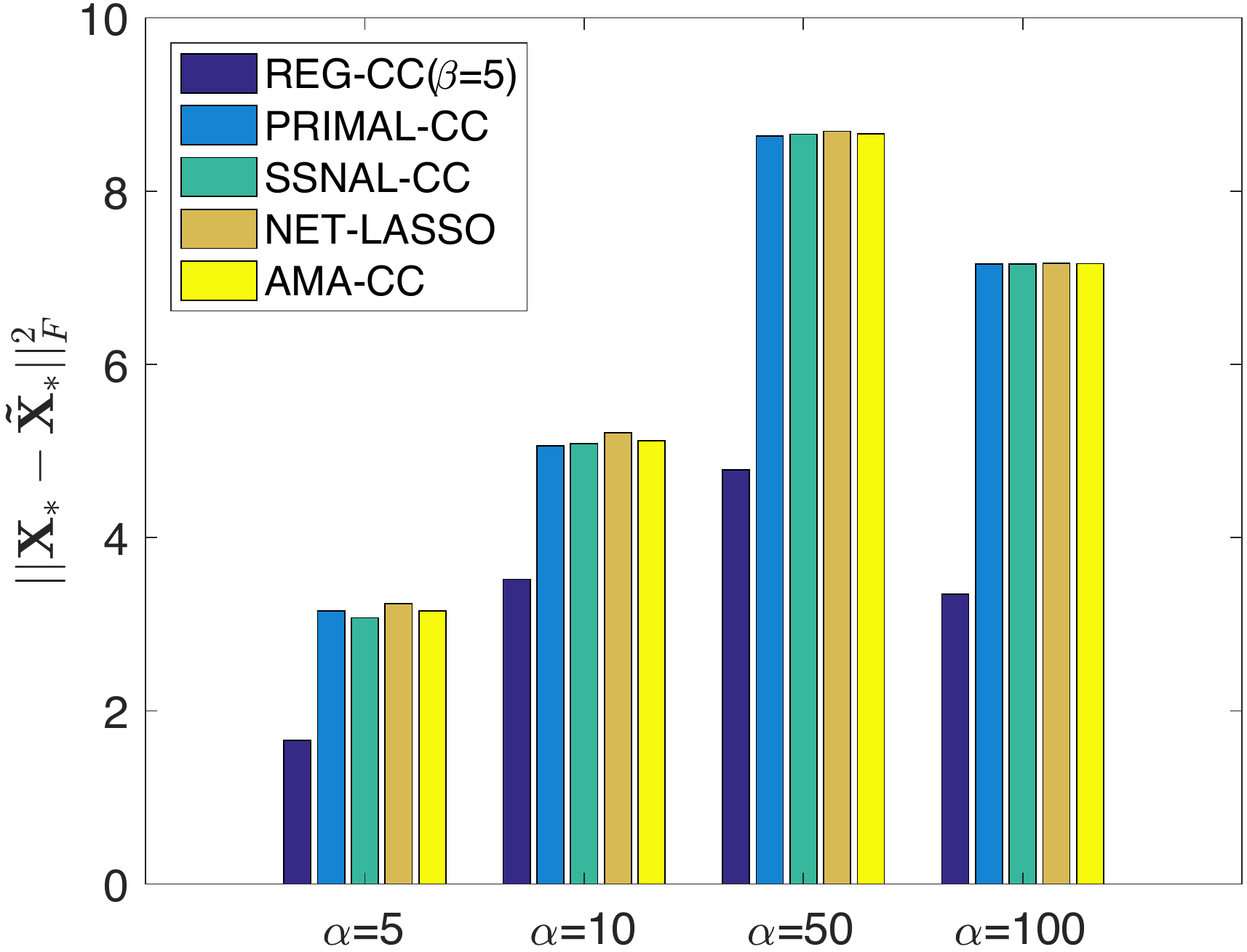}\label{figure_robustness_iris}}
\subfigure[\textit{moon}]{\includegraphics[width=0.49\columnwidth]{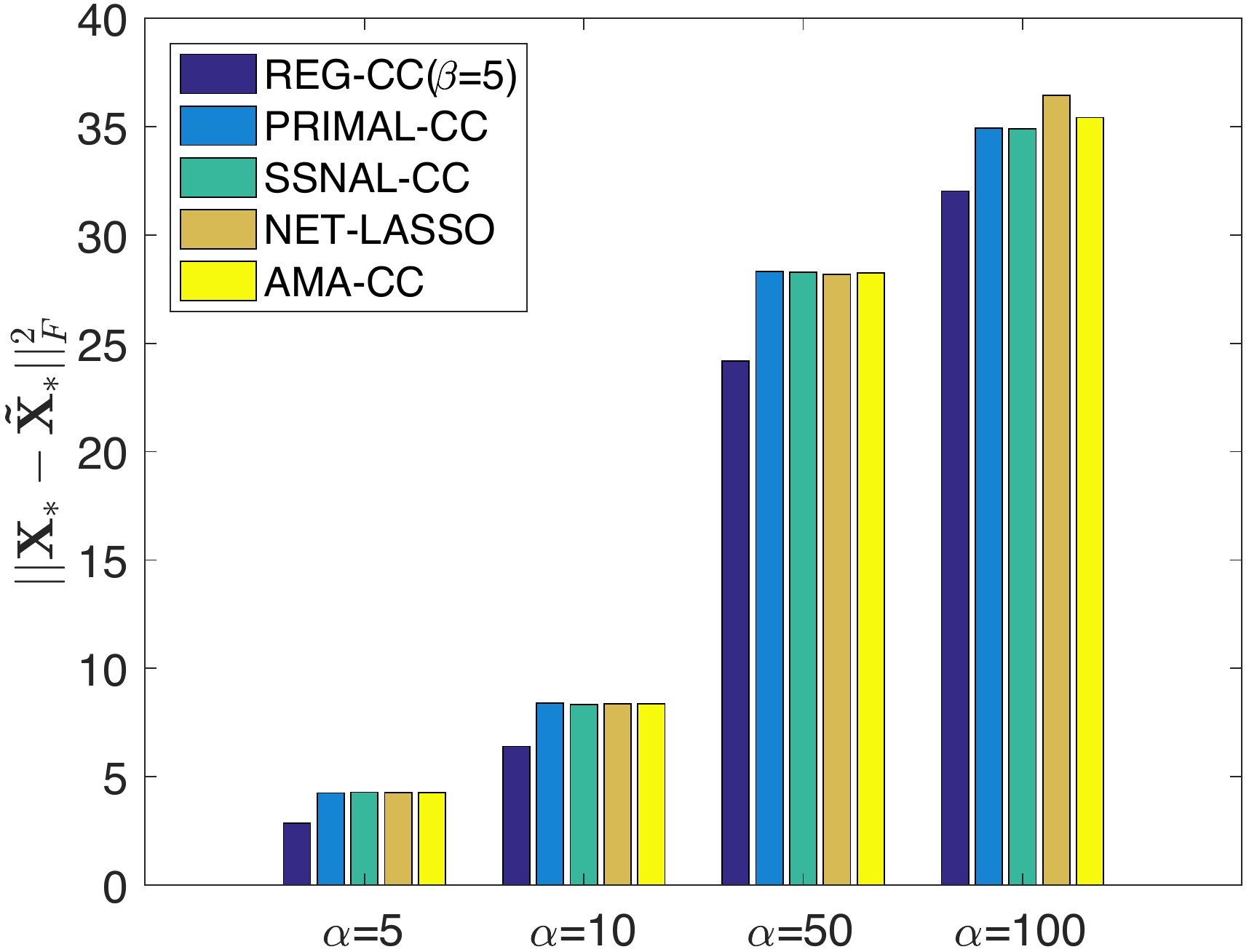}\label{figure_robustness_moons}}
\subfigure[\textit{segment}]{\includegraphics[width=0.49\columnwidth]{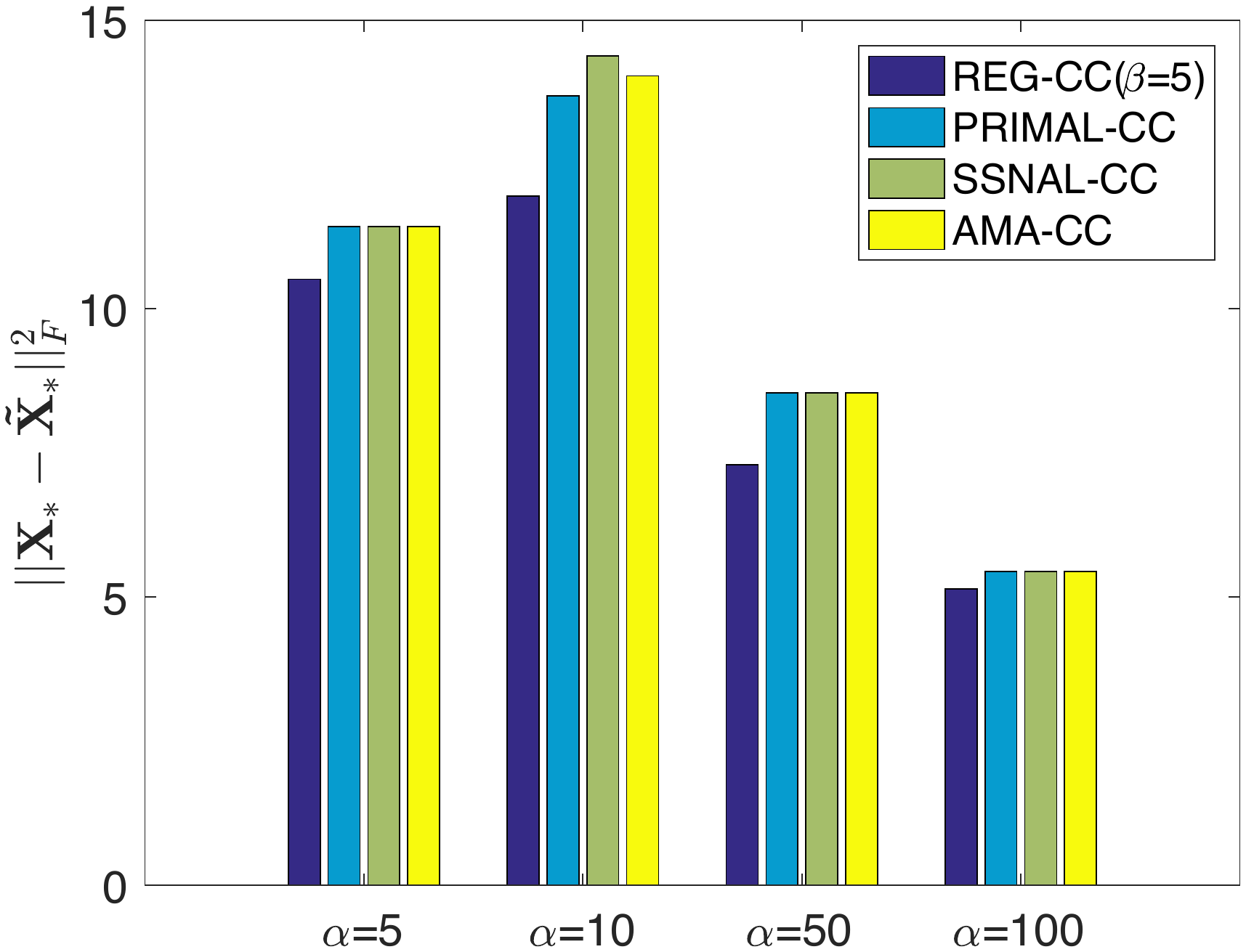}\label{figure_robustness_segment}}
\subfigure[\textit{svm-guide}]{\includegraphics[width=0.49\columnwidth]{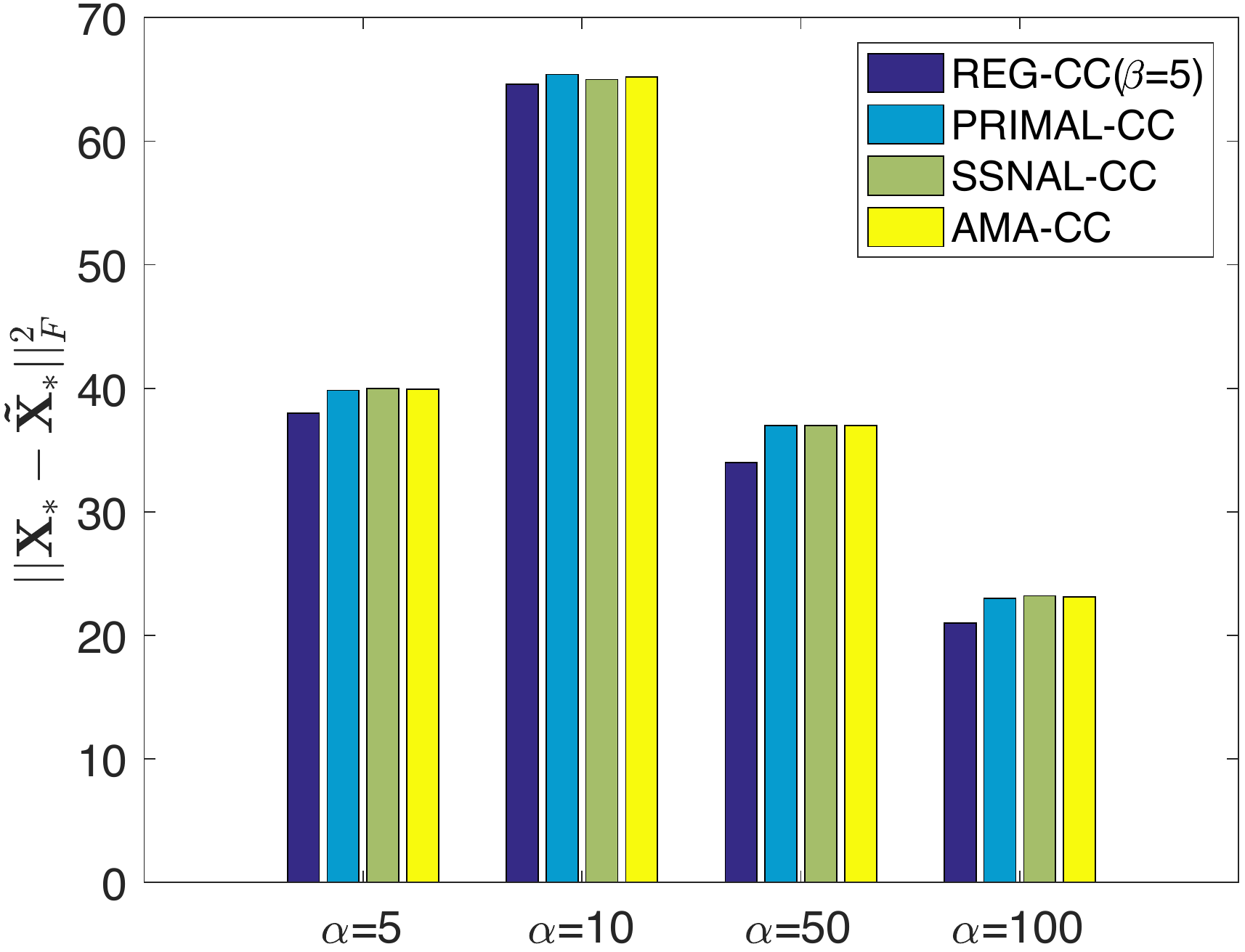}\label{figure_robustness_svmguide}}
\caption{When datasets evolve, our REG-CC method yields more accurate clustering results than those of its counterparts by varying $\alpha$.}
\label{figure_robustness_convex_clustering}
\end{figure*}

\section{Empirical studies}
\label{sect_experiment}
In the section, we first present the experimental settings. Then, we perform convex clustering, obtain the cluster path for video datasets, and perform the ridge regression tasks to evaluate our method. 

\subsection{Experimental settings}
The experiments are performed on a server equipped with an Intel Xeon(R) 32-core E5-2610 CPU and 48 GB of RAM. We implement all the algorithms by using Matlab 2016b and the CVX solver \cite{cvx_project}. 

We use six datasets: \textit{moon}\footnote{https://cs.joensuu.fi/sipu/datasets/jain.txt}, \textit{iris}\footnote{https://www.csie.ntu.edu.tw/$\sim$ cjlin/libsvmtools/datasets/multi\\ class.html\#iris}, \textit{segment}\footnote{https://www.csie.ntu.edu.tw/$\sim$ cjlin/libsvmtools/datasets/multi\\ class.html\#\textit{segment}}, \textit{svm-guide}\footnote{https://www.csie.ntu.edu.tw/$\sim$ cjlin/libsvmtools/datasets/binar\\ y.html\#svmguide1}, \textit{space\_ga}\footnote{https://www.csie.ntu.edu.tw/$\sim$ cjlin/libsvmtools/datasets/regre\\ ssion.html\#space\_ga}, and \textit{airfoil}\footnote{http://archive.ics.uci.edu/ml/datasets/Airfoil+Self-Noise}.  Given graph $\Gcal$, the weight for every edge is set to be inversely proportional to the distance between the vertices; this approach is widely used in \cite{Chi:2013ey,Tan:2015vr,Chen:2015kn}. Additionally, our method is compared with many existing methods. These existing methods are presented briefly as follows:
\begin{itemize}
\item PRIMAL-CC \cite{Lindsten2011Just} has been proposed to perform convex clustering. It solves the primal problem, i.e., \eqref{equa_basic_simultaneous_clustering_optimization_primal}, directly.
\item SSNAL-CC \cite{Yuan2018An} is a semi-smooth Newton-based augmented Lagrangian method. It has been proposed to perform convex clustering efficiently.
\item NET-LASSO \cite{Hallac:2015fy} has been proposed as a general framework for performing SCO for various tasks, e.g., convex clustering and ridge regression. 
\item AMA-CC \cite{Chi:2013ey} is an alternating minimization algorithm. It has been proposed to perform convex clustering efficiently.
\end{itemize}
We follow the above notation, whereby the minimizer of \eqref{equa_unware_regularized_dual_objective} for $\A$ is denoted by $\X_\ast$, and the true minimizer of \eqref{equa_unware_regularized_dual_objective} for $\A+\bDelta$ is denoted by $\tilde{\X}_\ast$. Their difference, i.e., $\lrnorm{\X_\ast - \tilde{\X}_\ast}_F^2$, is used to measure the approximation of $\X_\ast$ against $\tilde{\X}_\ast$. Here, $\lrnorm{\cdot}_F$ represents the Frobenius norm. The threshold $c$ in Algorithm \ref{algo_evolution_detected_regularization} is set to be $c=10$ by default.

\subsection{Convex clustering}

\begin{figure*}[!t]
\setlength{\abovecaptionskip}{0pt}
\setlength{\belowcaptionskip}{0pt}
\centering 
\subfigure[\textit{iris}]{\includegraphics[width=0.49\columnwidth]{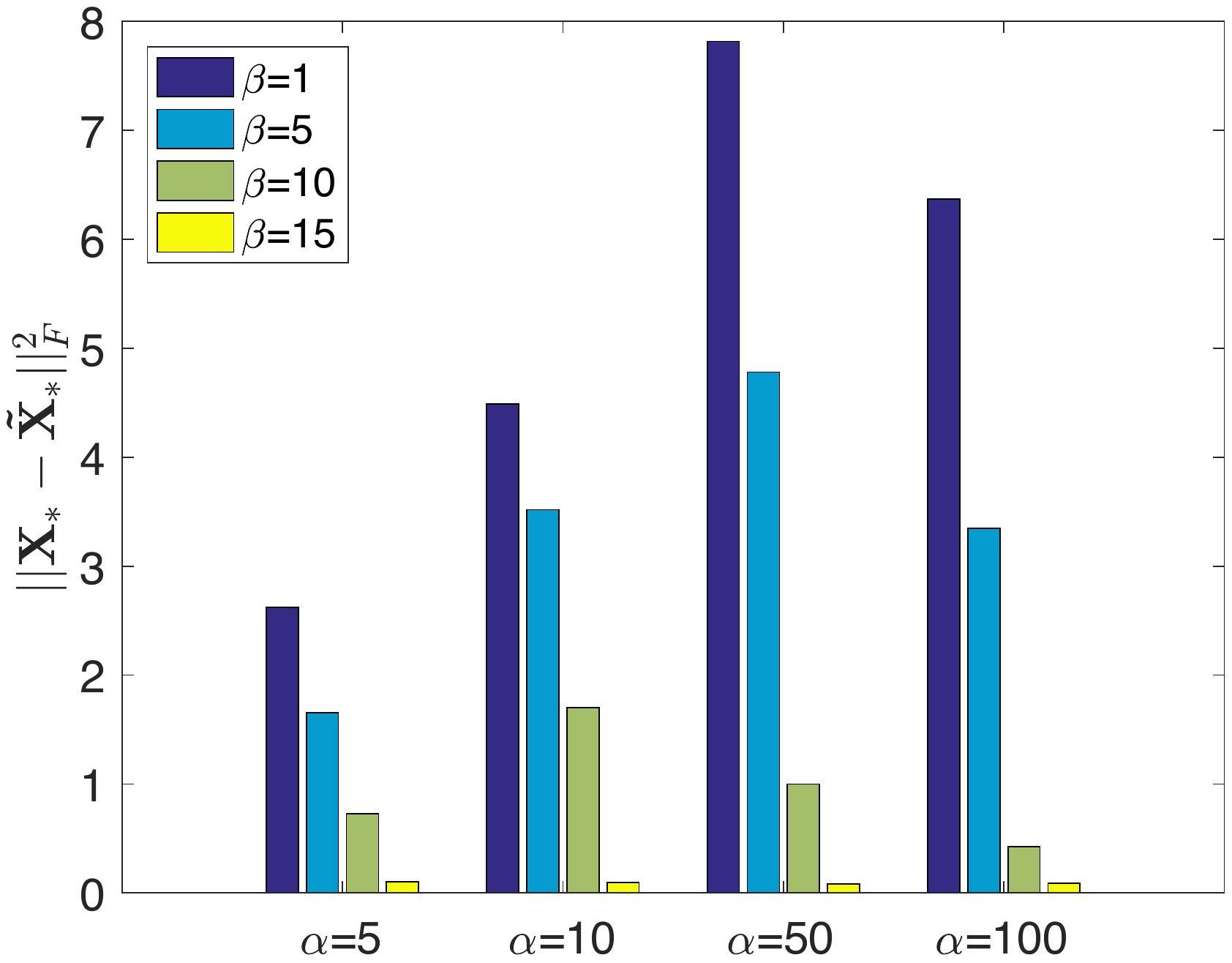}\label{figure_robustness_iris_beta}}
\subfigure[\textit{moon}]{\includegraphics[width=0.49\columnwidth]{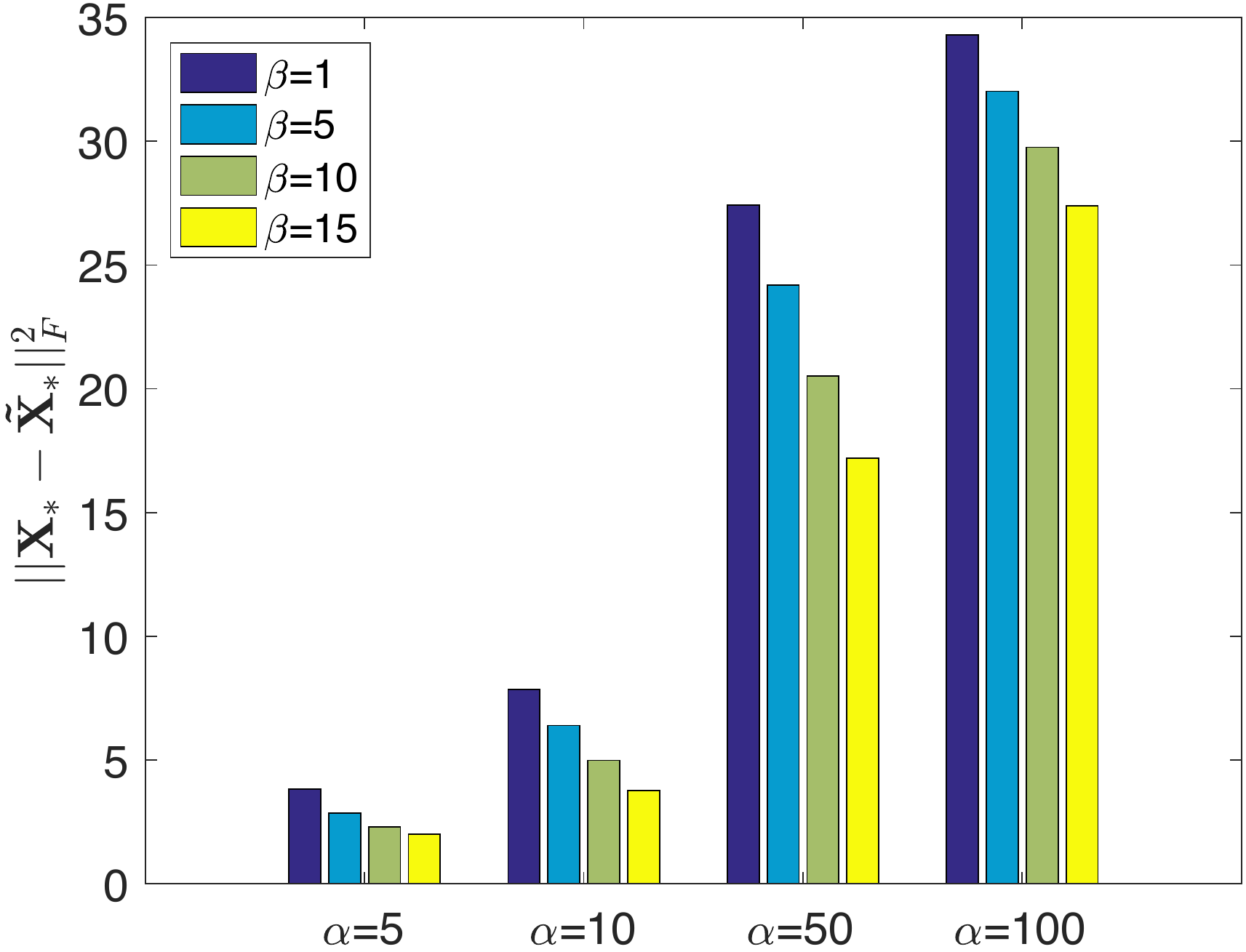}\label{figure_robustness_moons_beta}}
\subfigure[\textit{segment}]{\includegraphics[width=0.49\columnwidth]{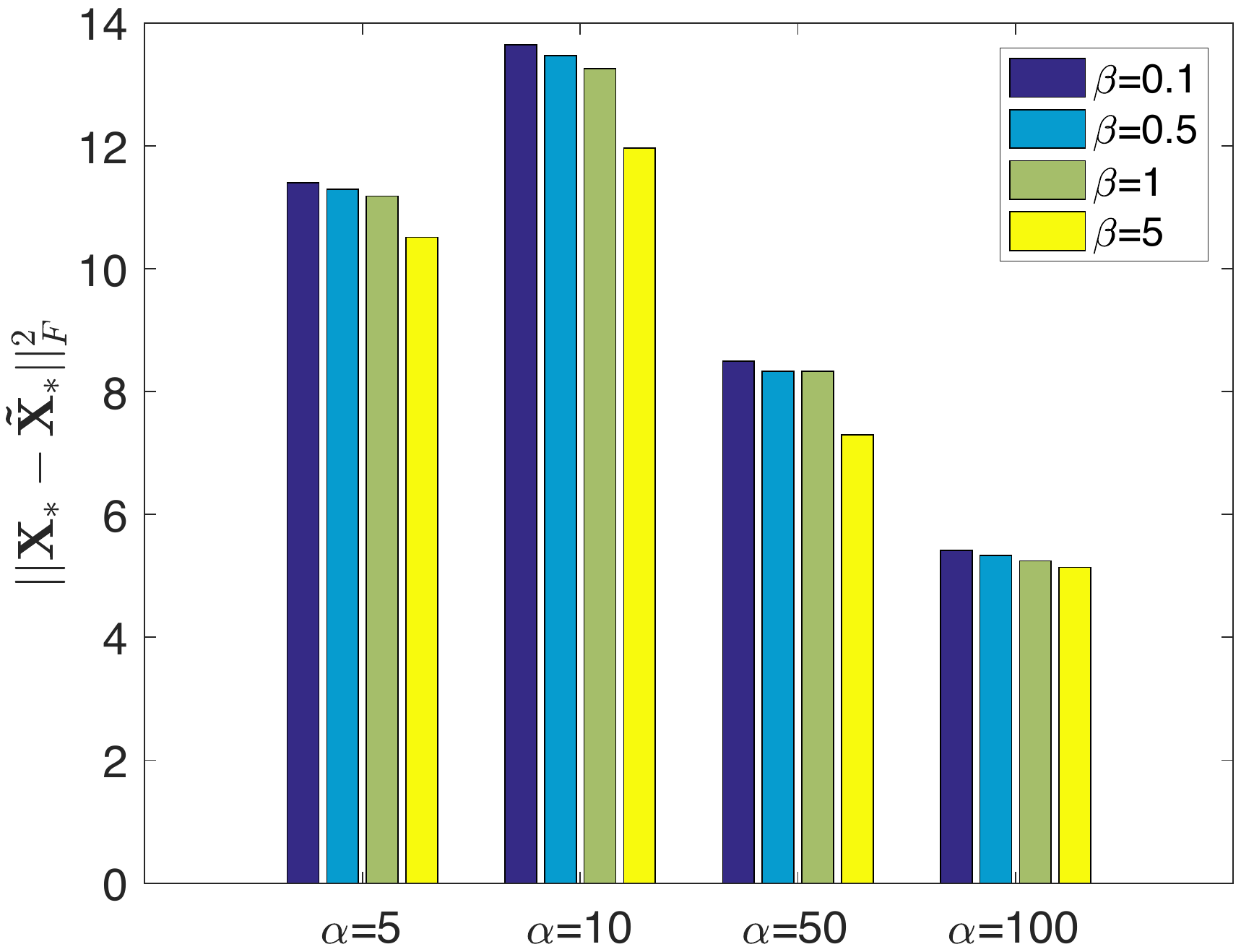}\label{figure_robustness_segment_beta}}
\subfigure[\textit{svm-guide}]{\includegraphics[width=0.49\columnwidth]{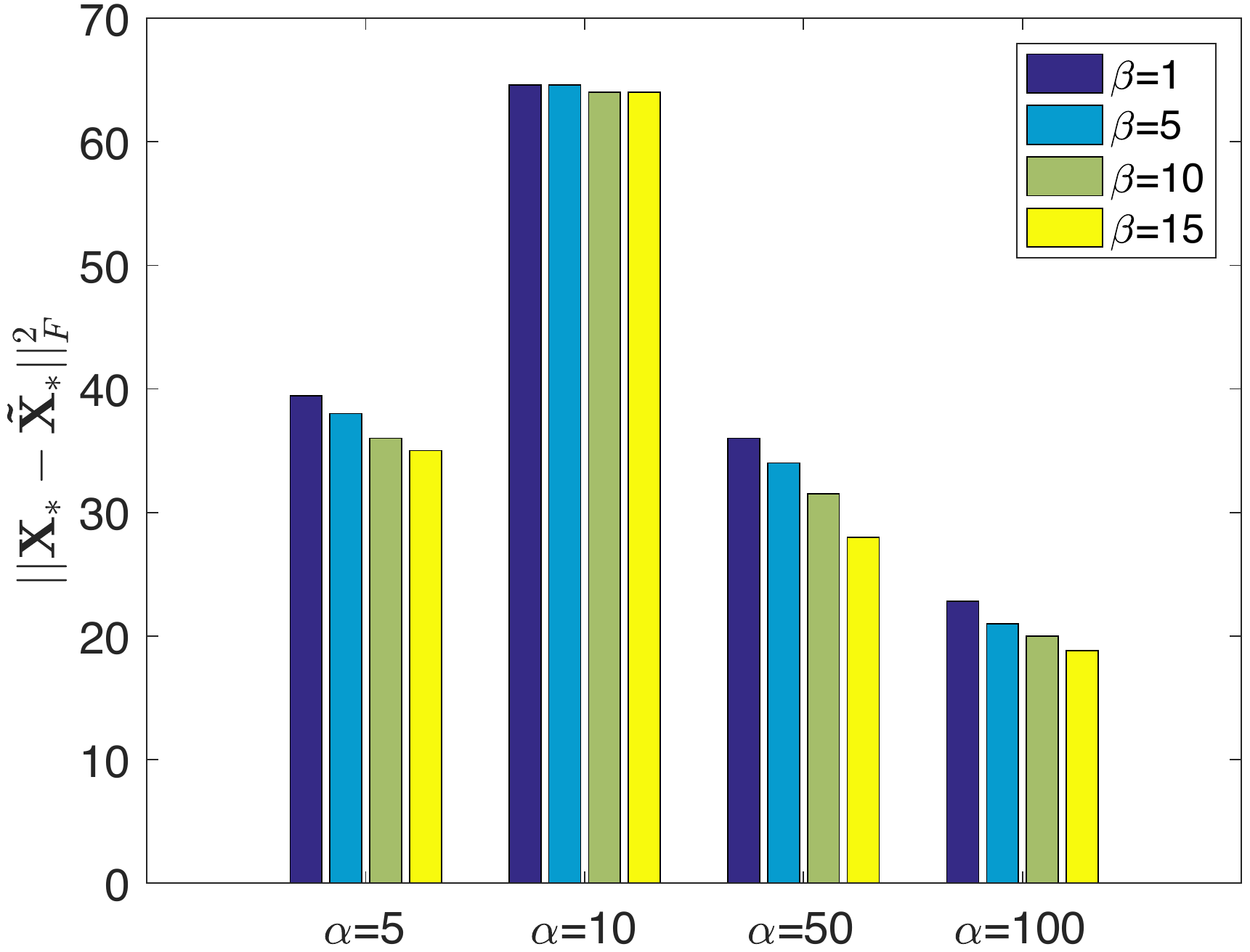}\label{figure_robustness_svmguide_beta}}
\caption{When datasets evolve, our REG-CC method recovers the clustering membership more accurately with a large $\beta$.}
\label{figure_robustness_beta}
\end{figure*}

\begin{figure*}[!t]
\setlength{\abovecaptionskip}{0pt}
\setlength{\belowcaptionskip}{0pt}
\centering 
\subfigure[\textit{iris}]{\includegraphics[width=0.49\columnwidth]{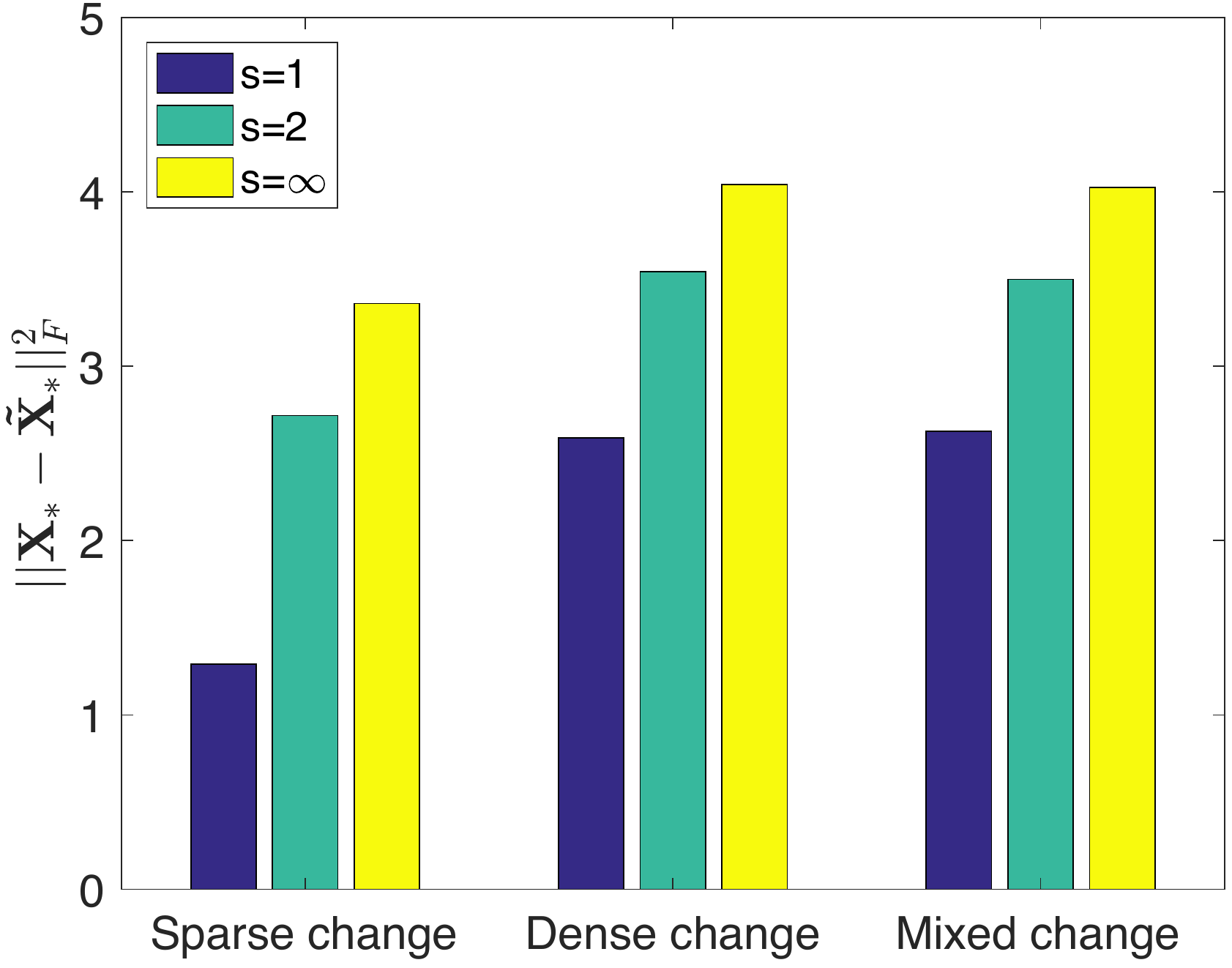}\label{figure_robustness_noise_type_iris}}
\subfigure[\textit{moon}]{\includegraphics[width=0.49\columnwidth]{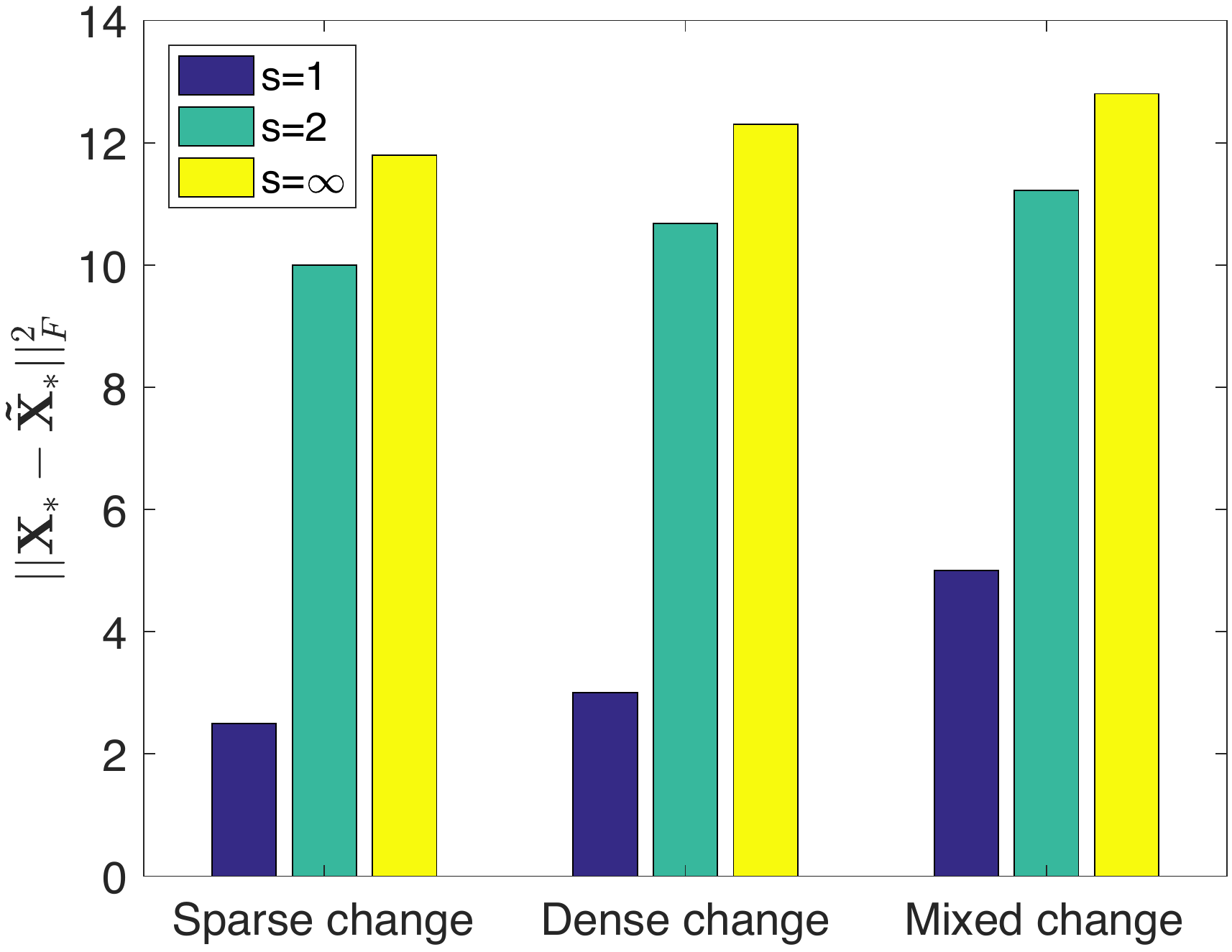}\label{figure_robustness_noise_type_moons}}
\subfigure[\textit{segment}]{\includegraphics[width=0.49\columnwidth]{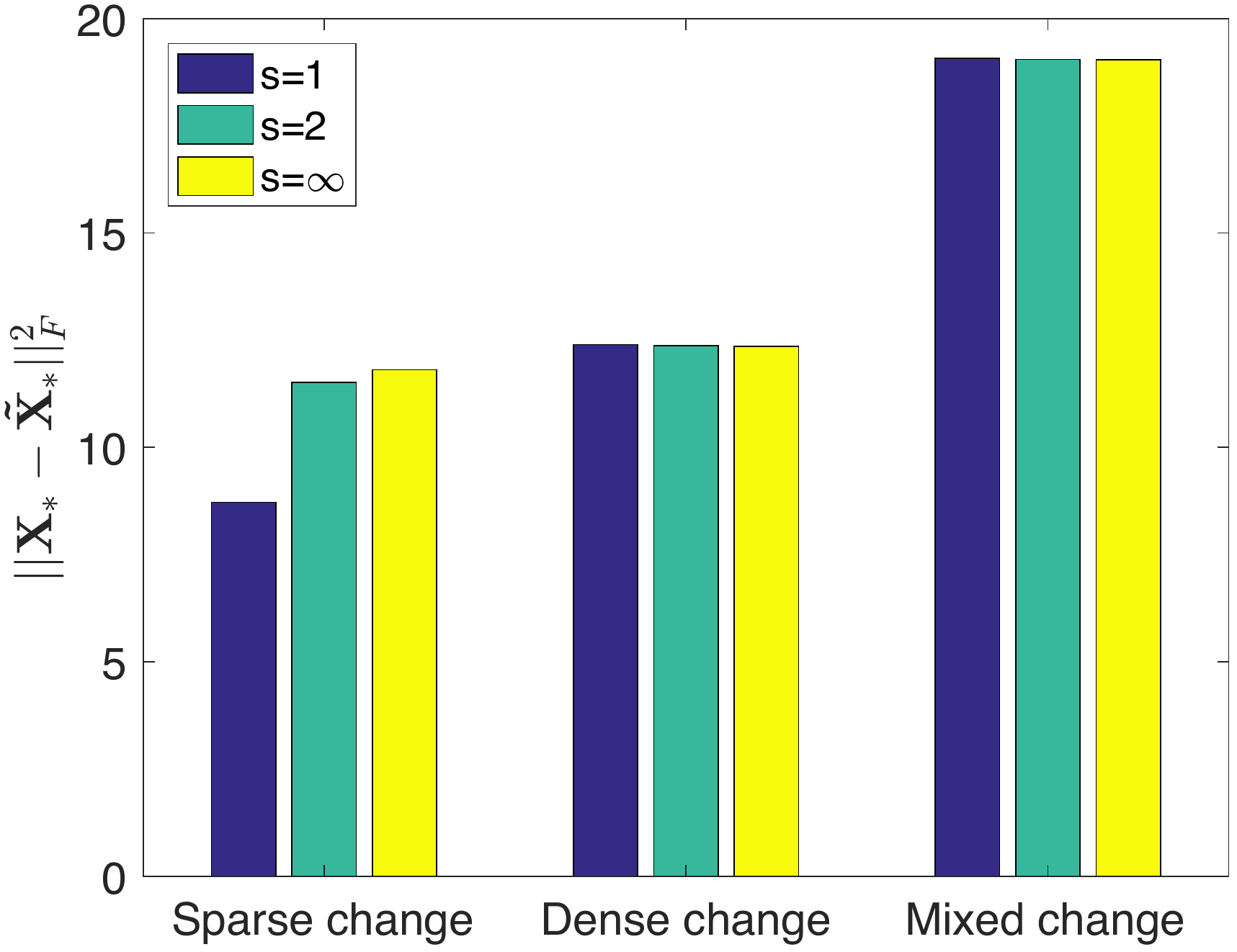}\label{figure_robustness_noise_type_segment}}
\subfigure[\textit{svm-guide}]{\includegraphics[width=0.49\columnwidth]{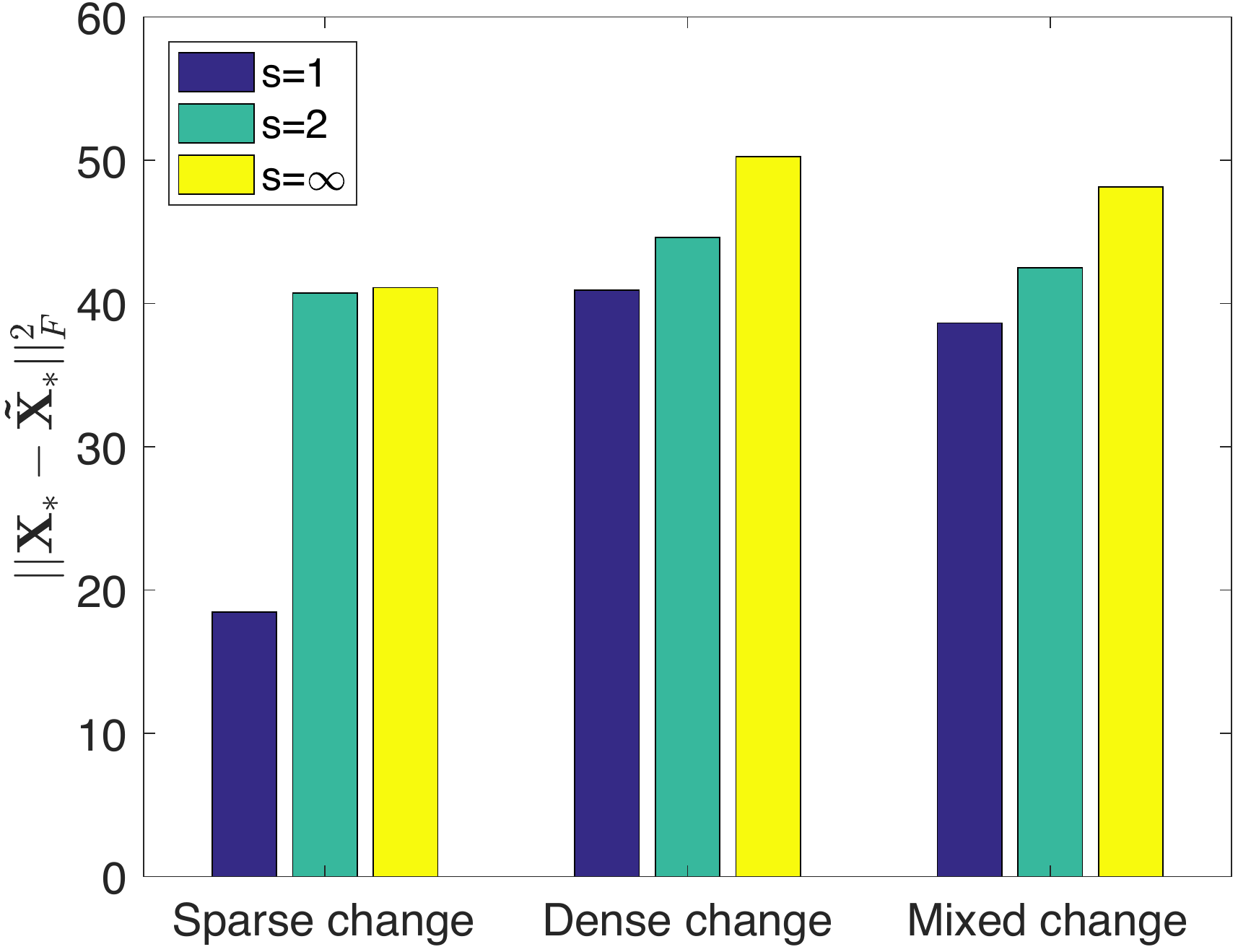}\label{figure_robustness_noise_type_svmguide}}
\caption{When the environment evolves, our REG-CC method performs best at the setting of $s=1$.}
\label{figure_robustness_noise_type}
\end{figure*}

\begin{table*}
\centering
\caption{Average CPU time, in seconds, of running the proposed REG-CC method by varying constraints. The numbers in parentheses represent the respective variances.}
\begin{tabular}{c|c|c|c|c}
\hline 
dataset/algos & REG-CC$(l_1)$ & REG-CC$(l_2)$ & REG-CC$(l_{\infty})$ & parallel REG-CC$(l_{\infty})$\tabularnewline
\hline 
\textit{iris} & $5.53(0.05)$ & $24.83(0.81)$ & $29.86(4.04)$ & $\mathbf{1.20(0.0004)}$\tabularnewline
\hline 
\textit{moon} & $4.26(0.02)$ & $33.25(0.85)$ & $29.50(0.71)$ & $\mathbf{2.36(0.0137)}$\tabularnewline
\hline 
\textit{segment} & $54.29(11.94)$ & $106.05(4.60)$ & $147.16(42.03)$ & $\mathbf{6.44(0.06)}$\tabularnewline
\hline 
\textit{svm-guide} & $318.58(24.70)$ & $731.83(38.28)$ & $1209.35(5.44)$ & $\mathbf{20.32(0.04)}$\tabularnewline
\hline 
\end{tabular}
\label{table_efficiency_convex_clustering}
\end{table*}

As we have shown in the previous section, $f(\X;\A) = \lrnorm{\X-\A}_F^2$ holds for convex clustering. Given a value $\alpha$, convex clustering recovers the clustering membership.  If an instance contains more than $2$ features, we present the cluster path by using its first and second principal components. The edges in graph $\Gcal$ are generated by running the $K$-NN method with $K=10$ for \textit{moon} and \textit{iris} and $K=5$ for \textit{segment} and \textit{svm-guide}. For datasets \textit{segment} and \textit{svm-guide}, we pick $200$ instances randomly and draw their cluster paths as illustrative examples. We increase $\alpha$ gradually and obtain a cluster path for every dataset. Additionally, we use the Gaussian distribution $N(0, 0.1^2)$ to generate the perturbed matrix $\bDelta$ that is used to simulate the evolving dataset. Parameter $\beta$ is set to $5$ by default in the evaluation.  Our method is denoted by \textit{REG-CC} for convex clustering. Note that the existing NET-LASSO method is unable to handle the datasets \textit{segment} and \textit{svm-guide} due to running out of memory. 

As illustrated in Figure \ref{figure_cluster_path}, every color represents a true cluster from the ground truth. If $\alpha$ is small, local communities of instances are detected.  As $\alpha$ increases, different local communities are fused into a large community. The number of clusters declines for a large $\alpha$. Furthermore, as illustrated in Figure \ref{figure_robustness_convex_clustering}, our \textit{REG-CC} method yields more accurate clustering results than do the existing methods for various values of $\alpha$. Due to the regularized item in the dual problem of \eqref{equa_unware_regularized_dual_objective}, our \textit{REG-CC} method penalizes a large fluctuation in $\X_\ast$. Thus, with the control of accuracy, the regularized item makes $\X_\ast$ robust to a perturbation of the data matrix. In other words, $\X_\ast$ is a satisfactory approximation of $\hat{\X}_\ast$ for the evolving dataset $\A+\bDelta$. Additionally, when we vary $\beta$, Figure \ref{figure_robustness_beta} shows that a large $\beta$ yields more accurate clustering results than does a small $\beta$.  The reason is that a large $\beta$ imposes a greater penalty on the fluctuation in $\X_\ast$, which makes $\X_\ast$ insensitive to the change in $\A$.

Additionally, we evaluate the performance of our method by varying $s$ in the regularizer. We test it in three kinds of evolving environments. In the first evolving environment, the perturbed matrix $\bDelta$ is sparse, and its nonzero values are generated from a Gaussian distribution, i.e., $N(0,0.1^2)$. This leads to a sparse change in the data matrix. In this case, the number of nonzero elements in the perturbed matrix is set to be $0.2n$, where $n$ is the number of instances in the dataset. In the second evolving environment, the perturbed matrix $\bDelta$ is dense, and its values are generated from the Gaussian distribution $N(0,0.01^2)$. This leads to a dense and relatively insignificant change in the data matrix. The third evolving environment generates the mixed perturbed matrix $\bDelta$ with $0.2n$ elements generated from the first case and the other elements generated from the second case. As illustrated in Figure \ref{figure_robustness_noise_type}, our method performs best in the setting of $s=1$ in all of these cases. In this setting, the regularizer favors obtaining a sparse $(\I_{d}\otimes\Q)\Tr\vec(\blambda_\ast)$ in the evolving environment. Although the data matrix changes in the evolving environment, the sparsity makes the minimizer $\blambda_\ast$ more robust to the change in the data matrix. 

Finally, we test the efficiency of our method by varying $q$. Various values of $q$ imply different types of constraints in the dual problem \eqref{equa_unware_regularized_dual_objective}. Variable $q$ is set to $1$, $2$ and $\infty$. In particular, if $q=\infty$, our method supports parallel computing naturally. It is implemented by using multiple cores.  According to Figure \ref{figure_robustness_noise_type}, performance at $s=1$ is usually better than at $s=2$ and at $\infty$. Thus, we set $s=1$ in this experiment. For all settings, our algorithm is executed three times. We show the mean and the variance (shown as the number in parentheses) of the used CPU time in seconds in Table \ref{table_efficiency_convex_clustering}. As shown in Table \ref{table_efficiency_convex_clustering}, our method is most efficient at $q=\infty$, when it is executed in parallel.

\begin{figure*}[t]
\setlength{\abovecaptionskip}{0pt}
\setlength{\belowcaptionskip}{0pt}
\centering 
\includegraphics[width=0.49\columnwidth]{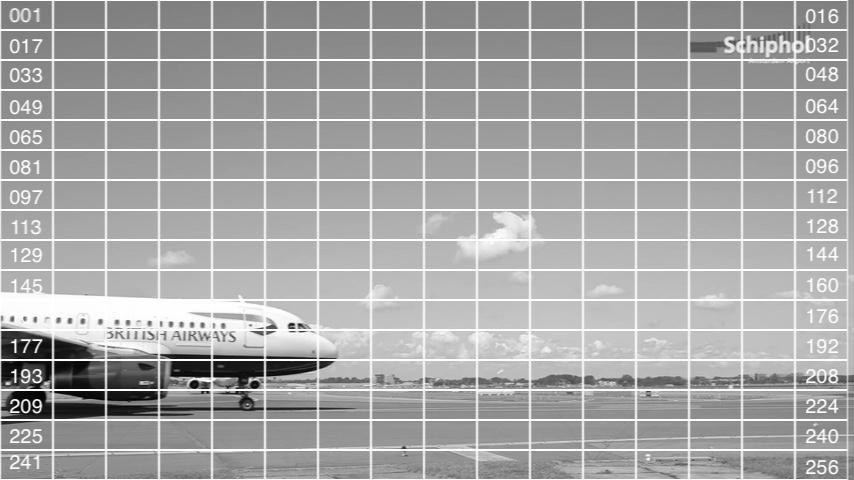}
\includegraphics[width=0.49\columnwidth]{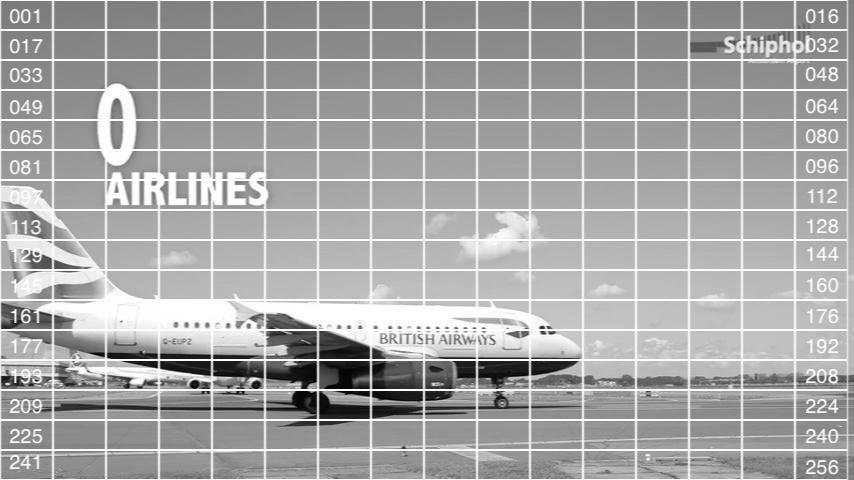}
\includegraphics[width=0.49\columnwidth]{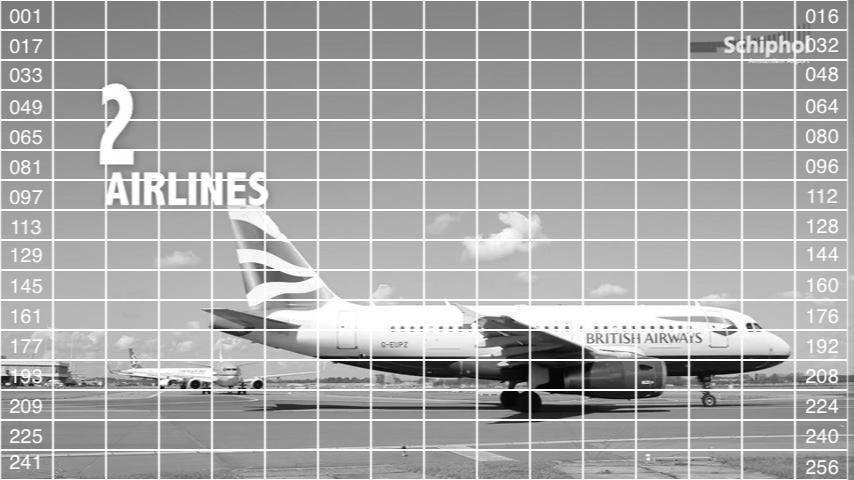}
\includegraphics[width=0.49\columnwidth]{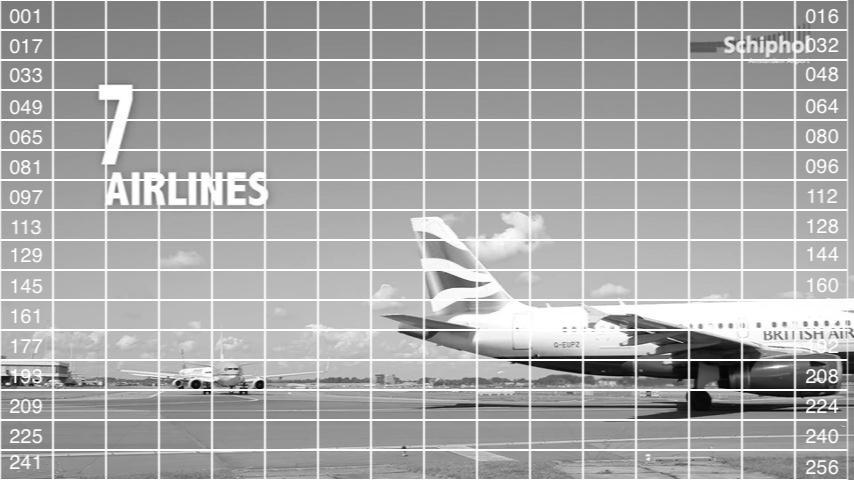}
\includegraphics[width=0.49\columnwidth]{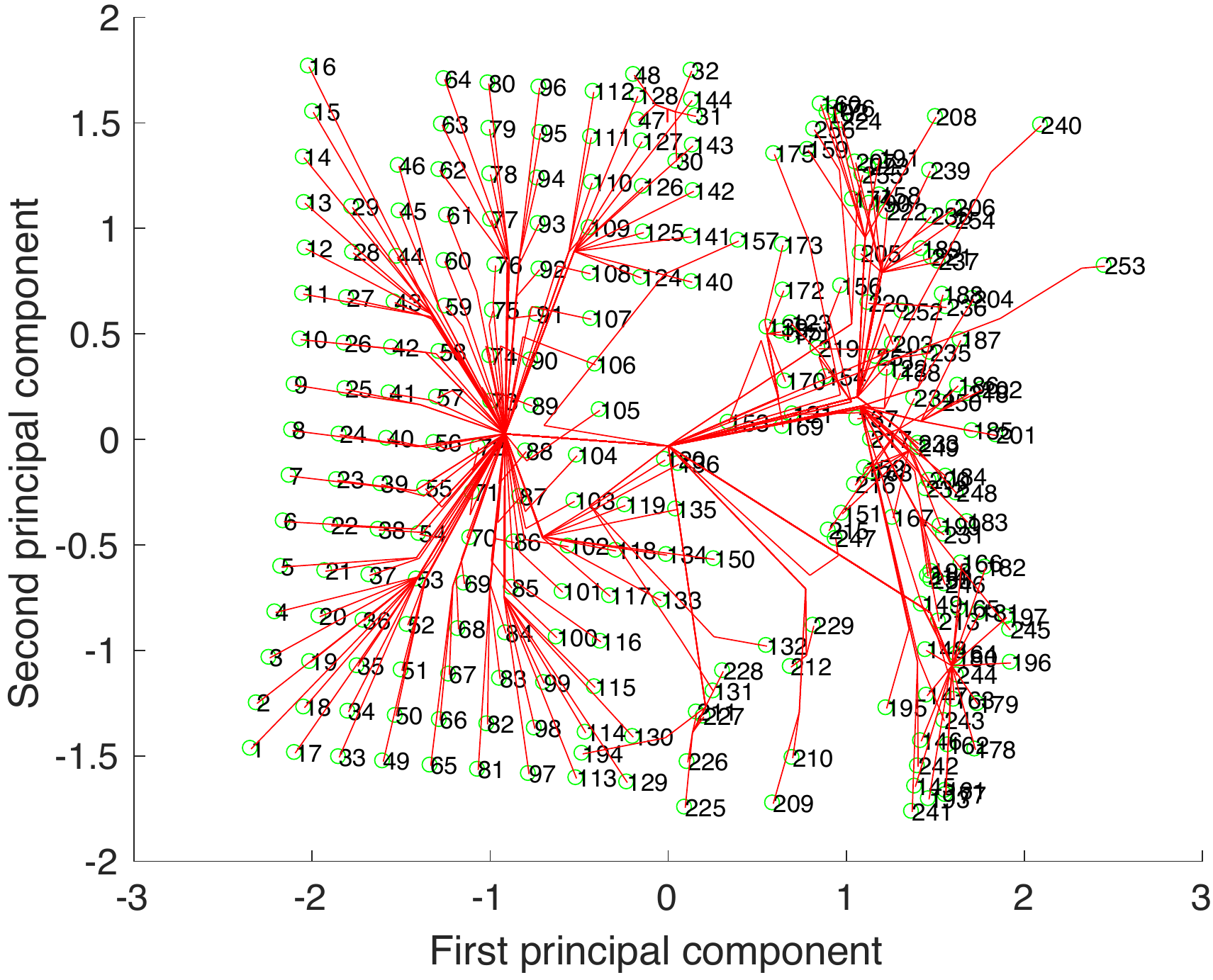}
\includegraphics[width=0.49\columnwidth]{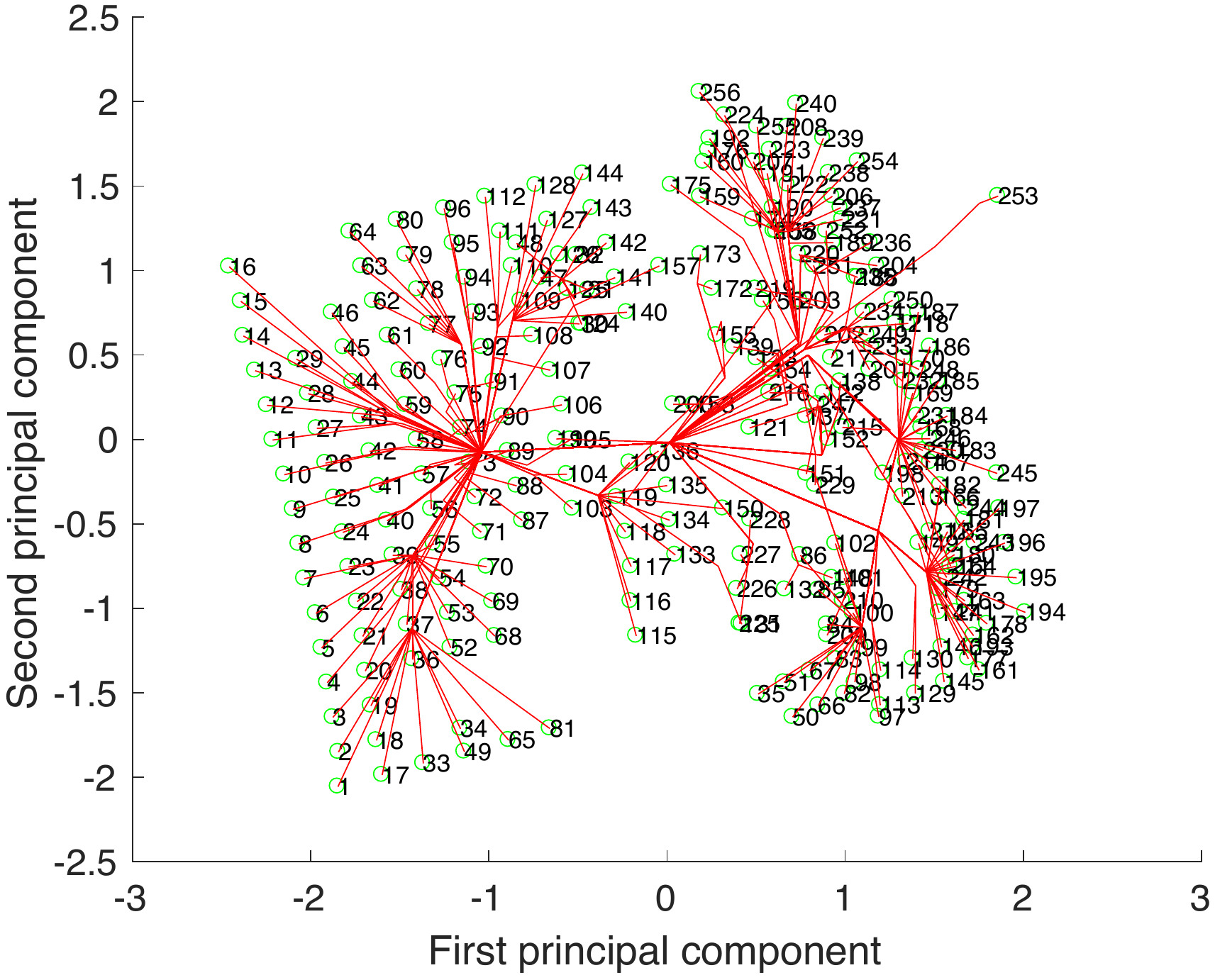}
\includegraphics[width=0.49\columnwidth]{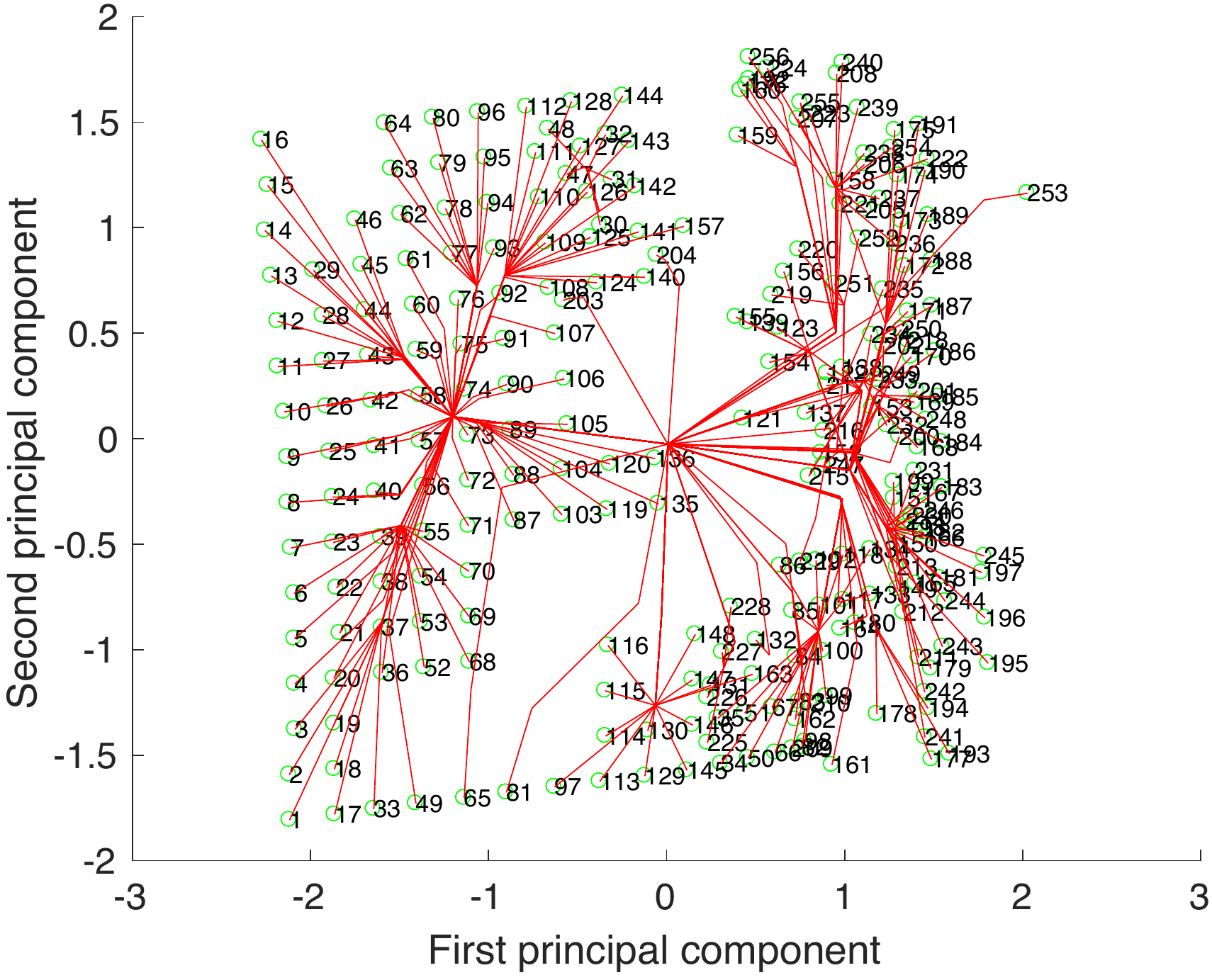}
\includegraphics[width=0.49\columnwidth]{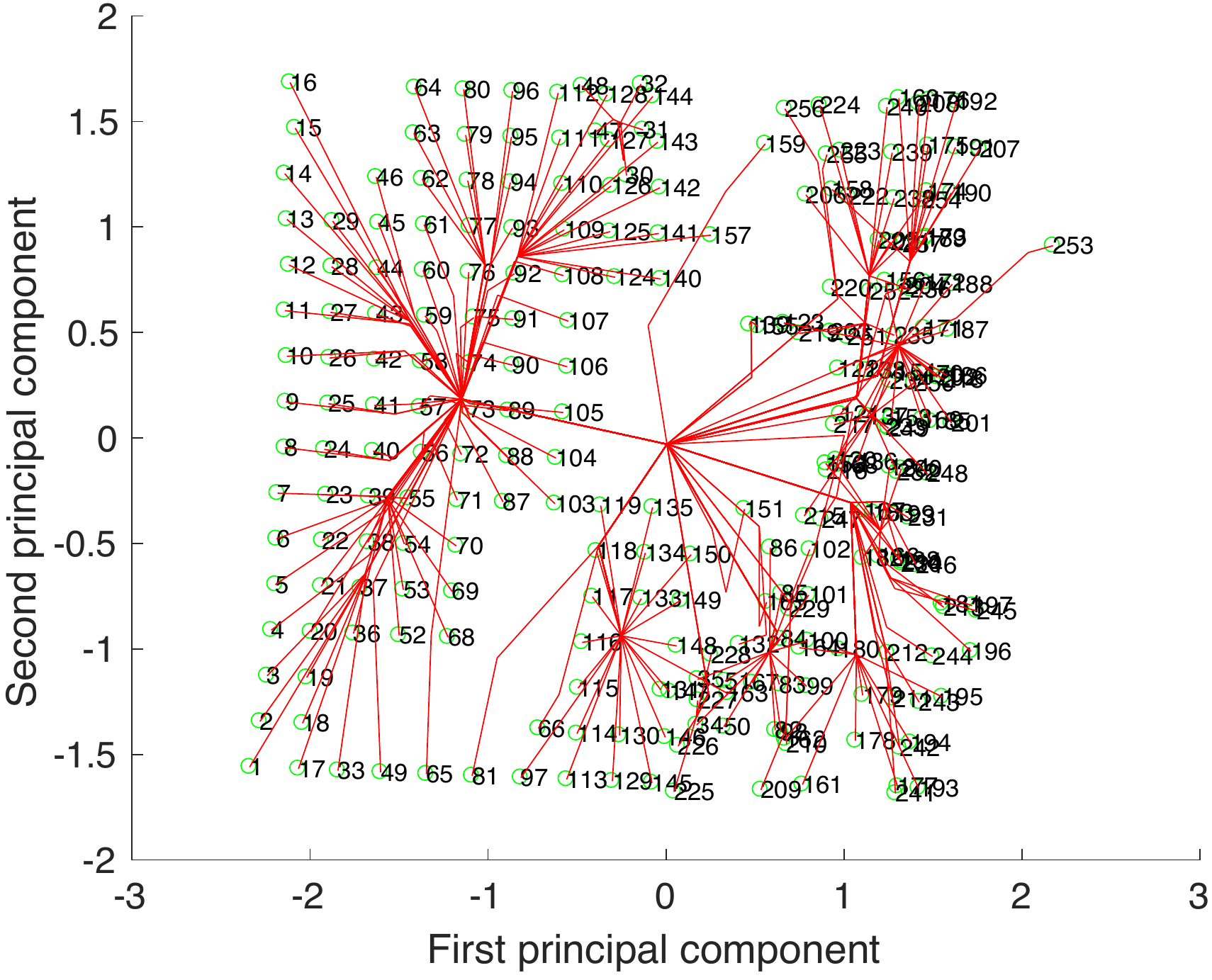}
\caption{Illustrative examples of cluster paths.}
\label{figure_video_clusterpath}
\end{figure*}

\begin{figure*}[t]
\setlength{\abovecaptionskip}{0pt}
\setlength{\belowcaptionskip}{0pt}
\centering 
\includegraphics[width=2\columnwidth]{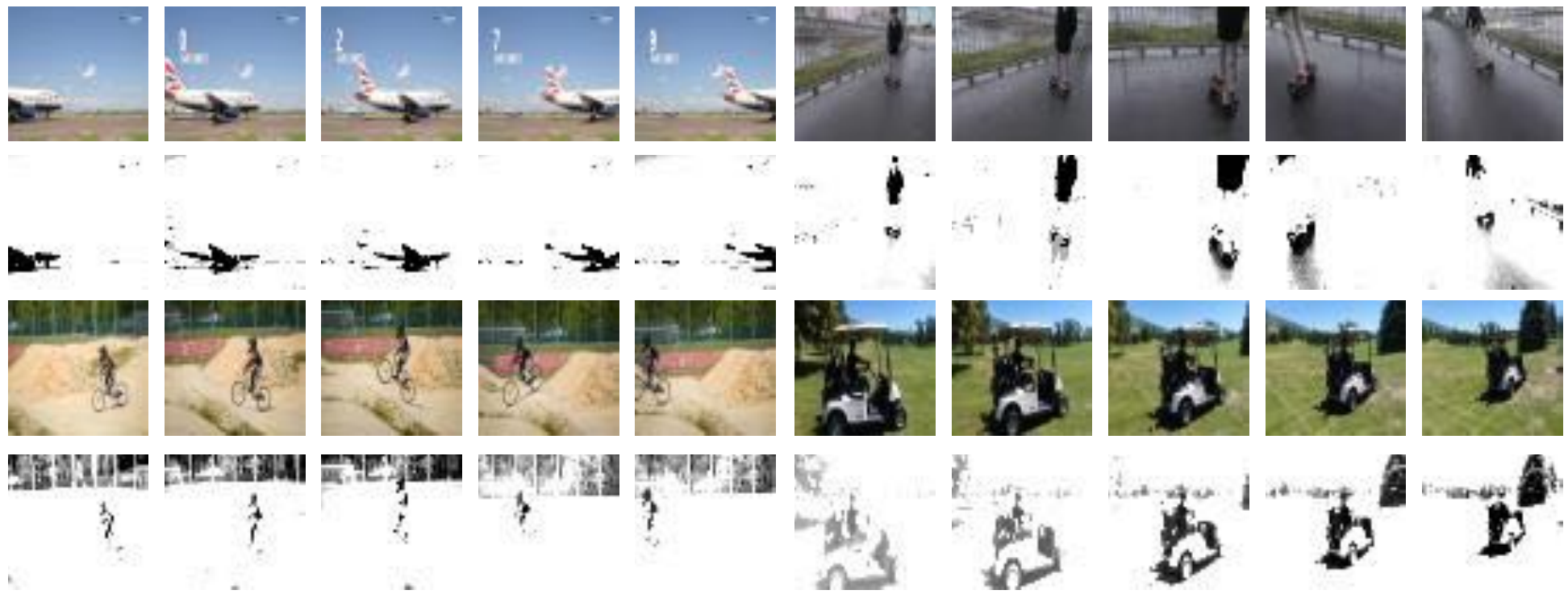}
\caption{Illustrative examples of segmentation of moving objects by using convex clustering.}
\label{figure_video_convex_clustering}
\end{figure*}

\subsection{Cluster paths for video datasets}
 We obtain cluster paths for the DAVIS 2017 video dataset\footnote{https://data.vision.ee.ethz.ch/csergi/share/davis/DAVIS-2017-test-challenge-480p.zip}. The videos in the dataset are at 480 p resolution.  In the experiment, each raw image is transformed into the corresponding grayscale image, and we determine the cluster path by using the grayscale image. Specifically, we partition every image into $256$ blocks according to the size of the image and assign every block a unique ID number. The number near a point in the cluster path represents the corresponding block in the grayscale image.  Together with the position representing location in the image, every block is represented by a $3\times 1$ vector. Additionally, we set the number of neighbors to $K=4$ when $K$-NN is used to construct the sparse graph $\Gcal$ for every image.  Other parameters such as $q$ and $s$ are set to $q=1$ and $s=\infty$ by default.

Figure \ref{figure_video_clusterpath} shows the cluster paths.  We obtain several interesting observations from cluster paths. At a low level, the cluster path shows the clustering membership among blocks of the grayscale image. For example, the blocks with id of $1$, $17$, and $33$ (blocks in the top-left region) belong to a cluster due to their similar gray values.  Blocks $112$, $109$, and $127$ (those in the center-right region) belong to a cluster. This captures the similarity of the blocks at a low level. At a high level, we compare the four images and their corresponding cluster paths. The event of the plane crossing is detected by analyzing the change of the structures of those cluster paths. In particular, the lower-right part of cluster paths changes from sparse paths to dense paths. The change in the cluster path reflects the event of the plane crossing. The above low-level and high-level observations provide significant insights for understanding the video. 

Additionally, as we have shown in the previous section, if $f(\X;\A)=\lrnorm{\X-\A}_F^2$ holds, our method can be used to perform object segmentation for a video. Similarly, we partition every image into $256$ blocks.  Every block is represented by its average gray value.  The illustrative examples are presented in Figure \ref{figure_video_convex_clustering}. As we observe, it is significant to detect the moving object in the video. Although clustering is a traditional method of performing image segmentation, convex clustering is more robust than the classic clustering method such as k-means clustering. The clustering result of convex clustering is determined and is not impacted by the seeds or heuristic rules used in k-means clustering.

\begin{figure}[!t]
\setlength{\abovecaptionskip}{0pt}
\setlength{\belowcaptionskip}{0pt}
\centering 
\subfigure[\textit{space-ga}]{\includegraphics[width=0.49\columnwidth]{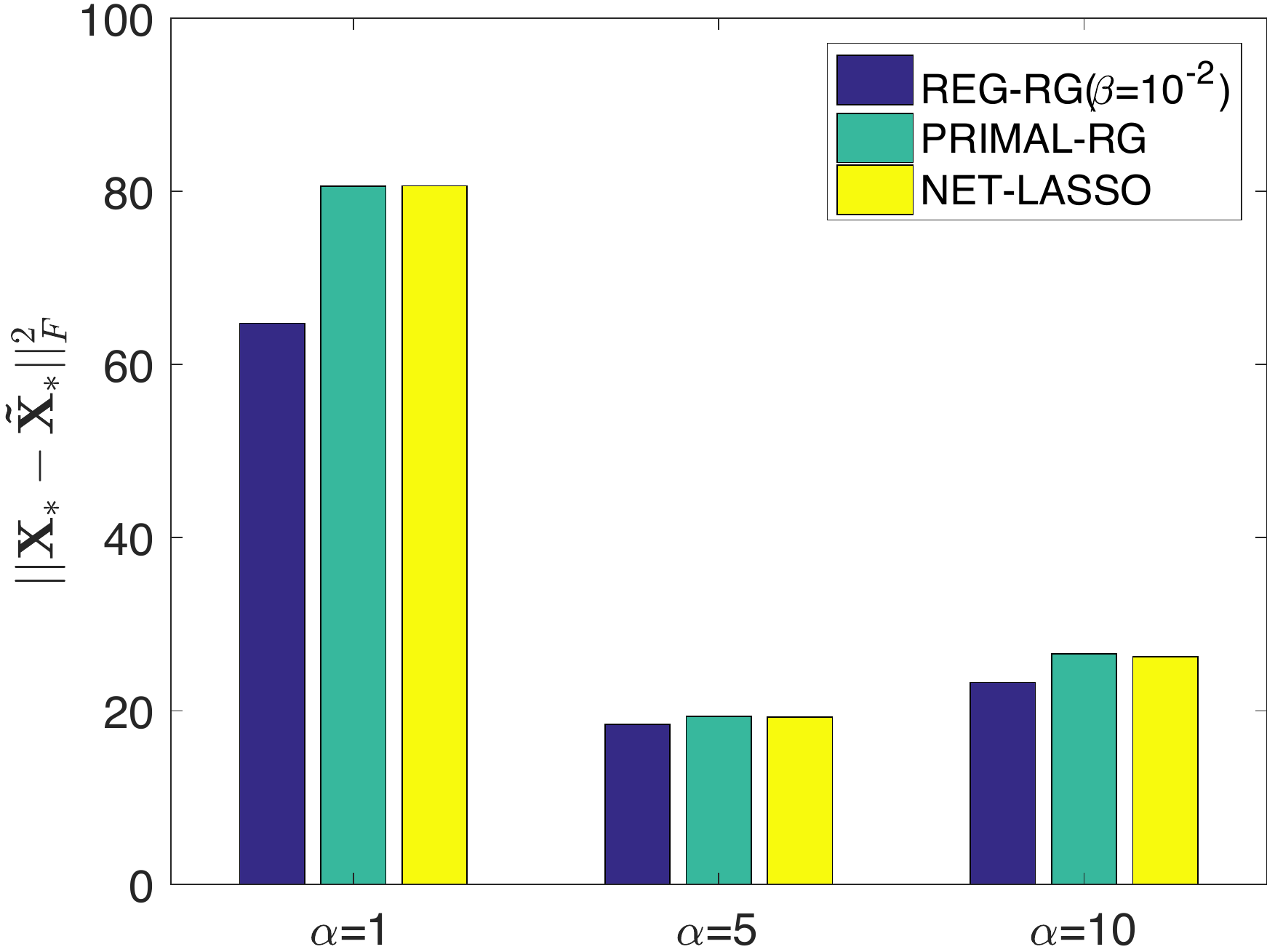}\label{figure_robustness_spacega}}
\subfigure[\textit{airfoil}]{\includegraphics[width=0.49\columnwidth]{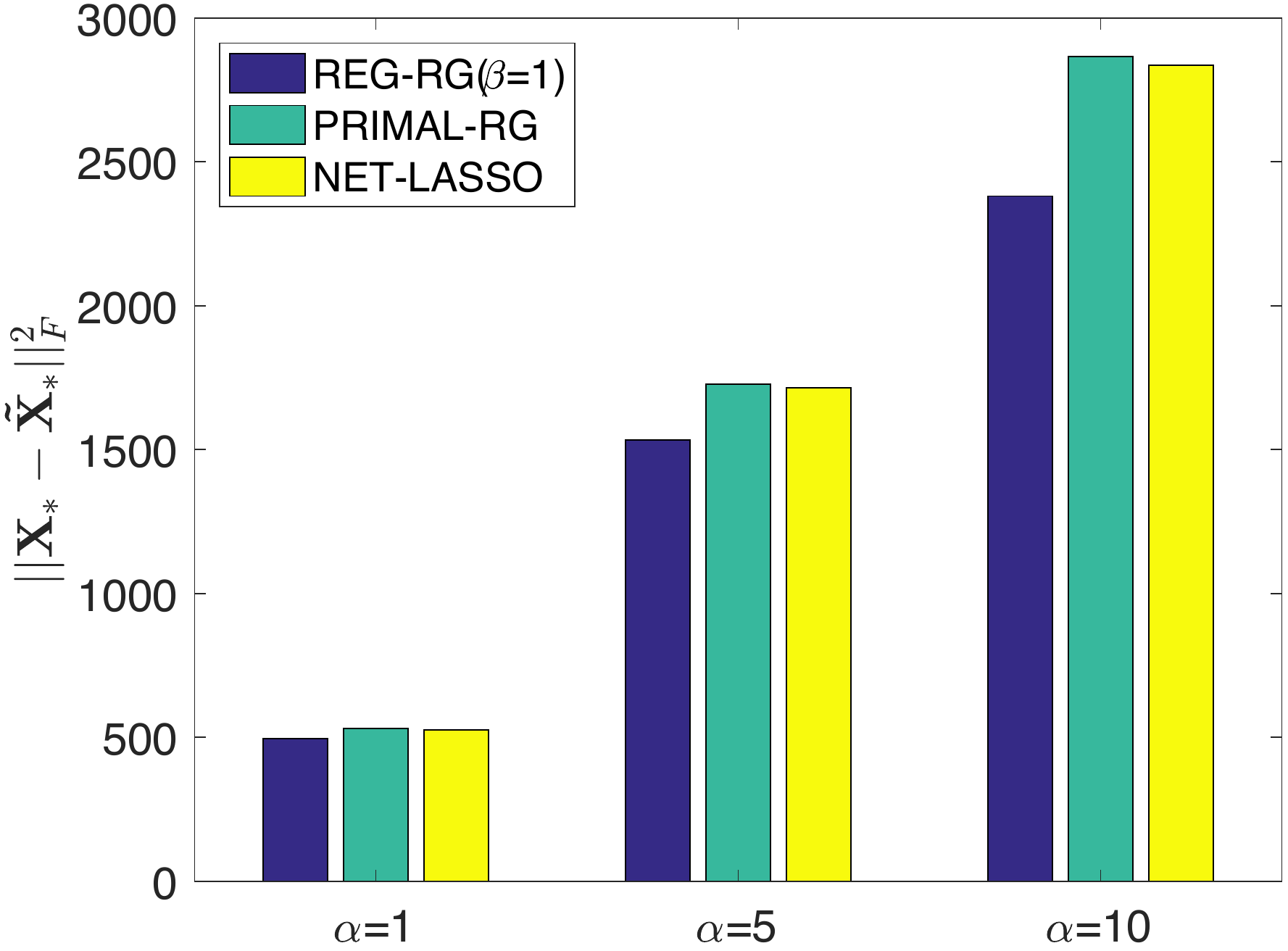}\label{figure_robustness_airfoil}}
\caption{When the datasets evolve, our REG-RG method yields more accurate prediction models than do its counterparts by varying $\alpha$.}
\label{figure_robustness_ridge_regression}
\end{figure}

\begin{figure}[!t]
\setlength{\abovecaptionskip}{0pt}
\setlength{\belowcaptionskip}{0pt}
\centering 
\subfigure[\textit{space-ga}]{\includegraphics[width=0.49\columnwidth]{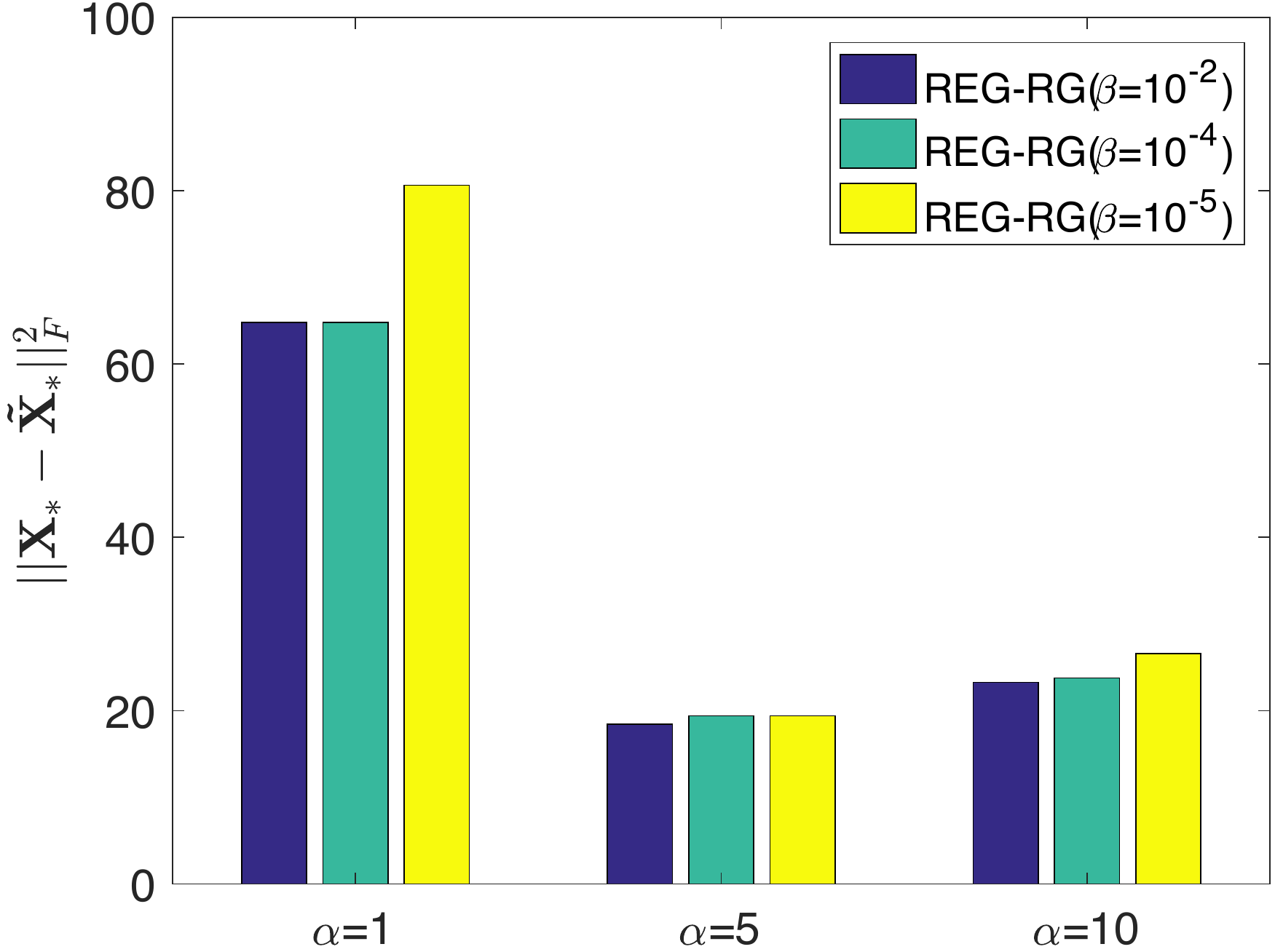}\label{figure_robustness_spacega_beta}}
\subfigure[\textit{airfoil}]{\includegraphics[width=0.49\columnwidth]{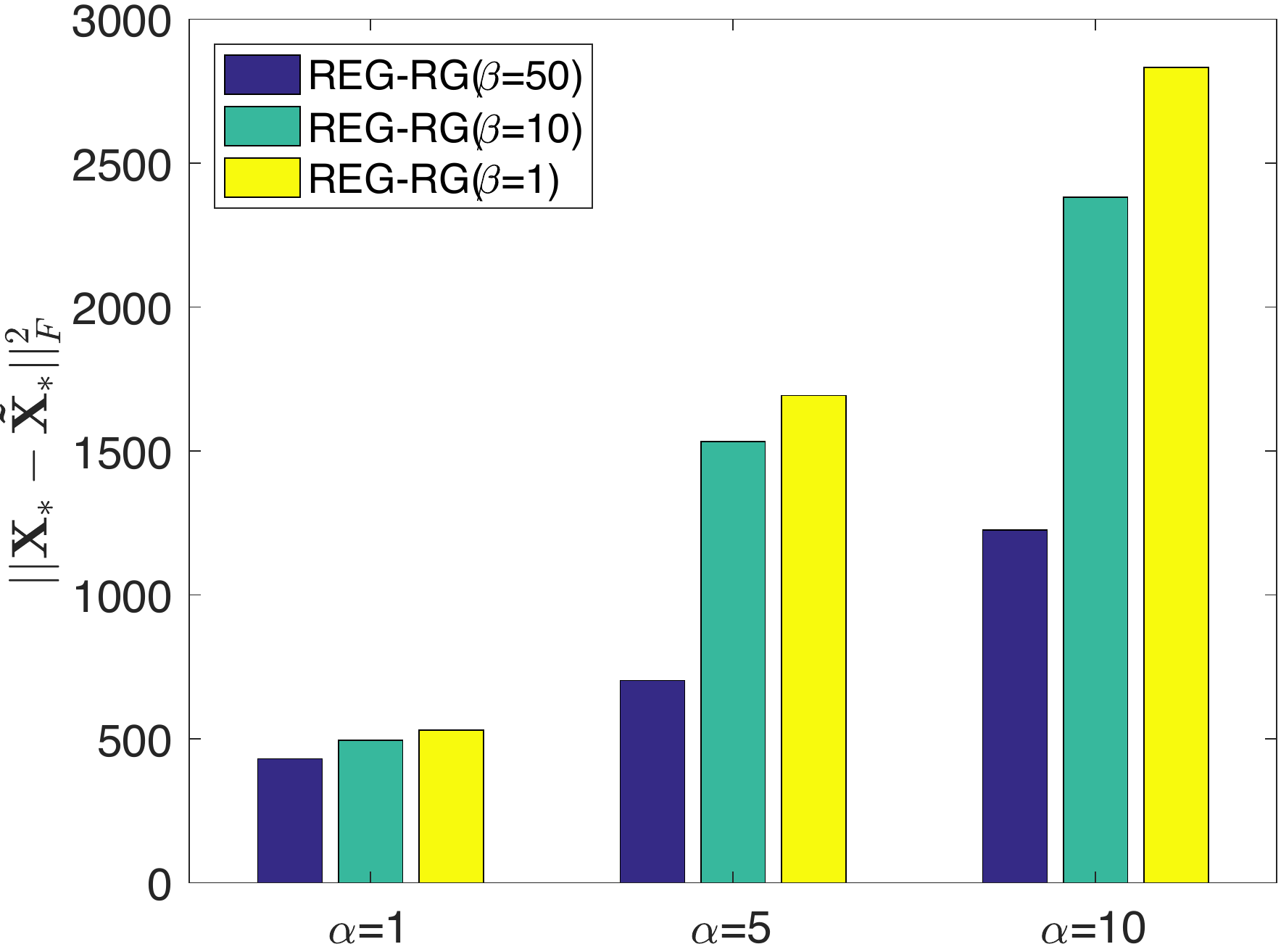}\label{figure_robustness_airfoil_beta}}
\caption{When the datasets evolve, our REG-RG method yields more accurate prediction models with a large $\beta$.}
\label{figure_robustness_ridge_regression_beta}
\end{figure}

\begin{figure}[!t]
\setlength{\abovecaptionskip}{0pt}
\setlength{\belowcaptionskip}{0pt}
\centering 
\subfigure[\textit{space-ga}]{\includegraphics[width=0.49\columnwidth]{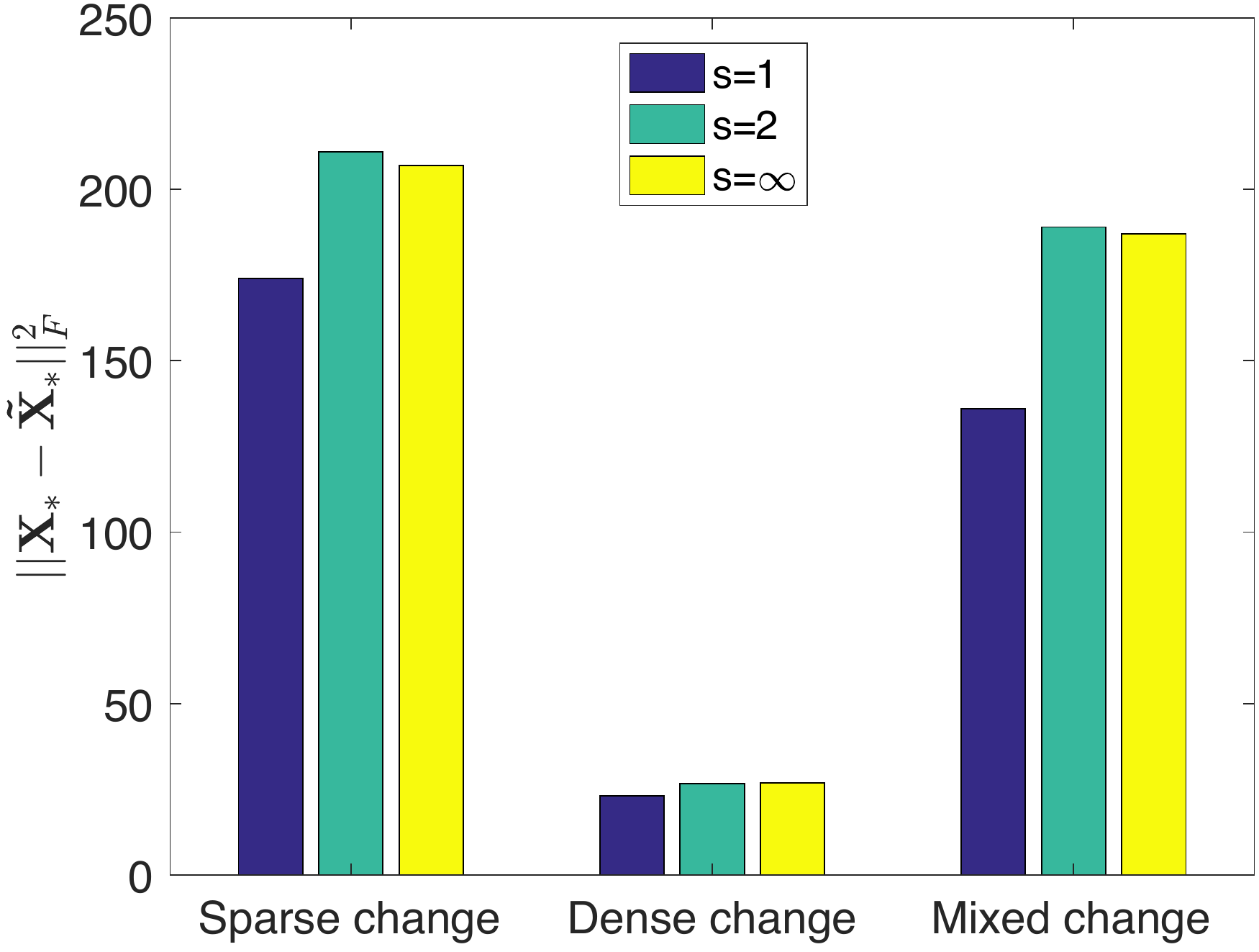}\label{figure_robustness_spacega_noise_type}}
\subfigure[\textit{airfoil}]{\includegraphics[width=0.49\columnwidth]{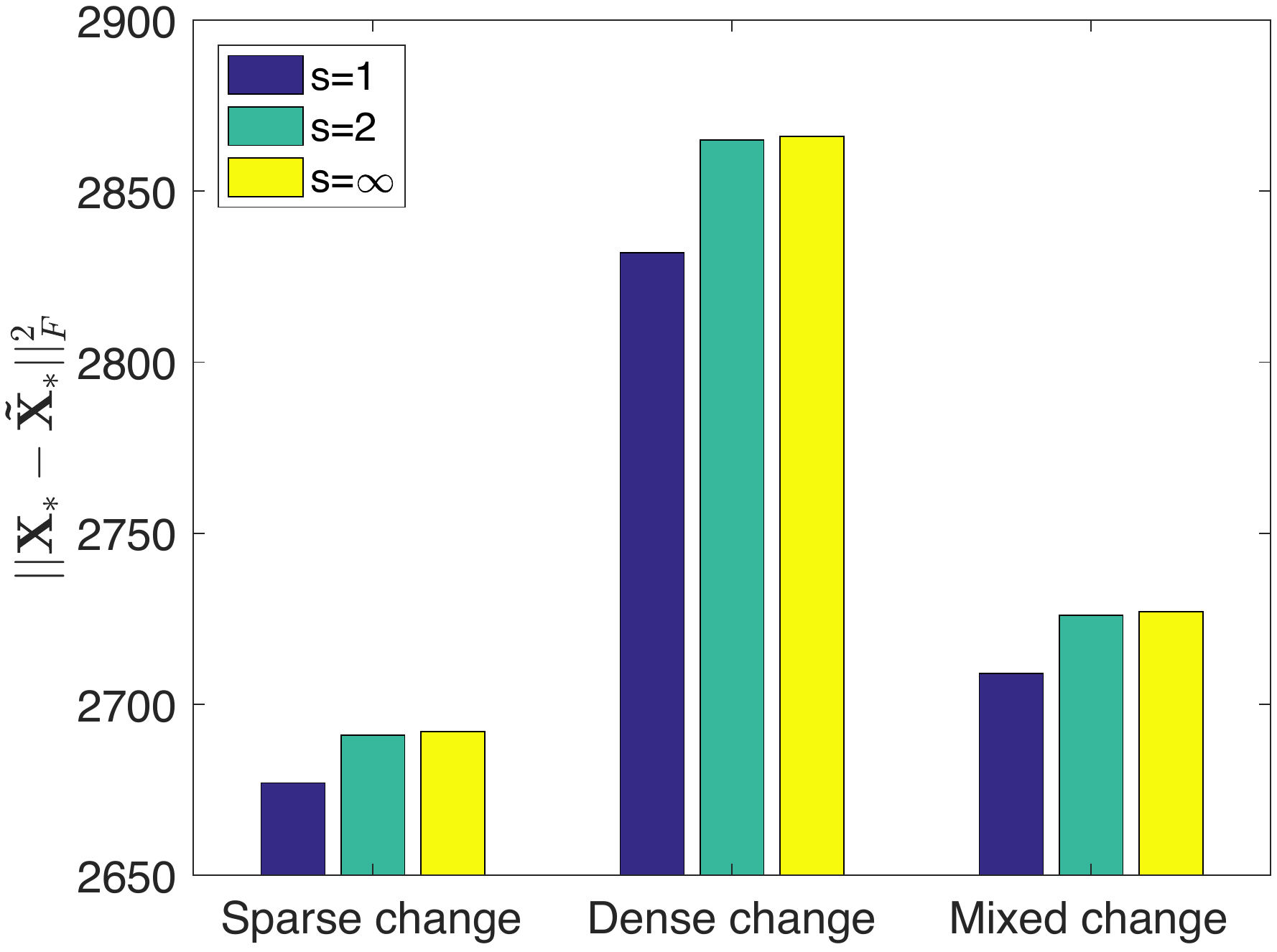}\label{figure_robustness_airfoil_noise_type}}
\caption{Our method yields the best solution at $s=1$.}
\label{figure_robustness_ridge_regression_noise_type}
\end{figure}

\begin{table*}
\centering
\caption{Average CPU time, in seconds, needed to execute the proposed REG-RG method by varying the constraints. The numbers in the parentheses represent variance.}
\begin{tabular}{c|c|c|c|c}
\hline 
dataset/algos & REG-RG$(l_1)$ & REG-RG$(l_2)$ & REG-RG$(l_{\infty})$ & parallel REG-RG$(l_{\infty})$\tabularnewline
\hline 
\textit{space-ga} & $26.41(0.145)$ & $25.25(1.586)$ & $26.63(0.225)$ & $\mathbf{9.43(0.34)}$\tabularnewline
\hline 
\textit{airfoil} & $29.43(0.52)$ & $29.61(0.004)$ & $29.40(0.13)$ & $\mathbf{5.68(0.051)}$\tabularnewline
\hline 
\end{tabular}
\label{table_efficiency_ridge_regression}
\end{table*}

\subsection{Ridge regression}

We further evaluate the performance of the proposed method by conducting the ridge regression task on the \textit{space-ga} and \textit{airfoil} datasets. As we have shown, $f(\X;\A)$ in \eqref{equa_basic_simultaneous_clustering_optimization_primal} is instantiated to be 
\begin{align}
\nonumber
f(\X;\A) = \sum_{i=1}^n \lrincir{\lrnorm{\A_i\X\Tr_i - \y_i}^2_2 + \gamma \lrnorm{\X_i}_2^2}
\end{align} for the ridge regression task. Here, $\gamma = 5$ in the experiment, and we construct the network $\Gcal$ by running the $K$-NN method on the chosen datasets with $K=5$. Additionally, in this analysis, the proposed method is denoted by \textit{REG-RG} for the ridge regression task.  The state-of-the-art method is network lasso \cite{Hallac:2015fy}, denoted by \textit{NET-LASSO}. The proposed method is also compared with the method of solving the primal problem \eqref{equa_basic_simultaneous_clustering_optimization_primal} directly, which is denoted by \textit{PRIMAL-RG}.

As illustrated in Figure \ref{figure_robustness_ridge_regression}, the proposed REG-RG method yields more accurate prediction models than do the existing methods when datasets evolve.   The superiority becomes significant with the increase in $\beta$, as shown in Figure \ref{figure_robustness_ridge_regression_beta}. The reason is that the proposed REG-RG method tends to yield the solution that is relatively robust to the evolving data. The large $\beta$ means that REG-RG applies a larger penalty to the solution, which encourages it to be insensitive to the evolving data. 

Next, we evaluate our method on three kinds of evolving data, including the sparse evolving data, the dense evolving data, and a mixture of both. As we have shown, the sparse evolving data consist of $0.2n$ nonzero values generated from the Gaussian distribution $N(0,0.1^2)$. The dense evolving data are generated from the Gaussian distribution $N(0,0.01^2)$. A mixture of them is obtained by generating $0.2n$ values from the sparse case, while the other values are generated from the dense case. As shown in Figure \ref{figure_robustness_ridge_regression_noise_type}, it is more effective to obtain an approximate solution for the proposed method at $s=1$ than at other values of $s$. Finally, we evaluate the efficiency of the proposed method by varying the constraints. 
We execute every algorithm three times and record the average and the variance (the number in the parentheses) of CPU time, in seconds. Table \ref{table_efficiency_ridge_regression} shows that the efficiency varies slightly when different types of constraints are used. However, due to the use of parallel computing, the proposed method is most efficient at $q=\infty$.

\section{Conclusions}
\label{sect_conclusion}

We investigate SCO in an evolving environment. We first reformulate the problem into a convex problem with cone constraints in a dual space. Then, a new regularizer is proposed to obtain an approximate solution for the evolving dataset. A novel ADMM method is proposed to solve the problem efficiently. Afterward, we analyze the quality of the solution theoretically for the cases of the convex clustering and ridge regression tasks.  Finally, extensive empirical studies show the advantages of the proposed method.

\section*{Acknowledgments}
This work was supported by the National Key R \& D Program of China (Grant No. 2018YFB1003203), the National Natural Science Foundation of China (Grant Nos. 61672528, 61773392, 61701451, 61671463, and 61772544), and the National Basic Research Program (the 973 program) under Grant No. 2014CB347800. 

\ifCLASSOPTIONcaptionsoff
  \newpage
\fi



%\bibliographystyle{IEEEtran}
%\begin{thebibliography}{1}

%\begin{thebibliography}{1}

\bibliographystyle{IEEEtran}  
\bibliography{convex_clustering_evolving_datasets corrections}
%\end{thebibliography}



%\vfill

\begin{IEEEbiography}[{\includegraphics[width=1in,height=1.25in,clip,keepaspectratio]{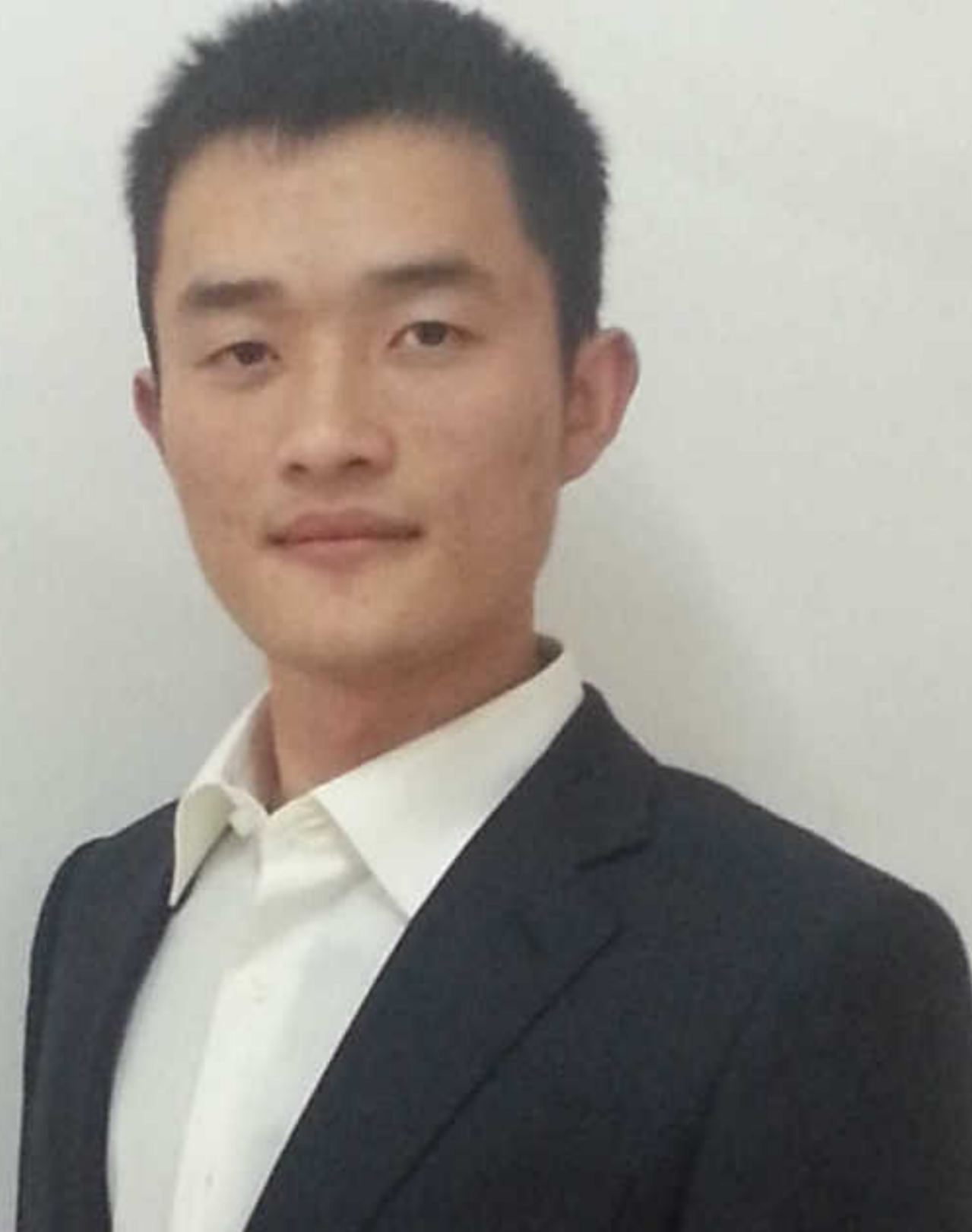}}]{Yawei Zhao} is currently a PhD candidate in Computer Science at the National University of Defense Technology, China. He received his B.E. degree and M.S. degree in Computer Science from the National University of Defense Technology, China, in 2013 and 2015, respectively. His research interests include asynchronous and parallel optimization algorithms, pattern recognition and machine learning.
\end{IEEEbiography}

\begin{IEEEbiography}[{\includegraphics[width=1in,height=1.25in,clip,keepaspectratio]{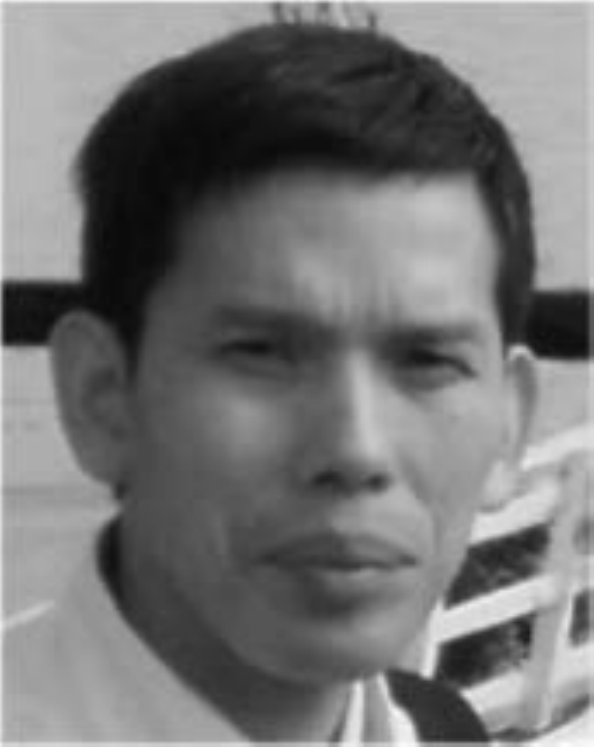}}]{En Zhu} received his M.S. and PhD degrees in Computer Science from the National University of Defense Technology, China, in 2001 and 2005, respectively. He currently works as a full professor in the School of Computer Science, National University of Defense Technology, China. His main research interests include pattern recognition, image processing, and information security.
\end{IEEEbiography}

\begin{IEEEbiography}[{\includegraphics[width=1in,height=1.25in,clip,keepaspectratio]{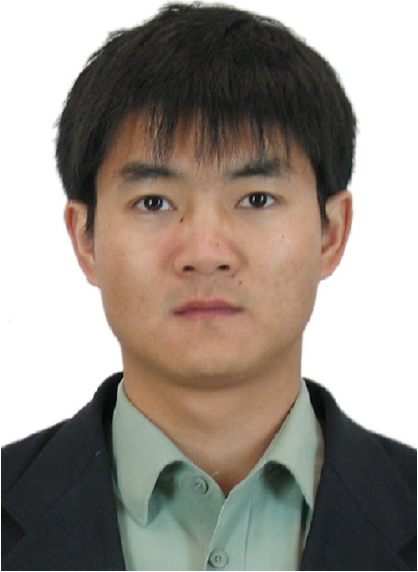}}]{Xinwang Liu} received his PhD degree from National University of Defense Technology (NUDT), China. He is currently an assistant researcher at the School of Computer Science, NUDT. His current research interests include kernel learning and unsupervised feature learning. Dr. Liu has published 40+ peer-reviewed papers, including those in highly regarded journals and conferences, such as IEEE T-IP, IEEE T-NNLS, ICCV, AAAI, IJCAI, etc. He served on the Technical Program Committees of IJCAI 2016-2017 and AAAI 2018-2018.
\end{IEEEbiography}

\begin{IEEEbiography}[{\includegraphics[width=1in,height=1.25in,clip,keepaspectratio]{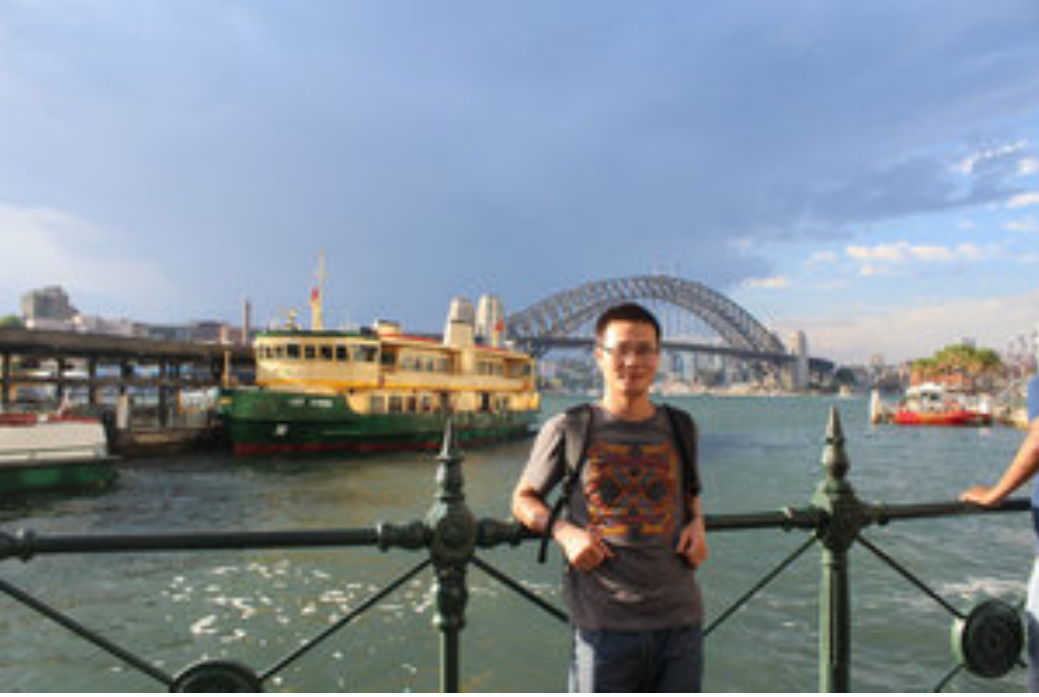}}]{Chang Tang}
received his PhD degree from Tianjin University, Tianjin, China, in 2016. He joined the AMRL Lab of the University of Wollongong between September 2014 and September 2015. He is currently an associate professor at the School of Computer Science, China University of Geosciences, Wuhan, China. Dr. Tang has published 20+ peer-reviewed papers, including those in highly regarded journals and conferences, such as IEEE-TPAMI, IEEE-TMM, IEEE T-HMS, IEEE SPL, ICCV, CVPR, AAAI, ACMMM, etc. He served on the Technical Program Committees of IJCAI 2018/2019, ICME 2018/2019, AAAI 2019, CVPR 2019 and ICCV 2019. His current research interests focus on building machine learning models for solving computer vision and data mining problems.
\end{IEEEbiography}

\begin{IEEEbiography}[{\includegraphics[width=1in,height=1.25in,clip,keepaspectratio]{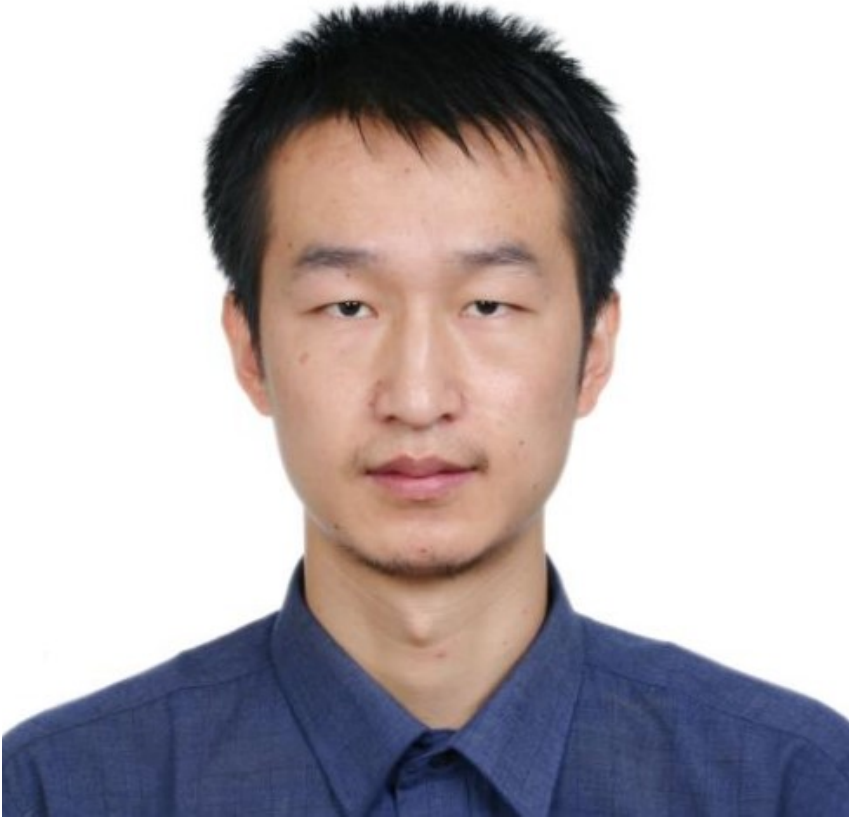}}]
{Deke Guo} received a B.S. degree
in industry engineering from Beijing University
of Aeronautics and Astronautics, Beijing,
China, in 2001, and a PhD degree in management
science and engineering from the National
University of Defense Technology, Changsha,
China, in 2008. He is currently a professor
with the College of Systems Engineering,
National University of Defense Technology. His
research interests include distributed systems,
software-defined networking, datacenter networking,
wireless and mobile systems, and interconnection networks.
He is a senior member of the IEEE and a member of the ACM.
\end{IEEEbiography}

\begin{IEEEbiography}[{\includegraphics[width=1in,height=1.25in,clip,keepaspectratio]{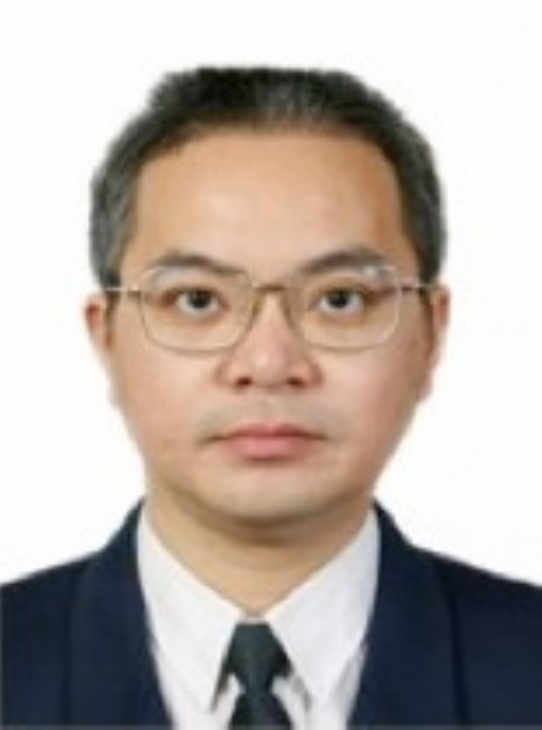}}]{Jianping Yin} received his M.S. and PhD degrees in Computer Science from the National University of Defense Technology, China, in 1986 and 1990, respectively. He is a professor of computer science at Dongguan University of Technology. His research interests involve artificial intelligence, pattern recognition, algorithm design, and information security.
\end{IEEEbiography}

\newpage

\appendix
\textbf{Proof of Theorem \ref{theorem_formulation_dual}.}
\begin{proof}
Recall that \eqref{equa_basic_simultaneous_clustering_optimization_primal} is equivalent to 
\begin{align}
\nonumber
\min_{\X\in\RR^{n\times d}} f(\X; \A) + \lrnorm{\Z}_{1,p} \\ \nonumber
\end{align}
subject to
\begin{align}
\nonumber
\Z = \Q\X.
\end{align} Its Lagrangian multiplier is
\begin{align}
\nonumber
L(\X,\Z,\blambda) = f(\X; \A) + \lrnorm{\Z}_{1,p} + \vec\Tr(\blambda) \vec(\Q\X-\Z).
\end{align} Thus, the dual-optimization objective function is $D(\blambda)$, which is  
\begin{align}
\nonumber
&D(\blambda) \\ \nonumber 
=& \inf\limits_{\X, \Z} f(\X; \A) + \lrnorm{\Z}_{1,p} + \vec\Tr(\blambda) \vec(\Q\X-\Z) \\ \nonumber
=& \inf\limits_{\X} f(\X; \A) + \vec\Tr(\blambda) \vec(\Q\X) \\ \nonumber 
&+ \inf\limits_{\Z} \lrnorm{\Z}_{1,p} - \vec\Tr(\blambda) \vec(\Z) \\ \nonumber
=& \inf\limits_{\X} f(\X; \A) + \vec\Tr(\blambda) (\I_d\otimes\Q)\vec(\X) \\ \nonumber 
&+ \inf\limits_{\Z} \lrnorm{\Z}_{1,p} - \vec\Tr(\blambda) \vec(\Z) \\ \nonumber
=& -\lrincir{\sup\limits_{\X} -f(\X; \A) - \vec\Tr(\blambda) (\I_d\otimes\Q)\vec(\X)} \\ \nonumber 
&- \lrincir{\sup\limits_{\Z} -\lrnorm{\Z}_{1,p} + \vec\Tr(\blambda) \vec(\Z)} \\ \nonumber
=& -f^{\ast}(-\vec\Tr(\blambda)(\I_d\otimes\Q); \A) - g^{\ast}(\blambda)
\end{align} where $f^{\ast}(\cdot)$ is the conjugate function of $f(\cdot)$ \cite{Boyd:2004}, and $g^{\ast}(\cdot)$ is the conjugate function of $g(\cdot)$, defined by  
\begin{align}
\nonumber
g(\cdot) := \lrnorm{\cdot}_{1,p}.
\end{align} Since $g(\blambda) = \sum_{1\le k\le m} \lrnorm{\blambda_k}$, its conjugate function is equivalent to 
\begin{align}
\nonumber
g^{\ast}(\blambda) = \left\{\begin{matrix}
0 if \lrnorm{\blambda_i}_{\ast} \le 1\\ 
\infty otherwise
\end{matrix}\right.
\end{align} where $\lrnorm{\cdot}_{\ast}$ is the dual norm of $\lrnorm{\cdot}$ \cite{Parikh2014Proximal} . If $\frac{1}{p}+\frac{1}{q} = 1$, then the $l_q$ norm, i.e., $\lrnorm{\cdot}_q$, is the dual norm of the $l_p$ norm, i.e., $\lrnorm{\cdot}_p$ \cite{Bertsekas20046}. As illustrated in Table \ref{table_constraint_norms_types}, $p$ and $q$ satisfy $\frac{1}{p}+\frac{1}{q}=1$. Thus, the dual problem of \eqref{equa_basic_simultaneous_clustering_optimization_primal} is 
\begin{align}
\nonumber
\max_{\blambda\in\RR^{m\times d}} & -f^{\ast}(-\vec\Tr(\blambda)(\I_d\otimes\Q); \A) - g^{\ast}(\blambda).
\end{align} Replacing $g^{\ast}(\blambda)$ with $0$ and adding the required  constraint on $\blambda_i$ with $1\le i \le m$, we obtain the dual formulation of \eqref{equa_basic_simultaneous_clustering_optimization_primal}. According to Slater's condition \cite{Boyd:2004}, this is a case of strong duality, and thus, we complete the proof.
\end{proof}

\textbf{Proof of Theorem \ref{theorem_admm_update_u}. }
\begin{proof}
We introduce an auxiliary notation $\u^+ \in \RR^{nd}$ with the $i$-th element ($1\le i \le nd$) defined by
\begin{align}
\nonumber
\u^+_i = \max \lrincir{0, 1-\frac{1}{\lrnorm{\u_i}_q}} = \max \lrincir{0, 1-\frac{1}{\abs{\u_i}}}
\end{align} where $\frac{1}{p}+\frac{1}{q}=1$. Additionally, the proximal operator of function $\phi(\cdot)$ is defined by 
\begin{align}
\nonumber
\mathbf{Prox}_{\phi,\nu}(\x) := \argmin_{\v\in\RR^{d}}\phi(\v) + \frac{\nu}{2}\lrnorm{\v-\x}_2^2
\end{align} where $\x\in\RR^d$, and $\nu$ is a constant positive scalar. The definition is presented in \cite{Parikh2014Proximal}, which we recommend for more detail.

Recall that $\u\in\RR^{nd}$ is a vector; thus, we have
\begin{align}
\nonumber
g(\u) = \lrnorm{\u}_{1,p} = \sum_{i=1}^{nd} \lrnorm{\u_i}_p.
\end{align} The following remark presents the closed form of the proximal operator of $\g(\u)$ provided in \cite{Parikh2014Proximal}. 
\begin{remark}
\label{remark_sum_norms_proximal_operator}
The proximal operator of $g(\u)$ has a closed form
\begin{align}
\label{equa_closed_form_sum_norms_proximal_operator}
\mathbf{Prox}_{g,\rho} = \u^+ \odot \u
\end{align} where $\odot$ represents the Hadamard product, i.e., elementwise multiplication.
\end{remark} Therefore, we obtain
\begin{align}
\nonumber
\u^{(t+1)} = & \argmin_{\u\in\RR^{nd}} L(\blambda^{(t+1)}, \u, \bmu^{(t)}) \\ \nonumber
= & \argmin_{\u\in\RR^{nd}} g(\u) - \u\Tr \bmu^{(t)} \\ \nonumber 
& + \frac{\rho}{2} \lrnorm{\u - (\I_d\otimes \Q)\Tr \vec(\blambda^{(t+1)})}_2^2 \\ \nonumber
= & \argmin_{\u\in\RR^{nd}} g(\u) - \u\Tr \bmu^{(t)} \\ \nonumber 
& + \frac{\rho}{2} \lrincir{\u\Tr\u - 2\u\Tr (\I_d\otimes \Q)\Tr \vec(\blambda^{(t+1)})} \\ \nonumber
= & \argmin_{\u\in\RR^{nd}} g(\u) \\ \nonumber 
& + \frac{\rho}{2} \lrincir{\u\Tr\u - \u\Tr \lrincir{\frac{2}{\rho}\bmu^{(t)} + 2(\I_d\otimes \Q)\Tr \vec(\blambda^{(t+1)})}} \\ \nonumber
= & \argmin_{\u\in\RR^{nd}} g(\u) \\ \nonumber  
& + \frac{\rho}{2} \lrnorm{\u - \lrincir{\frac{1}{\rho}\bmu^{(t)} + (\I_d\otimes \Q)\Tr \vec(\blambda^{(t+1)})}}_2^2 \\ \nonumber
=& \mathbf{Prox}_{g, \rho} \left \{ \frac{1}{\rho}\bmu^{(t)} + (\I_d\otimes \Q)\Tr \vec(\blambda^{(t+1)}) \right \}.
\end{align} This completes the proof.

\end{proof}

\begin{lemma}
\label{lemma_convex_clustering}
For convex clustering, we have 
\begin{align}
\label{equa_lemma_convex_clustering}
\lrnorm{(\I_d\otimes \Q)\Tr \vec(\tilde{\blambda}_\ast)}_s \le \frac{1}{\beta}\vec\Tr(\A+\bDelta)\vec(\A+\bDelta).
\end{align}
\end{lemma}
\begin{proof}

According to the definition of $\tilde{\blambda}_\ast$, we have 
\begin{align}
\nonumber
&0 \\ \nonumber
\ge &  f^{\ast}(-\vec\Tr(\tilde{\blambda}_\ast)(\I_d\otimes \Q);\A+\bDelta) + \beta \lrnorm{(\I_d\otimes \Q)\Tr\vec(\tilde{\blambda}_\ast)}_s,
\end{align} which is equivalent to 
\begin{align}
\nonumber
&\beta \lrnorm{(\I_d\otimes \Q)\Tr \vec(\tilde{\blambda}_\ast)}_s \\ \nonumber
 \le & \vec\Tr(\A+\bDelta)\lrincir{(\I_d\otimes \Q)\Tr \vec(\tilde{\blambda}_\ast)} \\ \nonumber 
 &- \frac{1}{4}\lrincir{\vec\Tr(\tilde{\blambda}_\ast)(\I_d\otimes \Q)}\lrincir{(\I_d\otimes \Q)\Tr \vec(\tilde{\blambda}_\ast)} \\ \nonumber
\le & \sup_{\a\in \RR^{md\times 1}}\vec\Tr(\A+\bDelta)(\I_d\otimes \Q)\Tr \a \\ \nonumber 
 &- \frac{1}{4}\a\Tr(\I_d\otimes \Q)(\I_d\otimes \Q)\Tr \a \\ \nonumber
= & \vec\Tr(\A+\bDelta)\vec(\A+\bDelta).
\end{align} This completes the proof.
\end{proof}

\textbf{Proof of Theorem \ref{theorem_x_bound_convex_clustering}.}
\begin{proof}
According to \eqref{equa_lambda_hat_difference}, we have 
\begin{align}
\nonumber
&\nabla f(\tilde{\X}_\ast; \A+\bDelta) - \nabla f(\X_\ast; \A) \\ \nonumber 
=& 2\lrincir{\vec(\tilde{\X}_\ast) - \vec(\X_\ast)} - 2\vec(\bDelta)  \\ \nonumber
=& (\I_d\otimes \Q)\Tr\lrincir{\vec(\tilde{\blambda}_\ast) - \vec(\blambda_\ast)}.
\end{align} Thus, we have
\begin{align}
\nonumber
&\vec(\tilde{\X}_\ast) - \vec(\X_\ast)  \\ \label{equa_proof_convex_clustering}
=& \frac{1}{2}(\I_d\otimes \Q)\Tr\lrincir{\vec(\tilde{\blambda}_\ast) - \vec(\blambda_\ast)} + \vec(\bDelta).
\end{align} 

 According to \eqref{equa_difference_lambda_hat_control_c}, we have 
\begin{align}
\nonumber
c\ge & \vec\Tr(\A)(\I_d\otimes\Q)\Tr\vec(\blambda_\ast) \\ \nonumber 
& - \vec\Tr(\A+\bDelta)(\I_d\otimes\Q)\Tr\vec(\tilde{\blambda}_\ast) \\ \nonumber
= & -\vec\Tr(\bDelta) (\I_d\otimes\Q)\Tr\vec(\tilde{\blambda}_\ast) \\ \nonumber
&+\vec\Tr(\A)(\I_d\otimes\Q)\Tr\lrincir{\vec(\blambda_\ast) - \vec(\tilde{\blambda}_\ast)}.
\end{align} Thus,
\begin{align}
\nonumber
&\vec\Tr(\A)(\I_d\otimes\Q)\Tr\lrincir{\vec(\blambda_\ast) - \vec(\tilde{\blambda}_\ast)} \\ \nonumber 
\le & c+  \vec\Tr(\bDelta)(\I_d\otimes\Q)\Tr\vec(\tilde{\blambda}_\ast)  \\ \nonumber
\le & c+  \lrnorm{\vec\Tr(\bDelta)}\lrnorm{(\I_d\otimes\Q)\Tr\vec(\tilde{\blambda}_\ast)}  \\ \label{equa_theorem_x_bound_temp}
\refabove{\le}{\ref{equa_lemma_convex_clustering}} & c+  \frac{1}{\beta}\lrnorm{\vec\Tr(\bDelta)}\vec\Tr(\A+\bDelta)\vec(\A+\bDelta).
\end{align} 

The last inequality holds due to Lemma \ref{lemma_convex_clustering}.

Combining \eqref{equa_proof_convex_clustering} and \eqref{equa_theorem_x_bound_temp}, we finally complete the proof.
\end{proof}

\begin{lemma}
\label{lemma_ridge_regression}
For ridge regression, we have 
\begin{align}
\label{equa_lemma_ridge_regression}
\lrnorm{(\I_d\otimes \Q)\Tr \vec(\tilde{\blambda}_\ast)}_s \le \frac{1}{4\beta}\y\Tr\tilde{\bOmega}^{-1}\bLambda\y.
\end{align}
\end{lemma}
\begin{proof}

According to the definition of $\tilde{\blambda}_\ast$, we have 
\begin{align}
\nonumber
 f^{\ast}(-\vec\Tr(\tilde{\blambda}_\ast)(\I_{d}\otimes \Q);\A+\bDelta) + \beta \lrnorm{(\I_{d}\otimes \Q)\Tr \vec(\tilde{\blambda}_\ast)}_s \le 0.
\end{align} It is equivalent to 
\begin{align}
\nonumber
&\beta \lrnorm{(\I_{d}\otimes \Q)\Tr \vec(\tilde{\blambda}_\ast)}_s \\ \nonumber 
\le & -\frac{1}{4} \vec\Tr(\tilde{\blambda}_\ast)(\I_{d}\otimes \Q) \tilde{\bOmega}^{-1}(\I_{d}\otimes \Q)\Tr \vec(\tilde{\blambda}_\ast) \\ \nonumber
& + \frac{1}{2} \vec\Tr(\tilde{\blambda}_\ast)(\I_d\otimes\Q)\tilde{\bOmega}^{-1}\tilde{\bLambda}\y \\ \nonumber
\le & \sup_{\b\in \RR^{md\times 1}}-\frac{1}{4} \b\Tr(\I_{d}\otimes \Q) \tilde{\bOmega}^{-1}(\I_{d}\otimes \Q)\Tr \b \\ \nonumber 
& + \frac{1}{2} \b\Tr(\I_d\otimes\Q)\tilde{\bOmega}^{-1}\tilde{\bLambda}\y \\ \nonumber
= & \frac{1}{4}\y\Tr\tilde{\bOmega}^{-1}\tilde{\bLambda}\y.
\end{align} This completes the proof.
\end{proof}

\textbf{Proof of Theorem \ref{theorem_x_bound_ridge_regression}.}
\begin{proof}
Recalling \eqref{equa_minimizer_primal}, we have $\nabla f(\X;\A) = 2\bOmega\vec(\X)$ and $\nabla f(\X;\A) = (\I_{d}\otimes \Q)\Tr \vec(\blambda_\ast)$ for ridge regression.

Substituting them into \eqref{equa_x_bound_ridge_regression}, we obtain
\begin{align}
\nonumber
4c\ge & \vec\Tr(\tilde{\X}_\ast)\tilde{\bOmega}\vec(\tilde{\X}_\ast)  - \vec\Tr(\X_\ast)\tilde{\bOmega}\vec\Tr(\X_\ast)
\end{align}
 
According to the definition of $\tilde{\bOmega}$, we have 
\begin{align}
\nonumber
&\tilde{\bOmega} \\ \nonumber 
= & \text{diag}(\vec(\A+\bDelta))\lrincir{\mathbf{1}_{ d \times d }\otimes \I_n}\text{diag}(\vec(\A+\bDelta)) \\ \nonumber 
&+ \gamma \I_{n d }  \\ \nonumber
=& \bOmega + 2\text{diag}(\vec(\bDelta))\lrincir{\mathbf{1}_{ d \times d }\otimes\I_n}\text{diag}(\vec(\A)) \\ \nonumber
=& \bOmega + \bPhi
\end{align} where $\bPhi = 2\text{diag}(\vec(\bDelta))\lrincir{\mathbf{1}_{ d \times d }\otimes\I_n}\text{diag}(\vec(\A))$. Thus, we have
\begin{align}
\nonumber
4c\ge & \vec\Tr(\tilde{\X}_\ast)\tilde{\bOmega}\vec(\tilde{\X}_\ast)  - \vec\Tr(\X_\ast)\tilde{\bOmega}\vec\Tr(\X_\ast) \\ \nonumber
= & \vec\Tr(\tilde{\X}_\ast)\bOmega\vec(\tilde{\X}_\ast)  - \vec\Tr(\X_\ast)\bOmega\vec\Tr(\X_\ast) \\ \nonumber
& + \vec\Tr(\tilde{\X}_\ast)\bPhi \vec(\tilde{\X}_\ast) - \vec\Tr(\X_\ast)\bPhi \vec(\X_\ast) \\ \nonumber
= & \vec\Tr(\tilde{\X}_\ast)\bOmega\vec(\tilde{\X}_\ast)  - \vec\Tr(\X_\ast)\bOmega\vec\Tr(\X_\ast) \\ \nonumber
& +  \vec\Tr(\tilde{\blambda}_\ast)(\I_{d}\otimes \Q) \tilde{\bOmega}^{-1} \bPhi \tilde{\bOmega}^{-1} (\I_{d}\otimes \Q)\Tr \vec(\tilde{\blambda}_\ast) \\ \nonumber 
& - \vec\Tr(\blambda_\ast)(\I_{d}\otimes \Q) \bOmega^{-1} \bPhi \bOmega^{-1} (\I_{d}\otimes \Q)\Tr \vec(\blambda_\ast) \\ \nonumber
\ge & \vec\Tr(\tilde{\X}_\ast)\bOmega\vec(\tilde{\X}_\ast)  - \vec\Tr(\X_\ast)\bOmega\vec\Tr(\X_\ast) \\ \nonumber
& -  \lrnorm{\vec\Tr(\tilde{\blambda}_\ast)(\I_{d}\otimes \Q)}^2 \lrnorm{\tilde{\bOmega}^{-1} \bPhi \tilde{\bOmega}^{-1}}  \\ \nonumber 
& - \lrnorm{\vec\Tr(\blambda_\ast)(\I_{d}\otimes \Q)}^2 \lrnorm{\bOmega^{-1} \bPhi \bOmega^{-1}}  \\ \nonumber
\refabove{\ge}{\ref{equa_lemma_ridge_regression}} & \vec\Tr(\tilde{\X}_\ast)\bOmega\vec(\tilde{\X}_\ast)  - \vec\Tr(\X_\ast)\bOmega\vec\Tr(\X_\ast) \\ \nonumber
& - \frac{1}{16\beta^2}\lrnorm{\y\Tr\tilde{\bOmega}^{-1}\tilde{\bLambda}\y}^2 \lrnorm{\tilde{\bOmega}^{-1}\bPhi \tilde{\bOmega}^{-1}} \\ \nonumber 
& - \frac{1}{16\beta^2}\lrnorm{\y\Tr\bOmega^{-1}\bLambda\y}^2\lrnorm{\bOmega^{-1} \bPhi\bOmega^{-1}}.
\end{align} The last inequality holds due to Lemma \ref{lemma_ridge_regression}. Rearranging the terms, we finally complete the proof.
\end{proof}

\textbf{Discussion of the parallel update of $\blambda$ in the case of $p=1$ and $q=\infty$.}

If $q= \infty$, the constraint set $\Ccal$ is 
\begin{align}
\nonumber
\lrnorm{\blambda_i}_{\infty} \le 1, 1\le i\le m.
\end{align} Equivalently, it is reformulated as 
\begin{align}
\nonumber
-\mathbf{1}_{m\times d} \preceq \blambda \preceq \mathbf{1}_{m\times d}.
\end{align} In other words, every element of $\blambda$ is in the interval $[-1, 1]$. Thus, $\Ccal$ can be partitioned into $d$ separable blocks, where the variables in a column of $\blambda$ consist of a block.  Consider that $(\I_d\otimes \Q)$ is a diagonal block matrix that can be partitioned into $d$ separable blocks denoted by $\Bcal_1$, $\Bcal_2$, ..., $\Bcal_d$.  For any $\Bcal_i$ with $1\le i \le d$, $\Bcal_i = \{ i, n+i, 2n+i, ..., (d-1)n+i \}$.   The update of $\blambda$ is thus reformulated as $d$ independent optimization subproblems. The $i$-th subproblem is formulated as
\begin{align}
\nonumber
\blambda^{(t+1)}_{\mathcal{B}_i} = & \argmin_{-\mathbf{1}_m \preceq\blambda_{\mathcal{B}_i}  \preceq \mathbf{1}_m} h(\blambda_{\mathcal{B}_i}) +  \lrincir{\vec\Tr(\blambda) (\I_d\otimes \Q) \bmu^{(t)}}_{\mathcal{B}_i} \\ \label{equa_update_lambda_parallel}
& + \frac{\rho}{2}\lrnorm{\lrincir{(\I_d\otimes \Q)\Tr \vec(\blambda)-\u^{(t)}}_{\mathcal{B}_i}}_2^2.
\end{align} 
 This means that the update of $\blambda$ can be performed in parallel. This advantage is vitally important in a high-dimensional setting.

\end{document}